\documentclass[twoside]{article}

\usepackage{amsmath}
\usepackage{amsthm}
\usepackage{amssymb}
\usepackage{mathtools}
\usepackage{mathrsfs}
\usepackage{upgreek}
\usepackage{mleftright}
\usepackage[pdftex,pdfstartview=FitB]{hyperref}
\usepackage[sf,bf]{titlesec}
\usepackage{bm}
\usepackage{bbold}
\usepackage{euscript}
\usepackage{url}            
\usepackage{booktabs}       
\usepackage{amsfonts}       
\usepackage{nicefrac}       
\usepackage{microtype}      
\usepackage{mathrsfs}
\usepackage{color}
\usepackage{tabu}
\usepackage{multirow}
\usepackage{float}
\usepackage{booktabs}
\usepackage{algorithm}
\usepackage{algorithmic}
\usepackage{color}
\usepackage{tabu}
\usepackage{comment}
\usepackage{graphicx}
\usepackage{epstopdf}
\usepackage{subcaption}
\usepackage{framed}
\usepackage{enumitem}
\usepackage{textcomp}
\usepackage{setspace}
\usepackage{cleveref}
\usepackage{latexsym}
\usepackage{tcolorbox}
\usepackage[round,sort]{natbib}
\usepackage{alltt}
\usepackage{stmaryrd}
\usepackage{xspace}
\usepackage{euscript}
\usepackage{xfrac}    
\usepackage{nicefrac} 
\usepackage{tikz}
\usepackage{tikz-qtree}
\usepackage{todonotes}

\usepackage[margin=1in]{geometry}
\usepackage{fancyhdr}
\fancypagestyle{plain}{\fancyhf{} 
	
	\fancyfoot[C]{\textsf{\thepage}}
}

\definecolor{mydarkblue}{rgb}{0,0.08,0.45}
\hypersetup{ %
	colorlinks=true,
	linkcolor=mydarkblue,
	citecolor=mydarkblue,
	filecolor=mydarkblue,
	urlcolor=mydarkblue}

\newcommand{\scrC}{\mathscr{C}}

\newcommand{\euB}{\EuScript{B}}
\newcommand{\euC}{\EuScript{C}}

\newcommand{\euF}{\EuScript{F}}
\newcommand{\euG}{\EuScript{G}}
\newcommand{\euH}{\EuScript{H}}

\newcommand{\euO}{\EuScript{O}}

\newcommand{\euU}{\EuScript{U}}

\newcommand{\euX}{\EuScript{X}}
\newcommand{\euY}{\EuScript{Y}}

\newcommand{\Ex}{\mathbb{E}}

\newcommand{\zero}{\bm{0}}
\newcommand{\One}{\bm{1}}

\newcommand{\RR}{\mathbb{R}}

\newcommand{\Rp}{\RR_+}

\newcommand{\bbS}{\mathbb{S}}

\newcommand{\NN}{\mathbb{N}}

\newcommand{\diff}{\mathrm{d}}
\DeclareMathOperator{\prox}{prox}
\DeclareMathOperator{\rprox}{rprox}
\DeclareMathOperator{\proj}{proj}

\newcommand{\Id}{\mathrm{Id}}
\newcommand{\e}{\mathrm{e}}

\newcommand{\bM}{\bm{M}}

\newcommand{\ba}{\bm{a}}
\newcommand{\bb}{\bm{b}}

\newcommand{\bx}{\bm{x}}
\newcommand{\by}{\bm{y}}

\newcommand{\bbeta}{\bm{\beta}}

\DeclareMathOperator*{\argmin}{argmin}

\DeclareMathOperator*{\Diag}{Diag}

\newcommand{\sign}{\mathrm{sign}}
\DeclareMathOperator*{\dom}{dom}
\DeclareMathOperator*{\interior}{int}

\newcommand{\bxi}{\bm{\xi}}

\newcommand{\tx}{\widetilde{x}}

\newcommand{\sumN}{\sum_{n=1}^N}
\newcommand{\sumK}{\sum_{k=1}^K}

\newcommand{\sumd}{\sum_{i=1}^d}

\newcommand{\balpha}{\bm{\alpha}}

\newcommand{\btheta}{\bm{\theta}}
\newcommand{\bsigma}{\bm{\sigma}}

\newcommand{\bvartheta}{\bm{\vartheta}}

\newcommand{\sfA}{\mathsf{A}}

\newcommand{\sfD}{\mathsf{D}}

\newcommand{\sfI}{\mathsf{I}}

\newcommand{\sfN}{\mathsf{N}}

\newcommand{\sfW}{\mathsf{W}}

\newcommand{\oRR}{\overline{\RR}}

\newcommand{\ty}{\widetilde{y}}

\newcommand{\iiddist}{\stackrel{\text{i.i.d.}}{\sim}}
\renewcommand{\mid}{\,|\,}
\newcommand{\midd}{\,|\kern-0.25ex|\,}

\newcommand{\setd}{\llbracket d\rrbracket}

\newcommand{\set}[1]{\llbracket #1\rrbracket}
\newcommand{\KL}{\mathrm{KL}}
\newcommand{\TV}{\mathrm{TV}}
\newcommand{\dotp}[2]{\left\langle #1, #2\right\rangle}

\newcommand{\RPP}{\ensuremath{\left]0,+\infty\right[}}
\newcommand{\RP}{\ensuremath{\left[0,+\infty\right[}}

\newcommand{\env}{\mathrm{env}}
\newcommand{\lenv}{\overleftarrow{\env}}
\newcommand{\renv}{\overrightarrow{\env}}
\newcommand{\lprox}[2]{\overleftarrow{\operatorname{P}}_{\negthinspace\negthinspace #1}^{#2}}
\renewcommand{\rprox}[2]{\overrightarrow{\operatorname{P}}_{\negthinspace\negthinspace #1}^{#2}}
\newcommand{\bprox}[2]{\operatorname{P}_{\negthinspace\negthinspace #1}^{#2}}

\newcommand{\lU}[2]{\overleftarrow{U}_{\negthinspace\negthinspace #1}^{#2}}
\newcommand{\rU}[2]{\overrightarrow{U}_{\negthinspace\negthinspace #1}^{#2}}
\newcommand{\lpi}[2]{\overleftarrow{\pi}_{\negthinspace\negthinspace #1}^{#2}}
\newcommand{\rpi}[2]{\overrightarrow{\pi}_{\negthinspace\negthinspace #1}^{#2}}
\newcommand{\Lip}{\mathrm{Lip}}

\DeclareMathOperator*{\arsinh}{arsinh}

\DeclareMathAlphabet\rsfscr{U}{rsfso}{m}{n}

\theoremstyle{plain}
\newtheorem{theorem}{Theorem}[section]
\newtheorem{proposition}[theorem]{Proposition}

\newtheorem{corollary}[theorem]{Corollary}
\theoremstyle{definition}
\newtheorem{definition}[theorem]{Definition}
\newtheorem{assumption}[theorem]{Assumption}
\theoremstyle{remark}
\newtheorem{remark}[theorem]{Remark}

\crefname{assumption}{Assumption}{Assumptions}
\Crefname{assumption}{Assumption}{Assumptions}

\crefname{problem}{Problem}{Problems}
\Crefname{problem}{Problem}{Problems}

\let\le\leqslant
\let\ge\geqslant

\let\tilde\widetilde
\let\bar\overline

\makeatletter
\DeclareFontFamily{OMX}{MnSymbolE}{}
\DeclareSymbolFont{MnLargeSymbols}{OMX}{MnSymbolE}{m}{n}
\SetSymbolFont{MnLargeSymbols}{bold}{OMX}{MnSymbolE}{b}{n}
\DeclareFontShape{OMX}{MnSymbolE}{m}{n}{
	<-6>  MnSymbolE5
	<6-7>  MnSymbolE6
	<7-8>  MnSymbolE7
	<8-9>  MnSymbolE8
	<9-10> MnSymbolE9
	<10-12> MnSymbolE10
	<12->   MnSymbolE12
}{}
\DeclareFontShape{OMX}{MnSymbolE}{b}{n}{
	<-6>  MnSymbolE-Bold5
	<6-7>  MnSymbolE-Bold6
	<7-8>  MnSymbolE-Bold7
	<8-9>  MnSymbolE-Bold8
	<9-10> MnSymbolE-Bold9
	<10-12> MnSymbolE-Bold10
	<12->   MnSymbolE-Bold12
}{}

\let\llangle\@undefined
\let\rrangle\@undefined
\DeclareMathDelimiter{\llangle}{\mathopen}%
{MnLargeSymbols}{'164}{MnLargeSymbols}{'164}
\DeclareMathDelimiter{\rrangle}{\mathclose}%
{MnLargeSymbols}{'171}{MnLargeSymbols}{'171}
\makeatother

\newcommand{\norm}[1]{\left\lVert#1\right\rVert}
\newcommand{\euclidnorm}[1]{\left\lVert#1\right\rVert_2}
\newcommand{\vecnorm}[2]{\left\| #1 \right\|_{{#2}}}
\newcommand{\matsnorm}[2]{|\kern-0.25ex|\kern-0.25ex| #1 |\kern-0.25ex|\kern-0.25ex|_{{#2}}}

\newcommand{\fronorm}[1]{\matsnorm{#1}{\mathrm{F}}}

\newcommand{\onenorm}[1]{\vecnorm{#1}{1}}

\renewcommand{\left}{\mleft}
\renewcommand{\right}{\mright}

\begin{document}
	\title{\sffamily Bregman Proximal Langevin Monte Carlo \\via Bregman--Moreau Envelopes}
	\author{
		Tim Tsz-Kit Lau
		\thanks{
			Department of Statistics and Data Science, Northwestern University, Evanston, IL 60208, USA; Email: \href{mailto:timlautk@u.northwestern.edu}{\texttt{timlautk@u.northwestern.edu}}. } 	
		\and 
		Han Liu\thanks{Department of Computer Science and Department of Statistics and Data Science, Northwestern University, Evanston, IL 60208, USA; Email: \href{mailto:hanliu@northwestern.edu}{\texttt{hanliu@northwestern.edu}}. } 
	}
	\date{}
	
	\maketitle

	\begin{abstract}
		We propose efficient Langevin Monte Carlo algorithms for sampling distributions with nonsmooth convex composite potentials, which is the sum of a continuously differentiable function and a possibly nonsmooth function. We devise such algorithms leveraging recent advances in convex analysis and optimization methods involving Bregman divergences, namely the Bregman--Moreau envelopes and the Bregman proximity operators, and in the Langevin Monte Carlo algorithms reminiscent of mirror descent. The proposed algorithms extend existing Langevin Monte Carlo algorithms in two aspects---the ability to sample nonsmooth distributions with mirror descent-like algorithms, and the use of the more general Bregman--Moreau envelope in place of the Moreau envelope as a smooth approximation of the nonsmooth part of the potential. A particular case of the proposed scheme is reminiscent of the Bregman proximal gradient algorithm. The efficiency of the proposed methodology is illustrated with various sampling tasks at which existing Langevin Monte Carlo methods are known to perform poorly. 
	\end{abstract}

	\section{Introduction}
	The problem of sampling efficiently from high-dimensional log-Lipschitz-smooth and (strongly) log-concave target distributions via discretized Langevin diffusions has been extensively studied in the machine learning and statistics literature lately. A thorough understanding of the nonasymptotic convergence properties of Langevin Monte Carlo (LMC) has been developed, where the log-Lipschitz-smoothness and (strong) log-concavity of the density play a vital role in characterizing its convergence rates. However, such conditions are not always satisifed in applications and there is recent effort to move beyond such scenarios. 	
	On the other hand, since the efficiency of LMC algorithms in the stanard Euclidean space heavily hinges on the shape of the target distributions, algorithms based on Riemannian Langevin diffusions \citep{girolami2011riemann} are considered in the case of ill-conditioned target distributions to exploit the local geometry of the log-density. However, algorithms derived by discretizing such Riemannian Langevin diffusions are notoriously hard to analyze, depending on the choice of the Riemannian metric. 
	
	In this paper, we propose two Riemannian LMC algorithms based on Bregman divergences to efficiently sample from high-dimensional distributions whose potentials (i.e., negative log-densities) are possibly not strongly convex nor (globally) Lipschitz smooth in the standard Euclidean geometry, but only strongly convex and Lipschitz smooth relative to a Legendre function subsequent to a smooth approximation. To be more precise, potentials can take the form of the sum of a relatively smooth part and a nonsmooth part (which includes the convex indicator function of a closed convex set) in the standard Euclidean geometry. A smooth approximation of the nonsmooth part based on the Bregman divergence is used and we instead sample from the smoothened distribution. By tuning a parameter of the smooth approximation, such a smoothened distribution is sufficiently close to the original target distribution. On the other hand, motivated by the connection between Langevin algorithms and convex optimization, the proposed algorithms can be viewed as the sampling analogue of the Bregman proximal gradient algorithm \citep{van2017forward,bauschke2017descent,bolte2018first,bui2021bregman,chizat2021convergence} (cf.~mirror descent in the smooth case), in which Riemannian structures of the algorithms are induced by the Hessian of some Legendre function. This specific choice of the Riemannian metric also offers us a principled way to analyze the behavior of the proposed algorithms.

	\subsection{Langevin and Mirror-Langevin Monte Carlo Algorithms}
	We consider the problem of sampling from a probability measure $\pi$ on $(\RR^d, \euB(\RR^d))$ which admits a density, with slight abuse of notation, also denoted by $\pi$, with respect to the Lebesgue measure 
	\begin{equation}\label{eqn:Leb}
		(\forall x\in\RR^d) \quad \pi(x) = \left.\e^{-U(x)} \,\middle/\, \int_{\RR^d} \e^{-U(y)}\,\diff y \right. , 
	\end{equation}
	where the \emph{potential} $U \colon \RR^d\to\RR\cup\{+\infty\}$ is measurable and we assume that $0<\int_{\euU} \e^{-U(y)}\,\diff y<+\infty$ for $\euU\coloneqq\dom U$. We also write $\pi\propto \e^{-U}$ for \eqref{eqn:Leb}. Usually, the number of dimensions $d \gg 1$. 
	
	To perform such a sampling task, the LMC algorithm \citep[see e.g.,][]{dalalyan2017theoretical} is arguably the most widely-studied gradient-based MCMC algorithm, which takes the form 
	\begin{equation}\label{eqn:ULA}
		(\forall k\in\NN) \quad x_{k+1} = x_k - \gamma\nabla U(x_k) + \sqrt{2\gamma}\,\xi_k, 
	\end{equation}
	where $\xi_k\iiddist\sfN_d(0, \sfI_d)$ for all $k\in\NN$ and $\gamma>0$ is a step size. Possibly with varying step sizes, the LMC algorithm is also referred to as the \emph{unadjusted} Langevin algorithm \citep[ULA;][]{durmus2017nonasymptotic} in the literature, while applying a Metropolis--Hastings correction step at each iteration of \eqref{eqn:ULA} the algorithm is often referred to as the \emph{Metropolis-adjusted} Langevin algorithm \citep[MALA;][]{roberts1996exponential}. 
	ULA is the discretization of the \emph{overdamped} Langevin diffusion, which is the solution to the stochastic differential equation (SDE) 
	\begin{equation}\label{eqn:langevin}
		(\forall t\in\RP)\quad\diff X_t = -\nabla U(X_t)\,\diff t + \sqrt{2}\,\diff W_t, 
	\end{equation}
	where $\{W_t\}_{t\in\RP}$ is a $d$-dimensional standard Wiener process (a.k.a.~Brownian motion). When $U$ is Lipschitz smooth and strongly convex, it is well known that $\pi$ has the \emph{unique} invariant measure, which is the Gibbs measure $X_\infty \propto \e^{-U}$.       
	Under such (or weaker) conditions of $U$, nonasymptotic error bounds of ULA in terms of various disimilarity measures of probability measures, e.g., total variation and Wasserstein distances, and KL, $\chi^2$- and R\'enyi divergences, are well studied and established \citep[see e.g.,][]{dalalyan2017further,durmus2017nonasymptotic,durmus2019high,durmus2019analysis,vempala2019rapid}.   
	To move beyond the Lipschitz smoothness assumption, we consider the case of a possibly nonsmooth composite potential $U$, which takes the following form 
	\begin{equation}\label{eqn:add_composite}
		(\forall x\in\RR^d) \quad U(x) \coloneqq f(x) + g(x), 
	\end{equation}
	where $f\in\Gamma_0(\RR^d)$ is continuously differentiable but possibly not globally Lipschitz smooth (i.e., do not admit a  globally Lipschitz gradient) and $g\in\Gamma_0(\RR^d)$ is possibly nonsmooth (see \Cref{subsec:notation} for the definition of $\Gamma_0(\RR^d)$). 
	
	To demonstrate the sampling counterpart of mirror descent, we consider the smooth case (i.e., $g=0$) which is well studied in the literature. Introduced in \citet{zhang2020wasserstein}, under certain assumptions on $U$, the mirror-Langevin diffusion (MLD) takes the form: for $ t\in\RP $, 
	\begin{equation}\label{eqn:MLD}
		\begin{cases}
			X_t = \nabla \varphi^*(Y_t), \\
			\diff Y_t = -\nabla U(X_t)\,\diff t + \sqrt{2} \left[ \nabla^2 \varphi(X_t) \right]^{\sfrac{1}{2}}\,\diff W_t, 
		\end{cases}
	\end{equation}
	where $\varphi$ is a Legendre function and $\varphi^*$ is the Fenchel conjugate of $\varphi$ (see \Cref{def:conjugate}). 
	An Euler--Maruyama discretization scheme yields the Hessian Riemannian LMC (HRLMC) algorithm: for $k \in\NN$, 
	\begin{equation}\label{eqn:HRLMC}
		x_{k+1} = \nabla\varphi^*\Big(\nabla\varphi(x_k) - \gamma\nabla U(x_k)
		\left. + \sqrt{2\gamma}\left[ \nabla^2 \varphi(x_k) \right]^{\sfrac{1}{2}}\xi_k \right). 
	\end{equation}
	This is the main discretization scheme considered in \citet{zhang2020wasserstein} and an earlier draft of \citet{hsieh2018mirrored}, and further studied in \citet{li2022mirror}, which is a specific instance of the Riemannian LMC reminiscent of the mirror descent algorithm. \citet{ahn2021efficient} consider an alternative discretization scheme motivated by the mirrorless mirror descent \citep{gunasekar2021mirrorless}, called the mirror-Langevin algorithm (MLA): 
	\begin{equation}\label{eqn:MLA}
		(\forall k\in\NN)\quad
		\begin{aligned}
			x_{k+\sfrac{1}{2}} &= \nabla\varphi^* \left(\nabla\varphi(x_k) -\gamma\nabla U(x_k) \right), \\
			x_{k+1} &= \nabla\varphi^*(Y_{\gamma_k}), 			
		\end{aligned}
	\end{equation}	
	where 
	\begin{equation}
		\begin{cases}
			\diff Y_t = \sqrt{2} \left[ \nabla^2 \varphi^*(Y_t) \right]^{-\sfrac{1}{2}}\,\diff W_t\\
			Y_0 = \nabla\varphi\left(x_{k+\sfrac{1}{2}}\right) = \nabla\varphi(x_k) -\gamma\nabla U(x_k). 
		\end{cases} 
	\end{equation}
	However, the mirror descent-type Langevin algorithms in \citet{hsieh2018mirrored,zhang2020wasserstein,ahn2021efficient} can only handle relatively smooth potentials (to a Legendre function; see \Cref{def:rel_smooth}) but not potentials with relatively smooth plus nonsmooth parts \eqref{eqn:add_composite} where $g\ne0$.

	\subsection{Contributions} 
	We fill this void by extending HRLMC in the following aspects: (i) the target potential $U$ takes the form \eqref{eqn:add_composite}, i.e., $U = f + g$, where $f$ is continuously differentiable but possibly not Lipschitz smooth yet smooth relative to a Legendre function $\varphi$, and $g$ is possibly nonsmooth; (ii) the nonsmooth part $g$ is enveloped by its continuously differentiable approximation, which is the Bregman--Moreau envelope \citep{kan2012moreau,bauschke2018regularizing,laude2020bregman,soueycatt2020regularization,bauschke2006joint,chen2012moreau}, in the same vein as using the Moreau envelope \citep{moreau1962fonctions,moreau1965proximite} in \citet{brosse2017sampling,durmus2018efficient,luu2021sampling}, so that we can adapt recent convergence results for mirror-Langevin algorithms for relatively smooth potentials \citep{zhang2020wasserstein,ahn2021efficient,li2022mirror,jiang2021mirror}.
	
	The proposed sampling algorithm can be viewed as a generalized version of the Moreau--Yosida Unadjusted Langevin Algorithm \citep[MYULA;][]{durmus2018efficient,brosse2017sampling}, and we recover MYULA if both the mirror map and the Legendre function in the smooth approximation are chosen as $\sfrac{\|\cdot\|^2}{2}$. 
	Similar to the resemblance of MYULA to the proximal gradient algorithm with specific choice of step sizes, the proposed discretized algorithms is also reminiscent of the Bregman proximal gradient algorithm or the Bregman forward-backward algorithm  \citep[][see \Cref{subsec:same} for details]{van2017forward,bauschke2017descent,bolte2018first,bui2021bregman}. The proposed schemes, however, are able to change the geometry of the potential through a mirror map. On the theoretical front, our convergence results reveal a biased convergence guarantee with a bias which vanishes with the step size and the smoothing parameter of the Bregman--Moreau envelope. 	
	Numerical experiments also illustrate the efficiency of the proposed algorithms. We perform various nonsmooth (composite) and/or constrained sampling tasks, including sampling from the nonsmooth anisotropic Laplace distributions, at which MYULA is known to underperform ascribed to the anisotropy. To the best of our knowledge, the proposed algorithms are the \emph{first} gradient-based Monte Carlo algorithms based on the \emph{overdamped} Langevin dynamics which are able to sample \emph{nonsmooth} composite distributions while adapting to the \emph{geometry} of such distributions.

	\subsection{Related Work}
	\subsubsection{Mirror Descent-Type Sampling Algorithms}
	In addition to \citet{zhang2020wasserstein,ahn2021efficient}, \citet{hsieh2018mirrored} introduce the mirrored-Langevin algorithm, which is also reminiscent of mirror descent, but only with $I_d$ instead of $\nabla^2\varphi$ in \eqref{eqn:MLD}, which entails a standard Gaussian noise in \eqref{eqn:HRLMC}. Their convergence guarantee is also based on the assumption that $\varphi$ is strongly convex. 
	\citet{chewi2020exponential} analyze the continuous-time MLD \eqref{eqn:MLD} and specialize their results to the case when the mirror map is equal to the potential, known as the Newton--Langevin diffusion due to its resemblance to the Newton's method in optimization. 	
	\citet{li2022mirror} improve upon the analysis of \citet{zhang2020wasserstein}, establishing a vanishing bias with the step size of the mirror Langevin algorithm under more relaxed assumptions.

	\subsubsection{Nonsmooth Sampling}
	Sampling efficiently from nonsmooth distributions remains a crucial problem in machine learning, statistics and imaging sciences. In particular, a significant amount of work borrows tools from convex/variational analysis and proximal optimization, i.e., the Moreau envelope and proximity operator, attributing their use to the connection between sampling and optimization, see e.g., \citet{pereyra2016proximal,brosse2017sampling,bubeck2018sampling,durmus2019analysis,durmus2018efficient,mou2019efficient,wibisono2019proximal,luu2021sampling,lee2021structured,lehec2021langevin,liang2021proximal}. Nonasymptotic convergence guarantees are generally obtained from the (Metropolis-adjusted) Langevin algorithms for smooth potentials. 
	A notable exception which does not use the Moreau envelope as a smooth approximation is \citet{chatterji2020langevin}, which applies Gaussian smoothing instead.

	\subsubsection{Bregman Divergences in Convex Analysis, Optimization and Machine Learning}
	The origin of convex analysis results involving Bregman divergences \citep{bregman1967} and related optimization methods date backs to more than four decades ago \citep[see e.g.,][]{bauschke1997legendre,bauschke2000dykstras,bauschke2001essential,bauschke2001joint,bauschke2003duality,bauschke2003bregman,bauschke2006joint,bauschke2009bregman,nemirovski1979efficient,nemirovskij1983problem}. 	
	The work by \citet{bauschke2017descent} is a major recent breakthrough which revives much interest in developing new optimization algorithms involving Bregman divergences and their convergence results \citep[see e.g.,][]{bui2021bregman,bolte2018first,bauschke2019linear,dragomir2021optimal,dragomir2021quartic,hanzely2021accelerated,teboulle2018simplified,takahashi2021new,chizat2021convergence}. \citet{bauschke2017descent} relax the globally Lipschitz gradient assumption commonly required in gradient descent or proximal gradient for convergence, by introducing the relative smoothness condition (\Cref{def:rel_smooth}). Our proposed sampling algorithms also rely on such an insightful condition.	
	Another long line of work studies the generalization of the notions of the classical Moreau envelope and the proximity operators \citep{moreau1962fonctions,moreau1965proximite,rockafellar1998,bauschke2017} in convex analysis using Bregman divergences, see e.g.,  \citet{bauschke2003bregman,bauschke2006joint,chen2012moreau,kan2012moreau,bauschke2018regularizing,laude2020bregman,soueycatt2020regularization}. This line of work motivates our use of the Bregman--Moreau envelopes as smooth approximations of the nonsmooth part of the potential. 	
	While there is an extensive amount of literature regarding the applications of Bregman divergences in machine learning other than mirror descent \citep{bubeck2015convex}, we refer to \citet{blondel2020learning} which includes useful results for sampling distributions on various convex polytopes such as the probability simplex based on our proposed schemes.

	\subsection{Notation} 
	\label{subsec:notation}
	We denote by $I_d\in\RR^{d\times d}$ the $d\times d$ identity matrix. 
	We also define $\oRR \coloneqq \RR\cup\{+\infty\}$. 
	Let $\bbS^d_{++}$ denote the set of symmetric positive definite matrices of $\RR^{d\times d}$.  
	Let $\euH$ be a real Hilbert space endowed with an inner product $\dotp{\cdot}{\cdot}$ and a norm $\|\cdot\|$. 
	The domain of a function $f\colon\euH\to\oRR$ is $\dom f \coloneqq \{x\in\euH : f(x)<+\infty\}$. 
	The set $\Gamma_0(\RR^d)$ denotes the class of lower-semicontinuous convex functions from $\RR^d$ to $\oRR$ with a nonempty domain (i.e., proper). 
	The convex indicator function $\iota_\euC(x)$ of a closed convex set $\euC\ne\varnothing$ at $x$ equals $0$ if $x\in\euC$ and $+\infty$ otherwise.  
	We denote by $\euB(\RR^d)$ the Borel $\sigma$-field of $\RR^d$. 
	For two probability measures $\mu$ and $\nu$ on $\euB(\RR^d)$, 
	the total variation distance between $\mu$ and $\nu$ is defined by $\|\mu - \nu\|_{\TV} = \sup_{\sfA\in\euB(\RR^d)}|\mu(\sfA) - \nu(\sfA)|$. 
	For $k\in\NN$, we denote by $\scrC^k$ the set of $k$-times continuously differentiable functions $f\colon\RR^d\to\RR$.     
	If $f$ is a Lipschitz function, i.e., there exists $L>0$ such that for all $x,y\in\RR^d$, $|f(x) - f(y)|\le L\|x-y\|$, then we denote $\|f\|_{\Lip} \coloneqq \inf\{|f(x) - f(y)|/\|x-y\|\mid x,y\in\RR^d, x\ne y\}$.

	\section{Preliminaries}
	\label{sec:prelim}
	In this section, we give definitions of important notions from convex analysis \citep{rockafellar1970,rockafellar1998,bauschke2017}, and state some related properties of such notions. 
	In this section, we let $\varphi\in\Gamma_0(\RR^d)$ and $\euX\coloneqq \interior\dom\varphi$. 
	\begin{definition}[Legendre functions]
		A function $\varphi$ is called (i) \emph{essentially smooth}, if it is differentiable on $\euX\ne\varnothing$ and $\|\nabla\varphi(x_n)\|\to+\infty$ whenever $x_n\to x \in\operatorname{bdry}\dom\varphi$ ; (ii) \emph{essentially strictly convex}, if it is strictly convex on $\euX$; (iii) \emph{Legendre}, if it is both essentially smooth and essentially strictly convex. 
	\end{definition}
	
	\begin{definition}[Fenchel conjugate]\label{def:conjugate}
		The \emph{Fenchel conjugate} of a proper function $f$ is defined by 
		$f^*(x) \coloneqq \sup_{y\in\RR^d} \,\left\lbrace \dotp{y}{x} - f(y) \right\rbrace$. 
		For a Legendre function $\varphi$, it is well known that $\nabla\varphi\colon\euX\to\euX^*\coloneqq\interior\dom\varphi^*$ with $(\nabla\varphi)^{-1} = \nabla\varphi^*$. 
	\end{definition}

	\begin{definition}[Bregman divergence]
		The \emph{Bregman divergence} between $x$ and $y$ associated with a Legendre function $\varphi$ is defined through 
		\begin{equation}\label{eqn:bregman_div}
			D_{\varphi}\colon \RR^d\times\RR^d \to[0, +\infty] \colon
			(x, y) \mapsto 
			\begin{cases*}
				\varphi(x) - \varphi(y) - \dotp{\nabla \varphi(y)}{x - y}, & if $y\in\euX$, \\
				+\infty, & otherwise. 
			\end{cases*}
		\end{equation}
	\end{definition}
	We now assume that $\varphi$ is a Legendre function in the remaining part of this section. 
	\begin{definition}[Bregman--Moreau envelopes]
		\label{def:bregman_env}
		For $\lambda >0$, the \emph{left} and \emph{right Bregman--Moreau envelopes} of $g\in\Gamma_0(\RR^d)$ associated with $\varphi$ are respectively defined by
		\begin{equation}\label{eqn:Bregman_Moreau}
			\lenv_{\lambda, g}^\varphi(x) \coloneqq \inf_{y\in\RR^d} \, \left\lbrace g(y) + \frac{1}{\lambda}D_\varphi(y, x) \right\rbrace, 
		\end{equation}
		and 
		\begin{equation}
			\renv_{\lambda, g}^\varphi(x) \coloneqq \inf_{y\in\RR^d} \, \left\lbrace g(y) + \frac{1}{\lambda}D_\varphi(x, y) \right\rbrace. 
		\end{equation}
	\end{definition}
	
	\begin{definition}[Bregman proximity operators]
		\label{def:bregman_prox}
		For $\lambda >0$, the \emph{left} and \emph{right Bregman proximity operators} of $g\in\Gamma_0(\RR^d)$ associated with $\varphi$ are respectively defined by 
		\begin{equation}
			\lprox{\lambda, g}{\varphi}(x) \coloneqq \argmin_{y\in\RR^d}  \, \left\lbrace g(y) + \frac{1}{\lambda}D_\varphi(y, x)\right\rbrace, 
		\end{equation} 
		and
		\begin{equation}
			\rprox{\lambda, g}{\varphi}(x) \coloneqq \argmin_{y\in\RR^d} \, \left\lbrace g(y) + \frac{1}{\lambda}D_\varphi(x, y) 	\right\rbrace. 
		\end{equation}
	\end{definition}
	We omit the arrows and write $\env_{\lambda, g}^\varphi$ and $\bprox{\lambda, g}{\varphi}$ when there is no need to distinguish the left and right Bregman--Moreau envelopes or Bregman proximity operators. 
	When $\varphi = \sfrac{\|\cdot\|^2}{2}$, we recover the classical \emph{Moreau envelope} and the \emph{Moreau proximity operator} \citep{moreau1962fonctions,moreau1965proximite}. 
	Note that $\env_{\lambda, g}^\varphi$ envelops $g$ from below and is decreasing in $\lambda$, in a sense that $\inf g(\RR^d) \le \env_{\lambda, g}^\varphi(x) \le \env_{\kappa, g}^\varphi(x) \le g(x)$ for any $x\in\euX$ and $0 < \kappa < \lambda < +\infty$ \citep[Proposition 2.2]{bauschke2018regularizing}. 
	
	\begin{definition}[Legendre strongly convex]
		A function $f$ is $\alpha$-\emph{Legendre strongly convex} with respect to $\varphi$ if there exists a constant $\alpha\ge0$ such that $ f - \alpha\varphi$ is convex on $\euX$. 	
	\end{definition}
	
	\begin{definition}[Relative smoothness]\label{def:rel_smooth}
		A function $f$ is $\beta$-smooth relative to $\varphi$ if there exists $\beta>0$ such that $\beta\varphi - f$ is convex on $\euX$. 	
	\end{definition}

	\section{Bregman Proximal LMC Algorithms}
	In the case of nonsmooth composite potentials, the mirror-Langevin algorithms \eqref{eqn:HRLMC} and \eqref{eqn:MLA} no longer work since the gradient of the nonsmooth part is not available. 
	Based on the mirror Langevin algorithms for relatively smooth potentials \citep{zhang2020wasserstein,ahn2021efficient}, we devise two possible Bregman proximal LMC algorithms involving the Bregman--Moreau envelopes and the Bregman proximity operators.

	\subsection{Assumptions and Related Properties}
	Instead of directly sampling from $\pi$, we propose to sample from distributions whose potentials being smooth surrogates of $U$, defined by 
	\begin{equation}\label{eqn:env_potentials}
		\lU{\lambda}{\psi} \coloneqq f + \lenv_{\lambda, g}^\psi \quad\text{and}\quad \rU{\lambda}{\psi} \coloneqq f + \renv_{\lambda, g}^\psi,
	\end{equation}
	where $\psi\in\Gamma_0(\RR^d)$ is a Legendre function possibly different from the Legendre function $\varphi$ in MLD \eqref{eqn:MLD} to allow full flexibility, and $\lambda>0$.   
	Then the corresponding surrogate target densities are 
	\begin{equation}\label{eqn:densities}
		\lpi{\lambda}{\psi} \propto \exp\left( -\lU{\lambda}{\psi}\right) \quad \text{and} \quad \rpi{\lambda}{\psi} \propto \exp\left( -\rU{\lambda}{\psi}\right). 
	\end{equation}
	We again omit the arrows and write $U_\lambda^\psi$ and $\pi_\lambda^\psi$ when we do not need to distinguish the left and right Bregman--Moreau envelopes. 
	In this section, after introducing some required assumptions, we show that they are well-defined (i.e., in $\left]0, +\infty\right[$), as close to $\pi$ by adjusting the (sufficiently small) approximation parameter $\lambda>0$, Legendre strongly log-concave, and relatively smooth. We also give some extra assumptions of the specific algorithms (\Cref{assum:functions}). 	
	Then, with all the assumptions of algorithms originally designed for the relatively smooth potentials satisfied by \eqref{eqn:env_potentials}, we enable the capabilities of these algorithms for approximate nonsmooth sampling.

	We write $\euF \coloneqq \dom f \subseteq \RR^d$ and $\euG \coloneqq \dom g \subseteq \RR^d$.     
	Throughout the whole paper, we assume that $\euX \coloneqq \interior\dom\varphi$ and $\euY \coloneqq \interior\dom\psi$ such that $(\interior\euF)\cap\euY\subseteq\overline{\euX}$, $(\interior\euF)\cap\euY\cap\euX\ne\varnothing$ and $\euG\cap\euY\ne\varnothing$. 
	Let us recall that the potential $U$ has the form $f+g$. We make the following assumptions on the functions $f$ and $g$, the Legendre functions $\varphi$ and $\psi$, and their associated Bregman divergences $D_\varphi$ and $D_\psi$. 
	\begin{assumption}\label{assum:smooth}\em
		The function $f\colon\RR^d\to\RR$ is
		(i) in $\Gamma_0(\RR^d)$, lower bounded and 
		differentiable (i.e., of $\scrC^1$) but may not admit a \emph{globally Lipschitz gradient}; 
		(ii) $\beta_f$-\emph{smooth} relative to $\varphi$. 
	\end{assumption}
	
	\begin{assumption}\label{assum:nonsmooth}\em
		The function $g\colon\RR^d\to\oRR$ is 
		(i) in $\Gamma_0(\RR^d)$, lower bounded and possibly \emph{nonsmooth}; either 
		(ii$^\dagger$) such that $\e^{-g}$ is \emph{integrable} with respect to the Lebesgue measure, or 
		(ii$^\ddagger$) \emph{Lipschitz}. 
	\end{assumption}
	
	\begin{assumption}\label{assum:mirror1}\em
		The function $\varphi\in\Gamma_0(\RR^d)$ is 
		(i) \emph{Legendre}; 
		(ii) of $\scrC^3$ on $\euX$; 
		(iii) \emph{supercoercive}, i.e., $\lim_{\|x\|\to+\infty} \varphi(x)/\|x\| = +\infty$; 	
		(iv) $M_\varphi$-modified self-concordant \citep[\textbf{(A1)}]{zhang2020wasserstein,li2022mirror}, i.e., there exists $M_\varphi\in\RP$ such that for any $ (x, \tx)\in\euX\times\euX $, 
		$\left|\kern-0.25ex\left|\kern-0.25ex\left| (\nabla^2\varphi(x))^{\sfrac12} - (\nabla^2\varphi(\tx))^{\sfrac12} \right|\kern-0.25ex\right|\kern-0.25ex\right|_{\mathrm{F}} \le \sqrt{M_\varphi} \norm{\nabla\varphi(x) - \nabla\varphi(\tx)}$, 
		where $\fronorm{\cdot}$ is the Frobenius norm. 
		(v) \emph{very strictly convex}, i.e., $\nabla^2\varphi(x)\in\bbS^d_{++}$ for all $x\in\euX\ne\varnothing$ \citep{bauschke2000dykstras}.  
	\end{assumption}
	
	\begin{assumption}\label{assum:mirror2}\em
		The function $\psi\in\Gamma_0(\RR^d)$ is 
		(i) \emph{Legendre}; 
		(ii) of $\scrC^3$ on $\euY$; 
		(iii) \emph{supercoercive}.
	\end{assumption}
	
	\begin{assumption}\label{assum:Bregman}\em
		The Bregman divergence associated with $\psi$ satisfy the following assumptions: 
		(i) $D_\psi$ is jointly convex, i.e., convex on $\RR^d\times\RR^d$;  
		(ii) $(\forall y\in\euY)\; D_\psi(y, \cdot)$ is strictly convex on $\euY$, continuous on $\euY$, and coercive, i.e., $(\forall y\in\euY)\; D_\psi(y,z)\to+\infty$ as $\|z\|\to+\infty$. 
	\end{assumption}
	\Cref{assum:smooth,assum:nonsmooth,assum:mirror1,assum:mirror2} are required for the convergence of the proposed algorithms. \Cref{assum:nonsmooth}(ii) is required for \eqref{eqn:densities} to be well-defined. 	
	\Cref{assum:Bregman} consists of the standard assumptions required for the well-posedness of the Bregman--Moreau envelopes and the Bregman proximity operators of $\psi$ \citep{bauschke2018regularizing}. 	
	\Cref{prop:approx} below implies that the densities \eqref{eqn:densities} are well-defined and as close to the target density $\pi$ as required when $\lambda$ is sufficiently small (in total variation distance). We also provide a computable error bound when evaluating exactly the expectation with respect to \eqref{eqn:densities} as opposed to the true target distribution $\pi$. 
	\begin{proposition}\label{prop:approx}
		Suppose that \Cref{assum:smooth,assum:nonsmooth,assum:mirror2,assum:Bregman} hold. Then the following statements hold. 		
		\begin{enumerate}[label=(\alph*)]
			\item Let $\lambda>0$. If either (i) \Cref{assum:nonsmooth}(ii$^\dagger$) holds or (ii) \Cref{assum:nonsmooth}(ii$^\ddagger$) holds and $\psi$ is $\rho$-strongly convex, then $\lpi\lambda\psi$ and $\rpi\lambda\psi$ define proper densities of probability measures on $\RR^d$, i.e., $\exp\left( -\lU{\lambda}{\psi}\right)$ and $\exp\left( -\rU{\lambda}{\psi}\right)$ are integrable w.r.t.~the Lebesgue measure. 
			\item $\pi_\lambda^\psi$ converges to $\pi$ as $\lambda\downarrow 0$, i.e., $\| \pi_\lambda^\psi - \pi\|_{\TV} \to 0$ as $\lambda\downarrow 0$. 
			\item If \Cref{assum:nonsmooth}(ii$^\ddagger$) holds and $\psi$ is $\rho$-strongly convex, then for all $\lambda > 0$, 
			$\| \pi_\lambda^\psi - \pi\|_{\TV} \le \lambda\|g\|_{\Lip}^2/\rho$. 
			In addition, for any $\pi$- and $\pi_\lambda^\psi$-integrable function $h\colon\RR^d\to\RR$, 
			\[
			\left|\Ex_{\pi_\lambda^\psi} h - \Ex_\pi h\right| \le \left( \e^{\lambda\|g\|_{\Lip}^2/\rho} - 1\right)\cdot \min\left\{\Ex_{\pi_\lambda^\psi}|h|, \Ex_\pi|h|\right\}.
			\]
		\end{enumerate}
	\end{proposition}
	All proofs are postponed to \Cref{sec:proofs}. 
	Next we show that the surrogate potentials \eqref{eqn:env_potentials} are indeed continuously differentiable approximations of $U$ under certain conditions. Their gradients and the conditions for them to be Lipschitz are also given. We also assert that the Bregman--Moreau envelopes $\env_{\lambda, g}^\psi(y)$ have desirable asymptotic behavior as $\lambda$ goes to $0$. 
	\begin{proposition}\label{prop:gradient}
		Suppose that \Cref{assum:nonsmooth,assum:mirror2,assum:Bregman} hold and $\lambda > 0$. The following statements hold. 
		\begin{enumerate}[label=(\alph*)]
			\item The left and right Bregman--Moreau envelopes are differentiable on $\euY$ and 
			\begin{equation}\label{eqn:grad_Bregman_Moreau}
				\nabla\lenv_{\lambda, g}^\psi(y) = \frac{1}{\lambda}\nabla^2\psi(y)\left(y - \lprox{\lambda, g}\psi(y) \right), 			
			\end{equation} 
			and 
			\begin{equation}
				\nabla\renv_{\lambda, g}^\psi(y) = \frac{1}{\lambda}\left(\nabla\psi(y) - \nabla\psi\left( \rprox{\lambda, g}\psi(y)\right)  \right), 
			\end{equation}
			for any $y\in\euY$, respectively. 
			\item If $D_\psi(y, \cdot)$ is continuous and convex on $\euY$ for all $y\in\euY$, and $\nabla\psi$ is Lipschitz on $\euY$, then $\nabla\lenv_{\lambda, g}^\psi$ and $\nabla\renv_{\lambda, g}^\psi$ are Lipschitz on $\euY$. 
			\item As $\lambda\downarrow0$, we have $\lenv_{\lambda, g}^\psi(y) \uparrow g(y)$ and $\renv_{\lambda, g}^\psi(y) \uparrow g(y)$ for all $y\in\euY$. 
		\end{enumerate}		
	\end{proposition}

	Finally, we make the following extra assumptions on $U_\lambda^\psi$, $\varphi$ and $\psi$. 
	\begin{assumption}\label{assum:functions}\em
		We assume the following: 
		For $\lambda>0$, 
		(i) $\lU\lambda\psi$ and $\rU\lambda\psi$ are $\alpha$-\emph{Legendre strongly convex} with respect to $\varphi$; 
		(ii) $\lenv_{\lambda, g}^\psi$ and $\renv_{\lambda, g}^\psi$ are $\beta_g$-\emph{smooth} relative to $\varphi$. 
	\end{assumption}
	\Cref{assum:functions} is a set of rather generic assumptions but gives us guidance to choose $\psi$ and $\varphi$. 
	Note that if $f$ is $\alpha$-\emph{Legendre strongly convex} with respect to $\varphi$, then \Cref{assum:functions}(i) is automatically satisfied.
	Also note that the constants $\alpha$ and $\beta_g$ can be different for the left and right versions of their corresponding quantities. 

	\begin{remark}\label{remark:1}
		Let us define $\beta\coloneqq\beta_f+\beta_g$. 
		\Cref{assum:smooth}(iv) and \Cref{assum:functions}(ii) implies $U_\lambda^\psi$ is $\beta$-smooth relative to $\varphi$. 
		Then, $U_\lambda^\psi$ satisfies (A2) and (A3) of \citet{li2022mirror}, which are required for the convergence of HRLMC. 
	\end{remark}

	We propose two mirror-Langevin algorithms which use different discretizations of the mirror-Langevin diffusion. 
	We give the details of the algorithm based on HRLMC in the main text. 
	The algorithm based on MLA, which is essentially \eqref{eqn:MLA} with $U$ replaced by $U_\lambda^\psi$, coined the Bregman--Moreau mirrorless mirror-Langevin algorithm (BMMMLA), is given in \Cref{sec:BMMMLA}.

	\subsection{The Bregman--Moreau Unadjusted Mirror-Langevin Algorithm}
	Let $x_0 \in \euY$. A discretization scheme for the case of composite potentials similar to HRLMC \eqref{eqn:HRLMC}, called the Bregman--Moreau unadjusted mirror-Langevin algorithm (BMUMLA), iterates, for $k\in\NN$, 
	\begin{equation}\label{eqn:BMUMLA}
		x_{k+1} = \nabla\varphi^*\Big(\nabla\varphi(x_k) - \gamma\nabla U_\lambda^\psi(x_k)
		\left. + \sqrt{2\gamma}\left[ \nabla^2 \varphi(x_k) \right]^{\sfrac{1}{2}}\xi_k \right). 
	\end{equation}
	More specifically, when $\varphi = \psi = \sfrac{\norm{\cdot}^2}2$, then we have $\nabla\varphi = \Id$, $\nabla\varphi^* = (\nabla\varphi)^{-1} = \Id$, $\nabla^2\varphi = I_d$, and $\left(\nabla \varphi + \lambda\partial g\right)^{-1} = (\Id + 	\lambda\partial g)^{-1} = \prox_{\lambda g}$, so that BMUMLA \eqref{eqn:BMUMLA} reduces to MYULA \citep{durmus2018efficient}. 
	Furthermore, letting $y_k = \nabla\varphi(x_k)$ for all $k\in\NN$, then the BMUMLA in the dual space $\nabla\varphi(\euX)$ takes the form
	\begin{equation}\label{eqn:BMUMLA_dual}
		y_{k+1} = y_k - \gamma\nabla U_\lambda^\psi\circ\nabla\varphi^*(y_k) 
		+ \sqrt{2\gamma}\left[ \nabla^2 \varphi^*(y_k) \right]^{\sfrac{1}{2}}\xi_k . 
	\end{equation}
	
	Recent results by \citet{li2022mirror} show that HRLMC indeed has a vanishing bias with the step size $\gamma$, as opposed to what was conjectured in \citet{zhang2020wasserstein}. 
	An advantage of applying HRLMC over MLA is that an exact simulator of the Brownian motion of varying covariance is not needed, which is usually approximated by inner loops of Euler--Maruyama discretization in practice \citep{ahn2021efficient}. It is however worth noting that the use of the Bregman--Moreau envelope still incurs bias in our proposed algorithms, but can be controlled via the smoothing parameter $\lambda$ (see \Cref{sec:conv}).

	\subsection{Reminiscence of Bregman Proximal Gradient Algorithm via Right BMUMLA}
	\label{subsec:same}
	The proposed BMUMLA can be simplified by specifying $\psi = \varphi$, regarding the iterates and the assumptions. 
	In particular, the right BMUMLA reduces to 
	\begin{equation}\label{eqn:RBMUMLA_special}
		\nabla\varphi\left( x_{k+1} \right)
		= \left(1 - \frac{\gamma}{\lambda}\right) \nabla\varphi(x_k) - \gamma\nabla f(x_k) 
		+ \frac{\gamma}{\lambda}\nabla\varphi\left( \rprox{\lambda, g}{\varphi}(x_k) \right)+ \sqrt{2\gamma}\left[ \nabla^2 \varphi(x_k) \right]^{\sfrac{1}{2}}\xi_k,
	\end{equation}
	which can be viewed as the generalization of MYULA \citep{durmus2018efficient} with the right Bregman--Moreau envelope, but with a  diffusion term of varying covariance. Furthermore, if we let $\gamma = \lambda$, then \eqref{eqn:RBMUMLA_special} becomes
	\begin{equation}\label{eqn:RBMUMLA_special2}
		x_{k+1} = \nabla\varphi^*\left( \nabla\varphi\left( \rprox{\lambda, g}{\varphi}(x_k) \right)  - \lambda\nabla f(x_k)
		+ \sqrt{2\gamma}\left[ \nabla^2 \varphi(x_k) \right]^{\sfrac{1}{2}}\xi_k \right), 
	\end{equation}
	which roughly resembles the iterates of the Bregman proximal gradient algorithm \citep{van2017forward,bauschke2017descent,bolte2018first,bui2021bregman,chizat2021convergence}, which takes the form 
	\begin{equation}
		x_{k+1} = \lprox{\lambda, g}{\varphi}\left( \nabla\varphi^*\left( \nabla\varphi(x_k) - \lambda\nabla f(x_k) \right)  \right), 
	\end{equation}
	or
	\begin{equation}\label{eqn:BPG}
		z_{k+1} = \nabla\varphi^*\left(\nabla\varphi\left( \lprox{\lambda, g}{\varphi}(z_k) \right)  - \lambda\nabla f\left( \lprox{\lambda, g}{\varphi}(z_k) \right) \right), 
	\end{equation}
	if we write $z_k = \nabla\varphi^*\left( \nabla\varphi(x_k) - \lambda\nabla f(x_k) \right)$. 
	The differences between \eqref{eqn:BPG} and \eqref{eqn:RBMUMLA_special2}, other than the diffusion term, are the use of different Bregman--Moreau envelopes and the argument of the gradient of the smooth part. 
	
	Another advantage of using the same mirror map is that \Cref{assum:functions} can be made more precise. In particular, regarding \Cref{assum:functions}(ii), since $\renv_{\lambda, g}^\varphi$ is $\lambda^{-1}$-smooth relative to $\varphi$ \citep[Proposition 3.8(ii)]{laude2020bregman}, for the right BMUMLA with $\psi=\varphi$, i.e., \eqref{eqn:RBMUMLA_special}, \Cref{assum:functions}(ii) is made precise with a relative smoothness constant $\beta_g = \lambda^{-1}$ for $\renv_{\lambda, g}^\varphi$.

	\section{Convergence Analysis}
	\label{sec:conv}
	We now state the main convergence results derived from \citet{li2022mirror}. To quantify the convergence, we introduce a modified Wasserstein distance previously introduced by \citet{zhang2020wasserstein} and further applied in the analysis of \citet{li2022mirror}. 
	\begin{definition}
		\label{def:wasserstein}
		For two probability measures $\mu$ and $\nu$ on $\euB(\euX)$, the (squared) \emph{modified Wasserstein distance} under the mirror map $\nabla\varphi$ from $\mu$ to $\nu$ is defined by 
		\[\sfW_{2, \varphi}^2(\mu, \nu) \coloneqq \inf_{u\sim\mu, v\sim\nu}\Ex \left[\|\nabla\varphi(u) - \nabla\varphi(v) \|^2\right]. \]
		Note that if $\tilde{\mu} \coloneqq (\nabla\varphi)_\sharp\mu$ and $\tilde{\nu} \coloneqq (\nabla\varphi)_\sharp\nu$ are the pushforward measures of $\mu$ and $\nu$ by $\nabla\varphi$ respectively, then 
		$\sfW_{2, \varphi}^2(\mu, \nu) = \sfW_2^2(\tilde{\mu}, \tilde{\nu}) \coloneqq \inf_{\tilde{u}\sim\tilde{\mu}, \tilde{v}\sim\tilde{\nu}}\Ex \left[\|\tilde{u} - \tilde{v} \|^2\right]$. 
	\end{definition}
	The main convergence result is given as follows. 
	\begin{theorem}\label{thm:conv}
		Let \Cref{assum:smooth,assum:nonsmooth,assum:mirror1,assum:mirror2,assum:Bregman,assum:functions} hold and $M_\varphi < \alpha/2$. 
		Let $x_k \sim \mu_k$ be the iterates of \eqref{eqn:BMUMLA} with step size $\gamma \in\left]0, \gamma_{\max}\right]$, where $\gamma_{\max} = \euO\left((\alpha-2M_\varphi)^2/\left( \beta^2(1+8M_\varphi)^2\right)\right)$. Then, from any $x_0\sim\mu_0$, we have 
		\begin{equation}\label{eqn:bound_W2_surrogate}
			\sfW_{2, \varphi}(\mu_k, \pi_\lambda^\psi) \le \sqrt{2}\e^{-(\alpha-2M_\varphi)\gamma k}\sfW_{2, \varphi}(\mu_0, \pi_\lambda^\psi) + C\sqrt{2\gamma},
		\end{equation}
		where $C = \euO\left(\beta(1+8M_\varphi)\sqrt{d}/(\alpha-2M_\varphi)\right)$ is a constant. 
		Furthermore, if the stronger \Cref{assum:nonsmooth}(ii$^\ddagger$) rather than (ii$^\dagger$) holds and $\psi$ is $\rho$-strongly convex, then 
		\begin{equation}\label{eqn:bound_W2}
			\sfW_{2, \varphi}(\mu_k, \pi) \le \sqrt{2}\e^{-(\alpha-2M_\varphi)\gamma k}\sfW_{2, \varphi}(\mu_0, \pi)
			+ C\sqrt{2\gamma} + \left(1 + \sqrt{2}\e^{-(\alpha-2M_\varphi)\gamma k}\right)\frac{\eta\lambda}{\rho}\|g\|_{\Lip}^2, 
		\end{equation}
		where $\eta\coloneqq\sup_{(u,v)\in\euX\times\euX}\|\nabla\varphi(u) - \nabla\varphi(v)\|^2$. 
	\end{theorem}
	From \Cref{thm:conv}, we can derive a mixing time bound for \eqref{eqn:BMUMLA} similar to Corollary 3.2 of \citet{li2022mirror}. 
	\begin{corollary}\label{cor:mixing}
		Suppose that \Cref{assum:smooth,assum:mirror1,assum:mirror2,assum:Bregman,assum:functions} and \Cref{assum:nonsmooth}(ii$^\ddagger$) rather than (ii$^\dagger$) hold, $\psi$ is $\rho$-strongly convex, and $M_\varphi < \alpha/2$. 
		Then, for any target accuracy $\varepsilon>0$, in order to achieve $\sfW_2 (\tilde{\mu}_k, \tilde{\pi}) \le \varepsilon$, it suffices to run BMUMLA in the dual space \eqref{eqn:BMUMLA_dual} with step size $\gamma = \varepsilon^2/(18C^2)$ and smoothing parameter $\lambda=\rho\varepsilon/(3\tilde{\eta}\|g\|_{\Lip}^2)$ for $k$ iterations, where 
		\begin{equation}\label{eqn:mixing_time}
			k\ge \frac{1}{(\alpha-2M_\varphi)\gamma}\log\left(\frac{3\sqrt{2}[\sfW_2 (\tilde{\mu}_0, \tilde{\pi}_\lambda^\psi) + \tilde{\eta}\lambda\|g\|_{\Lip}^2/\rho]}{\varepsilon}\right) 
			=\tilde{\euO}\left(\frac{\beta^2(1+8M_\varphi)^2d}{(\alpha-2M_\varphi)^3\varepsilon^2}\right), 
		\end{equation}
		where $\tilde{\eta}\coloneqq\sup_{(u,v)\in\euX\times\euX}\|u - v\|^2$.
\end{corollary}
Assuming all other constants including $\alpha$, $\beta$, $M_\varphi$, $\eta$, $\rho$ and $\|g\|_{\Lip}$ are independent of $d$, then, similar to \citet{li2022mirror} for the relatively smooth case, \eqref{eqn:BMUMLA} has a biased convergence guarantee with a bias incurred by the algorithm which scales as $\euO(\sqrt{d\gamma})$. Since essentially we are sampling from the surrogate distribution $\pi_\lambda^\psi$ which is different from $\pi$, this incurs an additional bias. From \eqref{eqn:bound_W2}, this bias attributed to smoothing with the Bregman--Moreau envelope scales as $\euO(\lambda)$ for large enough $k$ and $\eta<+\infty$. We then obtain the same mixing time bound of $\tilde{\euO}(d/\varepsilon^2)$ for BMUMLA as the one for MLA in \citet{li2022mirror}. Note that the appearance of $\eta$ limits the choice of mirror maps in \eqref{eqn:BMUMLA} as some choices of $\varphi$ might not give a bounded $\eta$ \citep[see][for related discussion]{jiang2021mirror}.

\begin{figure*}[t]
\centering
\hspace*{2mm}
\begin{subfigure}[b]{0.18\textwidth}
	\centering
	\includegraphics[width=\textwidth]{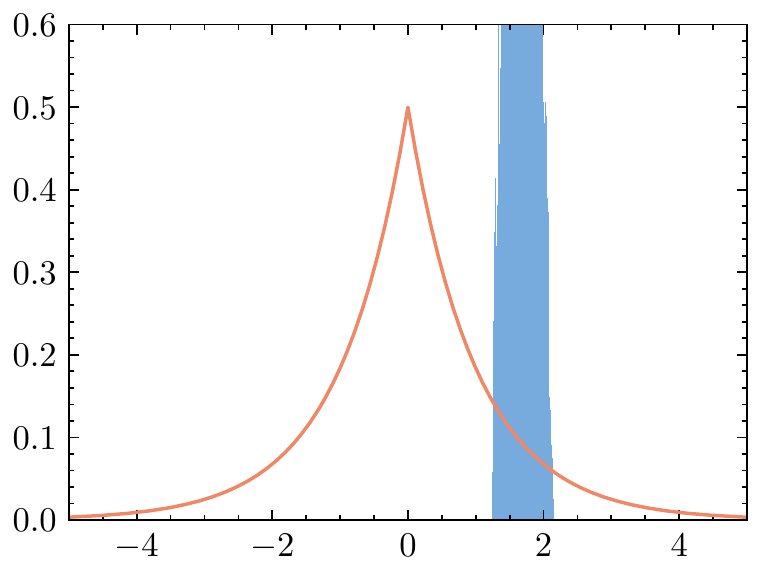}
	\label{fig:myula_1}
\end{subfigure}
\begin{subfigure}[b]{0.18\textwidth}
	\centering
	\includegraphics[width=\textwidth]{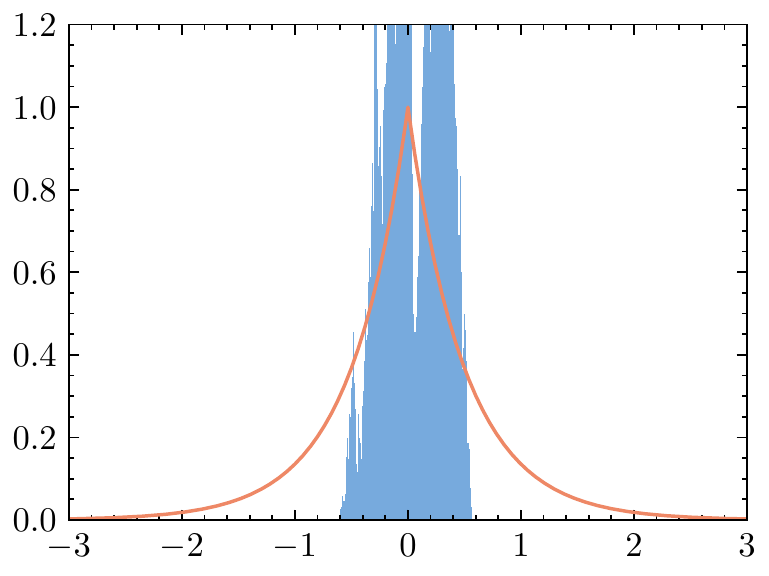}
	\label{fig:myula_2}
\end{subfigure}
\begin{subfigure}[b]{0.18\textwidth}
	\centering
	\includegraphics[width=\textwidth]{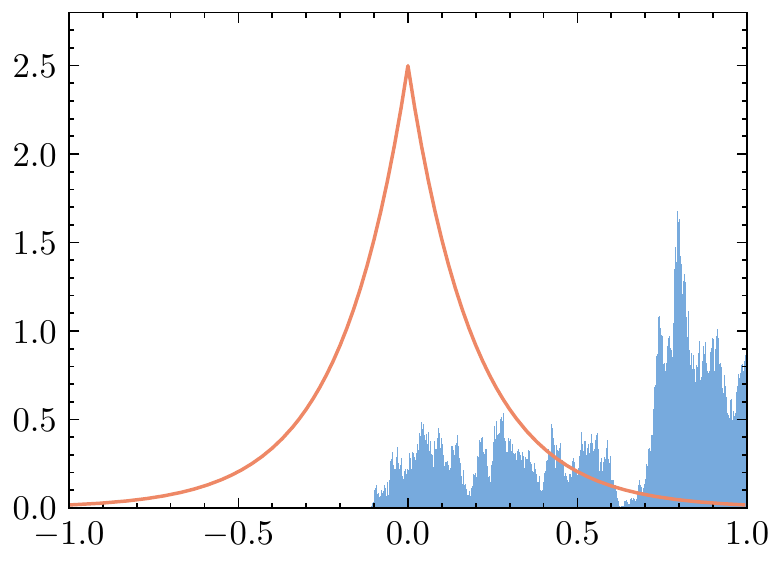}
	\label{fig:myula_5}
\end{subfigure}
\begin{subfigure}[b]{0.18\textwidth}
	\centering
	\includegraphics[width=\textwidth]{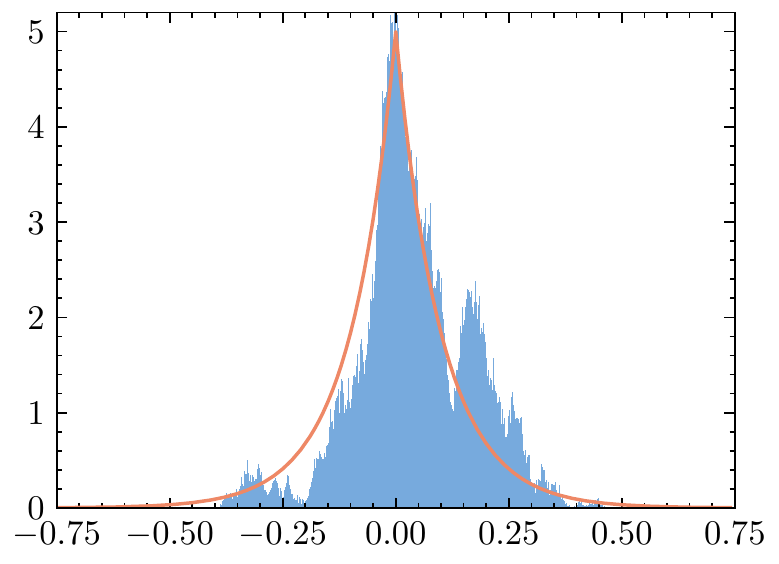}
	\label{fig:myula_10}
\end{subfigure}
\begin{subfigure}[b]{0.18\textwidth}
	\centering
	\includegraphics[width=\textwidth]{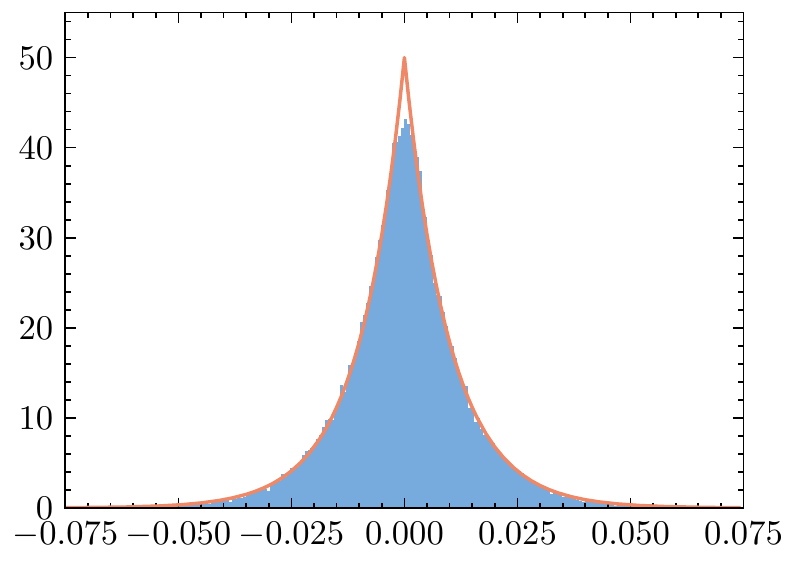}
	\label{fig:myula_100}
\end{subfigure}
\newline
\centering
\begin{subfigure}[b]{0.18\textwidth}
	\centering
	\includegraphics[width=\textwidth]{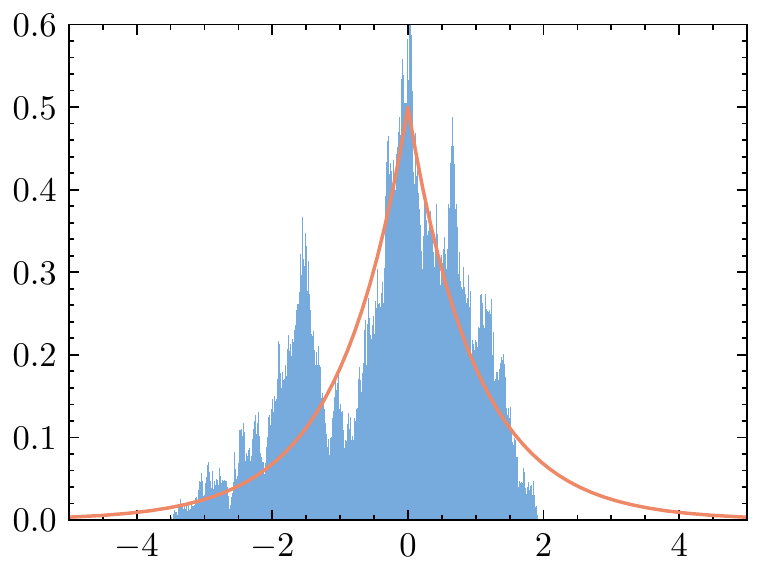}
	\caption{1\textsuperscript{st} dimension}
	\label{fig:bmmmla_1}
\end{subfigure}
\begin{subfigure}[b]{0.18\textwidth}
	\centering
	\includegraphics[width=\textwidth]{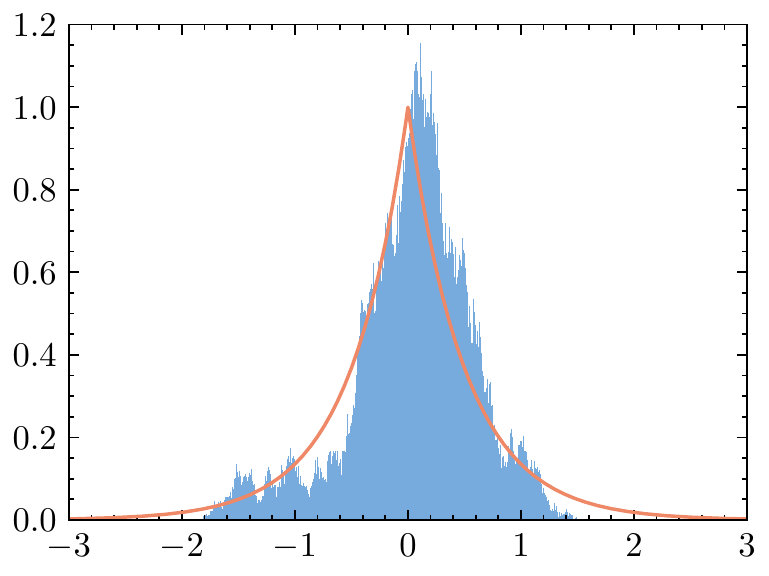}
	\caption{2\textsuperscript{nd} dimension}
	\label{fig:bmmmla_2}
\end{subfigure}
\begin{subfigure}[b]{0.18\textwidth}
	\centering
	\includegraphics[width=\textwidth]{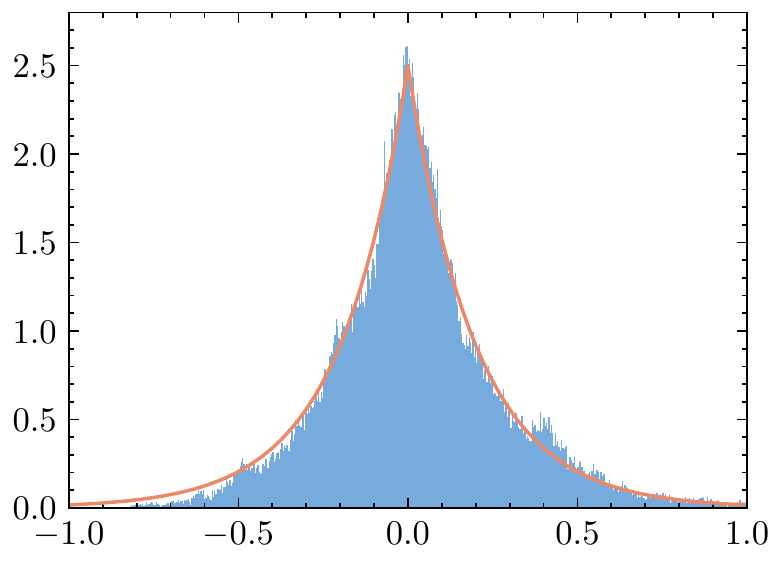}
	\caption{5\textsuperscript{th} dimension}
	\label{fig:bmmmla_5}
\end{subfigure}
\begin{subfigure}[b]{0.18\textwidth}
	\centering
	\includegraphics[width=\textwidth]{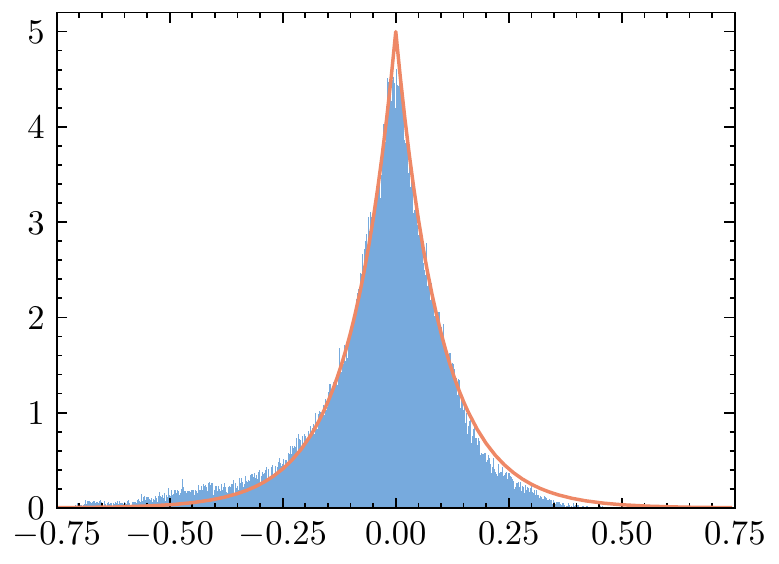}
	\caption{10\textsuperscript{th} dimension}
	\label{fig:bmmmla_10}
\end{subfigure}
\begin{subfigure}[b]{0.18\textwidth}
	\centering
	\includegraphics[width=\textwidth]{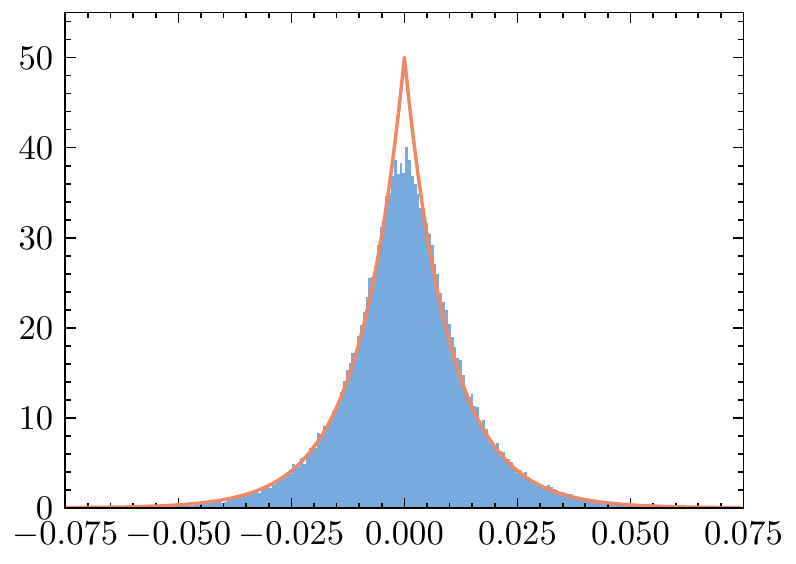}
	\caption{100\textsuperscript{th} dimension}
	\label{fig:bmumla_100}
\end{subfigure}
\caption{Histograms of samples (blue) from MYULA (1\textsuperscript{st} row), BMUMLA (2\textsuperscript{nd} row) and the true densities (orange). 
}
\label{fig:an_laplace}
\end{figure*}

\section{Numerical Experiments}	
\label{sec:expt}
We perform numerical experiments of sampling anisotropic Laplace distributions which have nonsmooth potentials. Other additional numerical experiments are given in \Cref{sec:add_expt}. In this section, we use bold lower case letters $\btheta = (\theta_i)_{1\le i\le d}^\top\in\RR^d$ to denote vectors. All numerical implementations can be found at \url{https://github.com/timlautk/bregman_prox_langevin_mc}.

For such a nonsmooth sampling task, inspired by \citet{vorstrup2021gradient,bouchard2018bouncy}, we consider the case where $f=0$ and $g(\btheta) = \onenorm{\balpha\odot\btheta} = \sumd \alpha_i|\theta_i|$ with $\balpha = (1,2,\ldots, d)^\top$. This is an example in which MYULA is known to perform poorly due to the \emph{anisotropy} \citep[\S4.1]{vorstrup2021gradient}: with a relatively small step size, MYULA mixes fast for the \emph{narrow} marginals, whereas it mixes slowly in the \emph{wide} ones. To alleviate this issue, the mirror map in our proposed scheme allows to adapt to the geometry of the potential, while the square root of the Hessian of $\varphi$ serves as a preconditioner of the diffusion term. 
We choose $\varphi$ to be the $\bbeta$-hyperbolic entropy  \citep[hypentropy;][]{ghai2020exponentiated}, defined by
\[
\varphi_{\bbeta}(\btheta) \coloneqq \sumd \left[\theta_i\arsinh\left(\theta_i/\beta_i\right) - \sqrt{\theta_i^2 + \beta_i^2}\right], 
\]
where $\btheta\in\RR^d$ and $\bbeta\in\RP^d$. We allow $\beta_i$'s to vary across different dimensions to enhance  flexibility. The hypentropy interpolates between the squared Euclidean distance and the Boltzmann--Shannon entropy as $\bbeta$ varies. We choose the associated Legendre function of the Bregman--Moreau envelope to be $\psi(\btheta) = \frac12\|\btheta\|_{\bM}^2 = \frac12\dotp{\btheta}{\bM\btheta}$, where $\bM = \Diag(\balpha/2)$, so that $\psi$ is strongly convex. 

We apply the proposed algorithms BMUMLA and BMMMLA, and compare their performance with that of MYULA. We consider 
$d=100$, draw $K=10^5$ samples, with a tight Bregman--Moreau envelope using a small smoothing parameter $\lambda= 10^{-5}$ and a small step size $\gamma = \lambda/2$. The parameter of the hyperbolic entropy is $\bbeta = (2\sqrt{d-i+1})_{1\le i\le d}^\top$. 
Further implementation details and verification of assumptions are given in \Cref{sec:details_expt}. 
The marginal empirical densities are given in \Cref{fig:an_laplace} (\Cref{fig:an_laplace_2} for BMMMLA in \Cref{sec:details_expt}).

In this example, MYULA does not mix fast for the wide marginals (the lower dimensions, even at the 10\textsuperscript{th} dimension), whereas BMUMLA and BMMMLA are able to mix equally fast across different dimensions. Although our proposed methods require knowledge of the target distribution, we expect even better mixing when $\bbeta$ is better tuned or adaptively learned using certain auxiliary procedures. 	
Moreover, a quick comparison with methods in \citet[Figure 2]{vorstrup2021gradient} indicates that, despite being asymptotically biased (since $\gamma$ and $\lambda$ are chosen as constants), our proposed algorithms also appear to be comparable to or even outperform some of the asymptotically exact algorithms such as pMALA \citep{pereyra2016proximal} and the bouncy particle sampler \citep{bouchard2018bouncy}. We however leave a comprehensive comparison with other classes of MCMC algorithms as future work.

\section{Discussion}
In this paper, we propose two efficient Bregman proximal Langevin algorithms for efficient sampling from nonsmooth convex composite potentials. Our proposed schemes enhance the flexibility of existing (overdamped) LMC algorithms in two aspects: the use of Bregman divergences in (i) altering the \emph{geometry} of the problem and hence the algorithm; (ii) imposing the smooth approximation. Theoretically, our proposed schemes have a vanishing bias with the step size and the smoothing parameter of the Bregman--Moreau envelope, while numerically they outperform MYULA in sampling nonsmooth anisotropic distributions. 

There are several interesting directions to extend the current work. Full gradients can be replaced by stochastic or mini-batch gradients in Langevin algorithms \citep[see e.g.,][]{welling2011bayesian,durmus2019analysis,salim2019stochastic,salim2020primal,nemeth2021stochastic} to avoid costly computation of the full gradient in high dimensions. 
Our proposed algorithm also has potential implications for nonconvex potentials or nonconvex optimization algorithms based on Langevin dynamics \citep{mangoubi2019nonconvex,cheng2018sharp,raginsky2017nonconvex,vempala2019rapid}, as the Bregman proximial gradient algorithm is able to solve nonconvex optimization algorithms \citep{bolte2018first}. We also refer to recent results of the Bregman--Moreau envelopes of nonconvex functions \citep{laude2020bregman} and the use of Moreau envelope in nonsmooth sampling algorithms 
for computing the Exponential Weighted Aggregation (EWA) estimators 
\citep{luu2021sampling}. Recently, \citet{jiang2021mirror} leverages the assumption of an isoperimetric inequality called the \emph{mirror log-Sobolev inequality} for the target density in mirror Langevin algorithms, which is weaker than assuming a Legendre strongly convex potential. 
It is however unclear to see how Bregman--Moreau envelopes in the potential would satisfy this assumption and other weaker notions of relative smoothness of the potential introduced in this paper. 
A natural extension is to consider sampling schemes based on the underdamped Langevin dynamics \citep{cheng2018underdamped} or Hamiltonian dynamics \citep{neal1993bayesian} with the Bregman--Moreau enveloped potentials, and to include the Metropolis--Hastings adjustment step to accelerate mixing. 
Other than the Bregman--Moreau envelope, the Bregman forward-backward envelope \citep{ahookhosh2021bregman} can also be used to envelop the whole composite potential; see the recent work by \citet{eftekhari2022forward} in a similar spirit using the forward-backward envelope with the overdamped Langevin algorithm. 	
More sophisticated discretization scheme such as the explicit stabilized SK-ROCK scheme \citep{abdulle2018optimal} in \citet{pereyra2020accelerating} could also constitute new sampling schemes based on MLD. 
It is also interesting to compare our proposed schemes with gradient-based MCMC algorithms based on piecewise-deterministic Markov processes for nonsmooth sampling as in \citet{vorstrup2021gradient}, e.g., the zig-zag sampler \citep{bierkens2019zig} and the bouncy particle sampler \citep{bouchard2018bouncy}.

\section*{Acknowledgements}
This work is partially supported by NIH R01LM01372201, NSF TRIPOD 1740735, NSF DMS1454377.
Part of this work was done when Tim Tsz-Kit Lau was participating in the workshop ``Sampling Algorithms and Geometries on Probability Distributions'' at the Simons Institute for the Theory of Computing.

\bibliographystyle{plainnat}
\bibliography{ref}

\begin{thebibliography}{82}
\providecommand{\natexlab}[1]{#1}
\providecommand{\url}[1]{\texttt{#1}}
\expandafter\ifx\csname urlstyle\endcsname\relax
  \providecommand{\doi}[1]{doi: #1}\else
  \providecommand{\doi}{doi: \begingroup \urlstyle{rm}\Url}\fi

\bibitem[Abdulle et~al.(2018)Abdulle, Almuslimani, and
  Vilmart]{abdulle2018optimal}
Assyr Abdulle, Ibrahim Almuslimani, and Gilles Vilmart.
\newblock Optimal explicit stabilized integrator of weak order 1 for stiff and
  ergodic stochastic differential equations.
\newblock \emph{SIAM/ASA Journal on Uncertainty Quantification}, 6\penalty0
  (2):\penalty0 937--964, 2018.

\bibitem[Ahn and Chewi(2021)]{ahn2021efficient}
Kwangjun Ahn and Sinho Chewi.
\newblock Efficient constrained sampling via the mirror-{L}angevin algorithm.
\newblock In \emph{Advances in Neural Information Processing Systems
  (NeurIPS)}, 2021.

\bibitem[Ahookhosh et~al.(2021)Ahookhosh, Themelis, and
  Patrinos]{ahookhosh2021bregman}
Masoud Ahookhosh, Andreas Themelis, and Panagiotis Patrinos.
\newblock A {B}regman forward-backward linesearch algorithm for nonconvex
  composite optimization: superlinear convergence to nonisolated local minima.
\newblock \emph{SIAM Journal on Optimization}, 31\penalty0 (1):\penalty0
  653--685, 2021.

\bibitem[Bauschke(2003)]{bauschke2003duality}
Heinz~H. Bauschke.
\newblock Duality for {B}regman projections onto translated cones and affine
  subspaces.
\newblock \emph{Journal of Approximation Theory}, 121\penalty0 (1):\penalty0
  1--12, 2003.

\bibitem[Bauschke and Borwein(1997)]{bauschke1997legendre}
Heinz~H. Bauschke and Jonathan~M. Borwein.
\newblock Legendre functions and the method of random {B}regman projections.
\newblock \emph{Journal of Convex Analysis}, 4\penalty0 (1):\penalty0 27--67,
  1997.

\bibitem[Bauschke and Borwein(2001)]{bauschke2001joint}
Heinz~H. Bauschke and Jonathan~M. Borwein.
\newblock Joint and separate convexity of the {B}regman distance.
\newblock In \emph{Studies in Computational Mathematics}, volume~8, pages
  23--36. Elsevier, 2001.

\bibitem[Bauschke and Combettes(2017)]{bauschke2017}
Heinz~H. Bauschke and Patrick~L. Combettes.
\newblock \emph{Convex Analysis and Monotone Operator Theory in Hilbert
  Spaces}.
\newblock Springer, 2nd edition, 2017.

\bibitem[Bauschke and Lewis(2000)]{bauschke2000dykstras}
Heinz~H. Bauschke and Adrian~S. Lewis.
\newblock Dykstras algorithm with {B}regman projections: A convergence proof.
\newblock \emph{Optimization}, 48\penalty0 (4):\penalty0 409--427, 2000.

\bibitem[Bauschke et~al.(2001)Bauschke, Borwein, and
  Combettes]{bauschke2001essential}
Heinz~H. Bauschke, Jonathan~M. Borwein, and Patrick~L. Combettes.
\newblock Essential smoothness, essential strict convexity, and {L}egendre
  functions in {B}anach spaces.
\newblock \emph{Communications in Contemporary Mathematics}, 3\penalty0
  (04):\penalty0 615--647, 2001.

\bibitem[Bauschke et~al.(2003)Bauschke, Borwein, and
  Combettes]{bauschke2003bregman}
Heinz~H. Bauschke, Jonathan~M. Borwein, and Patrick~L. Combettes.
\newblock Bregman monotone optimization algorithms.
\newblock \emph{SIAM Journal on Control and Optimization}, 42\penalty0
  (2):\penalty0 596--636, 2003.

\bibitem[Bauschke et~al.(2006)Bauschke, Combettes, and Noll]{bauschke2006joint}
Heinz~H. Bauschke, Patrick~L. Combettes, and Dominikus Noll.
\newblock Joint minimization with alternating {B}regman proximity operators.
\newblock \emph{Pacific Journal of Optimization}, 2:\penalty0 401--424, 2006.

\bibitem[Bauschke et~al.(2009)Bauschke, Wang, Ye, and
  Yuan]{bauschke2009bregman}
Heinz~H. Bauschke, Xianfu Wang, Jane Ye, and Xiaoming Yuan.
\newblock Bregman distances and {C}hebyshev sets.
\newblock \emph{Journal of Approximation Theory}, 159\penalty0 (1):\penalty0
  3--25, 2009.

\bibitem[Bauschke et~al.(2017)Bauschke, Bolte, and
  Teboulle]{bauschke2017descent}
Heinz~H. Bauschke, J{\'e}r{\^o}me Bolte, and Marc Teboulle.
\newblock A descent lemma beyond {L}ipschitz gradient continuity: first-order
  methods revisited and applications.
\newblock \emph{Mathematics of Operations Research}, 42\penalty0 (2):\penalty0
  330--348, 2017.

\bibitem[Bauschke et~al.(2018)Bauschke, Dao, and
  Lindstrom]{bauschke2018regularizing}
Heinz~H. Bauschke, Minh~N. Dao, and Scott~B. Lindstrom.
\newblock Regularizing with {Bregman--Moreau} envelopes.
\newblock \emph{SIAM Journal on Optimization}, 28\penalty0 (4):\penalty0
  3208--3228, 2018.

\bibitem[Bauschke et~al.(2019)Bauschke, Bolte, Chen, Teboulle, and
  Wang]{bauschke2019linear}
Heinz~H. Bauschke, J{\'e}r{\^o}me Bolte, Jiawei Chen, Marc Teboulle, and Xianfu
  Wang.
\newblock On linear convergence of non-{E}uclidean gradient methods without
  strong convexity and {L}ipschitz gradient continuity.
\newblock \emph{Journal of Optimization Theory and Applications}, 182\penalty0
  (3):\penalty0 1068--1087, 2019.

\bibitem[Bierkens et~al.(2019)Bierkens, Fearnhead, and
  Roberts]{bierkens2019zig}
Joris Bierkens, Paul Fearnhead, and Gareth Roberts.
\newblock The zig-zag process and super-efficient sampling for {B}ayesian
  analysis of big data.
\newblock \emph{The Annals of Statistics}, 47\penalty0 (3):\penalty0
  1288--1320, 2019.

\bibitem[Blondel et~al.(2020)Blondel, Martins, and
  Niculae]{blondel2020learning}
Mathieu Blondel, Andr\'e~F.T. Martins, and Vlad Niculae.
\newblock Learning with {Fenchel-Young} losses.
\newblock \emph{Journal of Machine Learning Research}, 21\penalty0
  (35):\penalty0 1--69, 2020.

\bibitem[Bolte et~al.(2018)Bolte, Sabach, Teboulle, and
  Vaisbourd]{bolte2018first}
J{\'e}r{\^o}me Bolte, Shoham Sabach, Marc Teboulle, and Yakov Vaisbourd.
\newblock First order methods beyond convexity and {L}ipschitz gradient
  continuity with applications to quadratic inverse problems.
\newblock \emph{SIAM Journal on Optimization}, 28\penalty0 (3):\penalty0
  2131--2151, 2018.

\bibitem[Bouchard-C{\^o}t{\'e} et~al.(2018)Bouchard-C{\^o}t{\'e}, Vollmer, and
  Doucet]{bouchard2018bouncy}
Alexandre Bouchard-C{\^o}t{\'e}, Sebastian~J. Vollmer, and Arnaud Doucet.
\newblock The bouncy particle sampler: A nonreversible rejection-free {M}arkov
  chain {M}onte {C}arlo method.
\newblock \emph{Journal of the American Statistical Association}, 113\penalty0
  (522):\penalty0 855--867, 2018.

\bibitem[Bregman(1967)]{bregman1967}
Lev~M. Bregman.
\newblock The relaxation method of finding the common point of convex sets and
  its application to the solution of problems in convex programming.
\newblock \emph{USSR Computational Mathematics and Mathematical Physics},
  7\penalty0 (3):\penalty0 200--217, 1967.

\bibitem[Brosse et~al.(2017)Brosse, Durmus, Moulines, and
  Pereyra]{brosse2017sampling}
Nicolas Brosse, Alain Durmus, {\'E}ric Moulines, and Marcelo Pereyra.
\newblock Sampling from a log-concave distribution with compact support with
  proximal {Langevin Monte Carlo}.
\newblock In \emph{Proceedings of the Conference on Learning Theory (COLT)},
  2017.

\bibitem[Bubeck(2015)]{bubeck2015convex}
S{\'e}bastien Bubeck.
\newblock Convex optimization: Algorithms and complexity.
\newblock \emph{Foundations and Trends{\textregistered} in Machine Learning},
  8\penalty0 (3-4):\penalty0 231--357, 2015.

\bibitem[Bubeck et~al.(2018)Bubeck, Eldan, and Lehec]{bubeck2018sampling}
S{\'e}bastien Bubeck, Ronen Eldan, and Joseph Lehec.
\newblock Sampling from a log-concave distribution with projected {Langevin
  Monte Carlo}.
\newblock \emph{Discrete \& Computational Geometry}, 59\penalty0 (4):\penalty0
  757--783, 2018.

\bibitem[B\`ui and Combettes(2021)]{bui2021bregman}
Minh~N. B\`ui and Patrick~L. Combettes.
\newblock Bregman forward-backward operator splitting.
\newblock \emph{Set-Valued and Variational Analysis}, 29\penalty0 (3):\penalty0
  583--603, 2021.

\bibitem[Chatterji et~al.(2020)Chatterji, Diakonikolas, Jordan, and
  Bartlett]{chatterji2020langevin}
Niladri Chatterji, Jelena Diakonikolas, Michael~I. Jordan, and Peter~L.
  Bartlett.
\newblock {Langevin Monte Carlo} without smoothness.
\newblock In \emph{Proceedings of the International Conference on Artificial
  Intelligence and Statistics (AISTATS)}, 2020.

\bibitem[Chen et~al.(2012)Chen, Kan, and Song]{chen2012moreau}
Ying~Ying Chen, Chao Kan, and Wen Song.
\newblock The {M}oreau envelope function and proximal mapping with respect to
  the {B}regman distances in {B}anach spaces.
\newblock \emph{Vietnam Journal of Mathematics}, 40\penalty0 (2\&3):\penalty0
  181--199, 2012.

\bibitem[Cheng et~al.(2018{\natexlab{a}})Cheng, Chatterji, Abbasi-Yadkori,
  Bartlett, and Jordan]{cheng2018sharp}
Xiang Cheng, Niladri~S. Chatterji, Yasin Abbasi-Yadkori, Peter~L. Bartlett, and
  Michael~I. Jordan.
\newblock Sharp convergence rates for {L}angevin dynamics in the nonconvex
  setting.
\newblock \emph{arXiv preprint arXiv:1805.01648v4}, 2018{\natexlab{a}}.

\bibitem[Cheng et~al.(2018{\natexlab{b}})Cheng, Chatterji, Bartlett, and
  Jordan]{cheng2018underdamped}
Xiang Cheng, Niladri~S. Chatterji, Peter~L. Bartlett, and Michael~I. Jordan.
\newblock Underdamped {L}angevin {MCMC}: A non-asymptotic analysis.
\newblock In \emph{Proceedings of the Conference on Learning Theory (COLT)},
  2018{\natexlab{b}}.

\bibitem[Chewi et~al.(2020)Chewi, Gouic, Lu, Maunu, Rigollet, and
  Stromme]{chewi2020exponential}
Sinho Chewi, Thibaut~Le Gouic, Chen Lu, Tyler Maunu, Philippe Rigollet, and
  Austin Stromme.
\newblock Exponential ergodicity of mirror-{L}angevin diffusions.
\newblock In \emph{Advances in Neural Information Processing Systems
  (NeurIPS)}, 2020.

\bibitem[Chizat(2021)]{chizat2021convergence}
L{\'e}na{\"\i}c Chizat.
\newblock Convergence rates of gradient methods for convex optimization in the
  space of measures.
\newblock \emph{arXiv preprint arXiv:2105.08368}, 2021.

\bibitem[Cordero-Erausquin(2017)]{cordero2017transport}
Dario Cordero-Erausquin.
\newblock Transport inequalities for log-concave measures, quantitative forms,
  and applications.
\newblock \emph{Canadian Journal of Mathematics}, 69\penalty0 (3):\penalty0
  481--501, 2017.

\bibitem[Corless et~al.(1996)Corless, Gonnet, Hare, Jeffrey, and
  Knuth]{corless1996lambertw}
Robert~M. Corless, Gaston~H. Gonnet, David~E.G. Hare, David~J. Jeffrey, and
  Donald~E. Knuth.
\newblock On the lambert {W} function.
\newblock \emph{Advances in Computational Mathematics}, 5\penalty0
  (1):\penalty0 329--359, 1996.

\bibitem[Dalalyan(2017{\natexlab{a}})]{dalalyan2017further}
Arnak~S. Dalalyan.
\newblock Further and stronger analogy between sampling and optimization:
  {L}angevin {M}onte {C}arlo and gradient descent.
\newblock In \emph{Proceedings of the Conference on Learning Theory (COLT)},
  2017{\natexlab{a}}.

\bibitem[Dalalyan(2017{\natexlab{b}})]{dalalyan2017theoretical}
Arnak~S. Dalalyan.
\newblock Theoretical guarantees for approximate sampling from smooth and
  log-concave densities.
\newblock \emph{Journal of the Royal Statistical Society: Series B (Statistical
  Methodology)}, 3\penalty0 (79):\penalty0 651--676, 2017{\natexlab{b}}.

\bibitem[Dragomir et~al.(2021{\natexlab{a}})Dragomir, d'Aspremont, and
  Bolte]{dragomir2021quartic}
Radu-Alexandru Dragomir, Alexandre d'Aspremont, and J{\'e}r{\^o}me Bolte.
\newblock Quartic first-order methods for low-rank minimization.
\newblock \emph{Journal of Optimization Theory and Applications}, 189\penalty0
  (2):\penalty0 341--363, 2021{\natexlab{a}}.

\bibitem[Dragomir et~al.(2021{\natexlab{b}})Dragomir, Taylor, d'Aspremont, and
  Bolte]{dragomir2021optimal}
Radu-Alexandru Dragomir, Adrien~B. Taylor, Alexandre d'Aspremont, and
  J{\'e}r{\^o}me Bolte.
\newblock Optimal complexity and certification of {B}regman first-order
  methods.
\newblock \emph{Mathematical Programming}, pages 1--43, 2021{\natexlab{b}}.

\bibitem[Durmus and Moulines(2017)]{durmus2017nonasymptotic}
Alain Durmus and Eric Moulines.
\newblock Nonasymptotic convergence analysis for the unadjusted {L}angevin
  algorithm.
\newblock \emph{The Annals of Applied Probability}, 27\penalty0 (3):\penalty0
  1551--1587, 2017.

\bibitem[Durmus and Moulines(2019)]{durmus2019high}
Alain Durmus and Eric Moulines.
\newblock High-dimensional {B}ayesian inference via the unadjusted {L}angevin
  algorithm.
\newblock \emph{Bernoulli}, 25\penalty0 (4A):\penalty0 2854--2882, 2019.

\bibitem[Durmus et~al.(2018)Durmus, Moulines, and Pereyra]{durmus2018efficient}
Alain Durmus, Eric Moulines, and Marcelo Pereyra.
\newblock Efficient {B}ayesian computation by proximal {M}arkov chain {M}onte
  {C}arlo: when {L}angevin meets {M}oreau.
\newblock \emph{SIAM Journal on Imaging Sciences}, 11\penalty0 (1):\penalty0
  473--506, 2018.

\bibitem[Durmus et~al.(2019)Durmus, Majewski, and
  Miasojedow]{durmus2019analysis}
Alain Durmus, Szymon Majewski, and B{{\l}}a{\.{z}}ej Miasojedow.
\newblock Analysis of {Langevin Monte Carlo} via convex optimization.
\newblock \emph{Journal of Machine Learning Research}, 20\penalty0
  (73):\penalty0 1--46, 2019.

\bibitem[Eftekhari et~al.(2022)Eftekhari, Vargas, and
  Zygalakis]{eftekhari2022forward}
Armin Eftekhari, Luis Vargas, and Konstantinos Zygalakis.
\newblock The forward-backward envelope for sampling with the overdamped
  {L}angevin algorithm.
\newblock \emph{arXiv preprint arXiv:2201.09096}, 2022.

\bibitem[Ghai et~al.(2020)Ghai, Hazan, and Singer]{ghai2020exponentiated}
Udaya Ghai, Elad Hazan, and Yoram Singer.
\newblock Exponentiated gradient meets gradient descent.
\newblock In \emph{Proceedings of the International Conference on Algorithmic
  Learning Theory (ALT)}, 2020.

\bibitem[Gibbs and Su(2002)]{gibbs2002choosing}
Alison~L. Gibbs and Francis~Edward Su.
\newblock On choosing and bounding probability metrics.
\newblock \emph{International Statistical Review}, 70\penalty0 (3):\penalty0
  419--435, 2002.

\bibitem[Girolami and Calderhead(2011)]{girolami2011riemann}
Mark Girolami and Ben Calderhead.
\newblock Riemann manifold {L}angevin and {H}amiltonian {M}onte {C}arlo
  methods.
\newblock \emph{Journal of the Royal Statistical Society: Series B (Statistical
  Methodology)}, 73\penalty0 (2):\penalty0 123--214, 2011.

\bibitem[Gunasekar et~al.(2021)Gunasekar, Woodworth, and
  Srebro]{gunasekar2021mirrorless}
Suriya Gunasekar, Blake Woodworth, and Nathan Srebro.
\newblock Mirrorless mirror descent: A more natural discretization of
  {R}iemannian gradient flow.
\newblock In \emph{Proceedings of the International Conference on Artificial
  Intelligence and Statistics (AISTATS)}, 2021.

\bibitem[Hanzely et~al.(2021)Hanzely, Richtarik, and
  Xiao]{hanzely2021accelerated}
Filip Hanzely, Peter Richtarik, and Lin Xiao.
\newblock Accelerated {B}regman proximal gradient methods for relatively smooth
  convex optimization.
\newblock \emph{Computational Optimization and Applications}, 79\penalty0
  (2):\penalty0 405--440, 2021.

\bibitem[Hsieh et~al.(2018)Hsieh, Kavis, Rolland, and
  Cevher]{hsieh2018mirrored}
Ya-Ping Hsieh, Ali Kavis, Paul Rolland, and Volkan Cevher.
\newblock Mirrored {L}angevin dynamics.
\newblock In \emph{Advances in Neural Information Processing Systems
  (NeurIPS)}, 2018.

\bibitem[Jiang(2021)]{jiang2021mirror}
Qijia Jiang.
\newblock Mirror {L}angevin {M}onte {Carlo}: the case under isoperimetry.
\newblock In \emph{Advances in Neural Information Processing Systems
  (NeurIPS)}, 2021.

\bibitem[Kan and Song(2012)]{kan2012moreau}
Chao Kan and Wen Song.
\newblock The {M}oreau envelope function and proximal mapping in the sense of
  the {B}regman distance.
\newblock \emph{Nonlinear Analysis: Theory, Methods \& Applications},
  75\penalty0 (3):\penalty0 1385--1399, 2012.

\bibitem[Lambert(1758)]{lambert1758observationes}
Johann~Heinrich Lambert.
\newblock Observationes variae in mathesin puram.
\newblock \emph{Acta Helvetica}, 3:\penalty0 128--168, 1758.

\bibitem[Laude et~al.(2020)Laude, Ochs, and Cremers]{laude2020bregman}
Emanuel Laude, Peter Ochs, and Daniel Cremers.
\newblock Bregman proximal mappings and {Bregman--Moreau} envelopes under
  relative prox-regularity.
\newblock \emph{Journal of Optimization Theory and Applications}, 184\penalty0
  (3):\penalty0 724--761, 2020.

\bibitem[Lee et~al.(2021)Lee, Shen, and Tian]{lee2021structured}
Yin~Tat Lee, Ruoqi Shen, and Kevin Tian.
\newblock Structured logconcave sampling with a restricted {G}aussian oracle.
\newblock In \emph{Proceedings of the Conference on Learning Theory (COLT)},
  2021.

\bibitem[Lehec(2021)]{lehec2021langevin}
Joseph Lehec.
\newblock The {Langevin Monte Carlo} algorithm in the non-smooth log-concave
  case.
\newblock \emph{arXiv preprint arXiv:2101.10695}, 2021.

\bibitem[Li et~al.(2022)Li, Tao, Vempala, and Wibisono]{li2022mirror}
Ruilin Li, Molei Tao, Santosh~S. Vempala, and Andre Wibisono.
\newblock The mirror {L}angevin algorithm converges with vanishing bias.
\newblock In \emph{Proceedings of the International Conference on Algorithmic
  Learning Theory (ALT)}, 2022.

\bibitem[Liang and Chen(2021)]{liang2021proximal}
Jiaming Liang and Yongxin Chen.
\newblock A proximal algorithm for sampling from non-smooth potentials.
\newblock \emph{arXiv preprint arXiv:2110.04597}, 2021.

\bibitem[Luu et~al.(2021)Luu, Fadili, and Chesneau]{luu2021sampling}
Tung~Duy Luu, Jalal Fadili, and Christophe Chesneau.
\newblock Sampling from non-smooth distributions through {L}angevin diffusion.
\newblock \emph{Methodology and Computing in Applied Probability}, 23\penalty0
  (4):\penalty0 1173--1201, 2021.

\bibitem[Mangoubi and Vishnoi(2019)]{mangoubi2019nonconvex}
Oren Mangoubi and Nisheeth~K. Vishnoi.
\newblock Nonconvex sampling with the {M}etropolis-adjusted {L}angevin
  algorithm.
\newblock In \emph{Proceedings of the Conference on Learning Theory (COLT)},
  2019.

\bibitem[Moreau(1962)]{moreau1962fonctions}
Jean-Jacques Moreau.
\newblock Fonctions convexes duales et points proximaux dans un espace
  hilbertien.
\newblock \emph{Comptes Rendus Hebdomadaires des S{\'e}ances de l'Acad{\'e}mie
  des Sciences}, 255:\penalty0 2897--2899, 1962.

\bibitem[Moreau(1965)]{moreau1965proximite}
Jean-Jacques Moreau.
\newblock Proximit{\'e} et dualit{\'e} dans un espace hilbertien.
\newblock \emph{Bulletin de la Soci{\'e}t{\'e} Math{\'e}matique de France},
  93:\penalty0 273--299, 1965.

\bibitem[Mou et~al.(2019)Mou, Flammarion, Wainwright, and
  Bartlett]{mou2019efficient}
Wenlong Mou, Nicolas Flammarion, Martin~J. Wainwright, and Peter~L. Bartlett.
\newblock An efficient sampling algorithm for non-smooth composite potentials.
\newblock \emph{arXiv preprint arXiv:1910.00551}, 2019.

\bibitem[Neal(1993)]{neal1993bayesian}
Radford~M. Neal.
\newblock Bayesian learning via stochastic dynamics.
\newblock In \emph{Advances in Neural Information Processing Systems
  (NeurIPS)}, 1993.

\bibitem[Nemeth and Fearnhead(2021)]{nemeth2021stochastic}
Christopher Nemeth and Paul Fearnhead.
\newblock Stochastic gradient {Markov chain Monte Carlo}.
\newblock \emph{Journal of the American Statistical Association}, 116\penalty0
  (533):\penalty0 433--450, 2021.

\bibitem[Nemirovski(1979)]{nemirovski1979efficient}
Arkadi~S. Nemirovski.
\newblock Efficient methods for large-scale convex optimization problems.
\newblock \emph{Ekonomika i Matematicheskie Metody}, 15\penalty0 (1), 1979.

\bibitem[Nemirovski and Yudin(1983)]{nemirovskij1983problem}
Arkadi~S. Nemirovski and David~B. Yudin.
\newblock \emph{Problem Complexity and Method Efficiency in Optimization}.
\newblock John Wiley \& Sons, 1983.

\bibitem[Nesterov(2018)]{nesterov2018lectures}
Yurii Nesterov.
\newblock \emph{Lectures on Convex Optimization}, volume 137 of \emph{Springer
  Optimization and Its Applications}.
\newblock Springer, 2nd edition, 2018.

\bibitem[Pereyra(2016)]{pereyra2016proximal}
Marcelo Pereyra.
\newblock Proximal {M}arkov chain {M}onte {C}arlo algorithms.
\newblock \emph{Statistics and Computing}, 26\penalty0 (4):\penalty0 745--760,
  2016.

\bibitem[Pereyra et~al.(2020)Pereyra, Mieles, and
  Zygalakis]{pereyra2020accelerating}
Marcelo Pereyra, Luis~Vargas Mieles, and Konstantinos~C. Zygalakis.
\newblock Accelerating proximal {Markov chain Monte Carlo} by using an explicit
  stabilized method.
\newblock \emph{SIAM Journal on Imaging Sciences}, 13\penalty0 (2):\penalty0
  905--935, 2020.

\bibitem[Raginsky et~al.(2017)Raginsky, Rakhlin, and
  Telgarsky]{raginsky2017nonconvex}
Maxim Raginsky, Alexander Rakhlin, and Matus Telgarsky.
\newblock Non-convex learning via stochastic gradient {L}angevin dynamics: a
  nonasymptotic analysis.
\newblock In \emph{Proceedings of the Conference on Learning Theory (COLT)},
  2017.

\bibitem[Roberts and Tweedie(1996)]{roberts1996exponential}
Gareth~O. Roberts and Richard~L. Tweedie.
\newblock Exponential convergence of {L}angevin distributions and their
  discrete approximations.
\newblock \emph{Bernoulli}, 2\penalty0 (4):\penalty0 341--363, 1996.

\bibitem[Rockafellar(1970)]{rockafellar1970}
Ralph~Tyrell Rockafellar.
\newblock \emph{Convex Analysis}.
\newblock Princeton University Press, Princeton, NJ, 1970.

\bibitem[Rockafellar and Wets(1998)]{rockafellar1998}
Ralph~Tyrell Rockafellar and Roger J.-B. Wets.
\newblock \emph{Variational Analysis}.
\newblock Springer, 1998.

\bibitem[Salim and Richt\'arik(2020)]{salim2020primal}
Adil Salim and Peter Richt\'arik.
\newblock Primal dual interpretation of the proximal stochastic gradient
  {L}angevin algorithm.
\newblock In \emph{Advances in Neural Information Processing Systems
  (NeurIPS)}, 2020.

\bibitem[Salim et~al.(2019)Salim, Kovalev, and
  Richt{\'a}rik]{salim2019stochastic}
Adil Salim, Dmitry Kovalev, and Peter Richt{\'a}rik.
\newblock Stochastic proximal {L}angevin algorithm: Potential splitting and
  nonasymptotic rates.
\newblock In \emph{Advances in Neural Information Processing Systems
  (NeurIPS)}, 2019.

\bibitem[Soueycatt et~al.(2020)Soueycatt, Mohammad, and
  Hamwi]{soueycatt2020regularization}
Mohamed Soueycatt, Yara Mohammad, and Yamar Hamwi.
\newblock Regularization in {B}anach spaces with respect to the {B}regman
  distance.
\newblock \emph{Journal of Optimization Theory and Applications}, 185\penalty0
  (2):\penalty0 327--342, 2020.

\bibitem[Takahashi et~al.(2021)Takahashi, Fukuda, and Tanaka]{takahashi2021new}
Shota Takahashi, Mituhiro Fukuda, and Mirai Tanaka.
\newblock New {B}regman proximal type algorithms for solving {DC} optimization
  problems.
\newblock \emph{arXiv preprint arXiv:2105.04873}, 2021.

\bibitem[Teboulle(2018)]{teboulle2018simplified}
Marc Teboulle.
\newblock A simplified view of first order methods for optimization.
\newblock \emph{Mathematical Programming}, 170\penalty0 (1):\penalty0 67--96,
  2018.

\bibitem[Van~Nguyen(2017)]{van2017forward}
Quang Van~Nguyen.
\newblock Forward-backward splitting with {B}regman distances.
\newblock \emph{Vietnam Journal of Mathematics}, 45\penalty0 (3):\penalty0
  519--539, 2017.

\bibitem[Vempala and Wibisono(2019)]{vempala2019rapid}
Santosh~S. Vempala and Andre Wibisono.
\newblock Rapid convergence of the unadjusted {L}angevin algorithm:
  Isoperimetry suffices.
\newblock In \emph{Advances in Neural Information Processing Systems
  (NeurIPS)}, 2019.

\bibitem[Vorstrup~Goldman et~al.(2021)Vorstrup~Goldman, Sell, and
  Singh]{vorstrup2021gradient}
Jacob Vorstrup~Goldman, Torben Sell, and Sumeetpal~Sidhu Singh.
\newblock Gradient-based {M}arkov chain {M}onte {C}arlo for {B}ayesian
  inference with non-differentiable priors.
\newblock \emph{Journal of the American Statistical Association}, pages 1--12,
  2021.

\bibitem[Welling and Teh(2011)]{welling2011bayesian}
Max Welling and Yee~Whye Teh.
\newblock Bayesian learning via stochastic gradient {L}angevin dynamics.
\newblock In \emph{Proceedings of the International Conference on Machine
  Learning (ICML)}, 2011.

\bibitem[Wibisono(2019)]{wibisono2019proximal}
Andre Wibisono.
\newblock Proximal {L}angevin algorithm: Rapid convergence under isoperimetry.
\newblock \emph{arXiv preprint arXiv:1911.01469}, 2019.

\bibitem[Zhang et~al.(2020)Zhang, Peyr{\'e}, Fadili, and
  Pereyra]{zhang2020wasserstein}
Kelvin~Shuangjian Zhang, Gabriel Peyr{\'e}, Jalal Fadili, and Marcelo Pereyra.
\newblock Wasserstein control of mirror {Langevin Monte Carlo}.
\newblock In \emph{Proceedings of the Conference on Learning Theory (COLT)},
  2020.

\end{thebibliography}

\newpage
\onecolumn
\appendix
\numberwithin{equation}{section}
\numberwithin{figure}{section}

\begin{center}
	{\LARGE \textsc{Appendix}}
\end{center}

\section{Proofs of Main Text}
\label{sec:proofs}

\subsection{Proof of \Cref{prop:approx}}
\Cref{prop:approx} includes statements similar to those in \citet[Proposition 3.1]{durmus2018efficient} and \citet[Theorem 3]{vorstrup2021gradient}. We provide the proofs here for self-containedness. In particular, we further need $\psi$ to be $\rho$-strongly convex in (c). 
\begin{proposition}
	\label{prop:strong_convex}
	Let $\psi$ be a Legendre function and $\rho$-strongly convex ($\rho>0$), then
	\begin{equation}\label{eqn:strong_convex}
		(\forall (y, \ty)\in\euY\times\euY)\quad \frac\rho2\|y-\ty\|^2\le D_\psi(y,\ty). 
	\end{equation}
	\begin{proof}[Proof of \Cref{prop:strong_convex}]
		By definition, $\psi$ is $\rho$-strongly convex if and only if 
		\begin{multline*}
			(\forall (y, \ty)\in\euY\times\euY)\quad \psi(y) \ge \psi(\ty) + \dotp{\nabla\psi(\ty)}{y-\ty}+\frac\rho2\|y-\ty\|^2\\
			\Leftrightarrow (\forall (y, \ty)\in\euY\times\euY)\quad \psi(y) - \psi(\ty) - \dotp{\nabla\psi(\ty)}{y-\ty}\ge \frac\rho2\|y-\ty\|^2.
		\end{multline*}
		Then the result follows from the definition of the Bregman divergence \eqref{eqn:bregman_div}. 

	\end{proof}
\end{proposition}

\begin{proof}[Proof of \Cref{prop:approx}]\hfill
	\begin{enumerate}[label=(\alph*)]
		\item 
		\begin{enumerate}[label=(\roman*)]		
			\item We first suppose that \Cref{assum:nonsmooth}(ii$^\dagger$) holds. By \citet[Proposition 2.2]{bauschke2018regularizing}, $U \ge U_\lambda^\psi$, which implies 
			\[0 < \int_{\RR^d} \e^{-U(x)}\,\diff x < \int_{\RR^d} \e^{-U_\lambda^\psi(x)}\,\diff x. \]
			It suffices to prove that $\e^{-\env_{\lambda, g}^\psi}$ is integrable (with respect to the Lebesgue measure) which in turn implies $\e^{-U_\lambda^\psi}$ is integrable since $f$ is lower bounded. 	
			By \Cref{assum:nonsmooth}(i) and \citet[Lemma A.1]{durmus2018efficient}, there exist $\rho_g >0$, $x_g\in\RR^d$ and $M_1\in\RR$ such that for all $x\in\RR^d$, 
			\[g(x) - g(x_g) \ge M_1 + \rho_g\|x-x_g\|. \]
			Then, by \Cref{def:bregman_env,def:bregman_prox}, for any $x\in\RR^d$, we have
			\begin{align}\label{eqn:bound_lenv}
				\lenv_{\lambda, g}^\psi(x) - g(x_g) &= g\left(\lprox{\lambda, g}{\psi}(x)\right) - g(x_g) + \frac{1}{\lambda}D_\psi\left(\lprox{\lambda, g}{\psi}(x), x\right) \nonumber\\
				&\ge M_1 + \rho_g\left\| \lprox{\lambda, g}{\psi}(x) - x_g\right\| + \frac{1}{\lambda}D_\psi\left(\lprox{\lambda, g}{\psi}(x), x\right) \nonumber\\
				&\ge M_1 + \inf_{y\in\RR^d}\left\{ \rho_g\|y - x_g\| + \frac{1}{\lambda}D_\psi\left(y, x\right)\right\} \nonumber\\
				&\ge M_1 + \lenv_{\lambda, h}^\psi(x), 
			\end{align}
			where $h\colon\RR^d\to\RR\colon x\mapsto \rho_g\|x-x_g\|$. Likewise, using the right Bregman--Moreau envelope, we have 		
			\begin{align}\label{eqn:bound_renv}
				\renv_{\lambda, g}^\psi(x) - g(x_g) &= g\left(\rprox{\lambda, g}{\psi}(x)\right) - g(x_g) + \frac{1}{\lambda}D_\psi\left(x,  \rprox{\lambda, g}{\psi}(x)\right) \nonumber\\
				&\ge M_1 + \rho_g\left\| \rprox{\lambda, g}{\psi}(x) - x_g\right\| + \frac{1}{\lambda}D_\psi\left(x, \rprox{\lambda, g}{\psi}(x)\right) \nonumber\\
				&\ge M_1 + \inf_{y\in\RR^d}\left\{ \rho_g\|y - x_g\| + \frac{1}{\lambda}D_\psi\left(x, y\right)\right\} \nonumber\\
				&\ge M_1 + \renv_{\lambda, h}^\psi(x).  
			\end{align}
			Next, using \Cref{def:bregman_env} again, for all $x\in\RR^d$, 
			\begin{align*}
				\lenv_{\lambda, h}^\psi(x) &= h\left(\lprox{\lambda, h}{\psi}(x)\right) + \frac{1}{\lambda}D_\psi\left(\lprox{\lambda, h}{\psi}(x), x\right) \ge h\left(\lprox{\lambda, h}{\psi}(x)\right) = \rho_g\left\| \lprox{\lambda, h}{\psi}(x) - x_g\right\|,\\
				\renv_{\lambda, h}^\psi(x) &= h\left(\rprox{\lambda, h}{\psi}(x)\right) + \frac{1}{\lambda}D_\psi\left(x, \rprox{\lambda, h}{\psi}(x)\right) \ge h\left(\rprox{\lambda, h}{\psi}(x)\right) = \rho_g\left\| \rprox{\lambda, h}{\psi}(x) - x_g\right\|. 
			\end{align*}
			It follows that there exists $M_2\in\RR$ such that for all $x\in\RR^d$, 
			\[\min\left\{\lenv_{\lambda, h}^\psi(x), \renv_{\lambda, h}^\psi(x)\right\} \ge \rho_g\|x-x_g\| + M_2. \]
			Combining this with \eqref{eqn:bound_lenv} and \eqref{eqn:bound_renv} yields the desired result.

			\item Now we suppose that \Cref{assum:nonsmooth}(ii$^\ddagger$) holds and $\psi$ is $\rho$-strongly convex, then for any $\lambda>0$, 
			\begin{equation}\label{eqn:envelope_bound}
				\sup_{x\in\RR^d}\left\{ g(x) - \env_{\lambda, g}^\psi(x)\right\} \le \frac{\lambda}{2\rho}\|g\|_{\Lip}^2. 
			\end{equation}
			If \eqref{eqn:envelope_bound} holds, then
			\[(\forall x\in\RR^d) \quad U_\lambda^\psi(x) \coloneqq f(x) + \env_{\lambda, g}^\psi(x) \ge f(x) + g(x) - \frac{\lambda}{2\rho}\|g\|_{\Lip}^2, \]
			which implies
			\[\int_{\RR^d} \e^{-U_\psi^\lambda(x)}\,\diff x \le \e^{\lambda\|g\|_{\Lip}^2/(2\rho)} \int_{\RR^d} \e^{-U(x)}\,\diff x < +\infty. \]

			Since \Cref{assum:nonsmooth}(ii$^\ddagger$) holds, we have 
			\begin{align}
				(\forall x\in\RR^d) \;\; g(x) - \lenv_{\lambda, g}^\psi(x) &= g(x) - \inf_{y\in\RR^d}\left\{g(y) + \frac{1}{\lambda} D_\psi(y, x)\right\} \nonumber\\
				&= \sup_{y\in\RR^d} \left\{g(x) - g(y) - \frac{1}{\lambda}D_\psi(y, x)\right\} \nonumber\\
				&\le \sup_{y\in\RR^d} \left\{\|g\|_{\Lip}\cdot\|x-y\| - \frac{1}{\lambda}D_\psi(y, x)\right\} \nonumber\\
				&\le \sup_{y\in\RR^d} \left\{\|g\|_{\Lip}\cdot\|x-y\| - \frac{\rho}{2\lambda}\|y- x\|^2\right\} & \text{by \eqref{eqn:strong_convex}} \nonumber\\
				&\le \frac{\lambda}{2\rho}\|g\|_{\Lip}^2, \label{eqn:bound_env_diff}
			\end{align}
			since the maximum of $u\mapsto au-bu^2$ for $a\in\RP$ and $b\in\RPP$ is $a^2/(4b)$. 
			Likewise, we also have the same bound for the right Bregman--Moreau envelope
			\[(\forall x\in\RR^d) \quad g(x) - \renv_{\lambda, g}^\psi(x) \le \frac{\lambda}{2\rho}\|g\|_{\Lip}^2.\]
		\end{enumerate}
		
		\item Recall that $\pi$ has a density with respect to the Lebesgue measure and $U_\lambda^\psi(x) \le U(x)$ for all $x\in\RR^d$. Then we have 
		\begin{equation}\label{eqn:bound_integral}
			\int_{\RR^d}\e^{-U(x)}\,\diff x \le \int_{\RR^d} \e^{-U_\lambda^\psi(x)}\,\diff x. 
		\end{equation} 
		This implies that, for all $ x\in\RR^d $, 
		\begin{multline}\label{eqn:bound_density}
			\pi(x) \le \frac{\pi(x)\int_{\RR^d}\e^{-U(y)}\,\diff y}{\int_{\RR^d} \e^{-U_\lambda^\psi(y)}\,\diff y} =  \frac{\e^{-U(x)}}{\int_{\RR^d} \e^{-U_\lambda^\psi(y)}\,\diff y} \\
			= \frac{\e^{-U_\lambda^\psi(x)}}{\int_{\RR^d}  \e^{-U_\lambda^\psi(y)}\,\diff y} \cdot\e^{-U(x) + U_\lambda^\psi(x)} 
			= \pi_\lambda^\psi\cdot\e^{\env_{\lambda, g}^\psi(x) - g(x)} \le \pi_\lambda^\psi(x),
		\end{multline}
		since $\env_{\lambda, g}^\psi(x) \le g(x)$ for all $ x\in\RR^d $. 
		Then for any $\lambda>0$, we have
		\begin{align}
			\|\pi_\lambda^\psi - \pi \|_{\TV} 
			&= \sup_{\sfA\in\euB(\RR^d)} \left|\int_{\sfA} \pi_\lambda^\psi(x) - \pi(x)\,\diff x \right| \le \sup_{\sfA\in\euB(\RR^d)}\int_{\sfA}  \left|\pi_\lambda^\psi(x) - \pi(x)\right|\,\diff x  \nonumber\\
			&\le \int_{\RR^d} \left| \pi_\lambda^\psi(x) - \pi(x) \right|\,\diff x  \nonumber\\
			&= \int_{\RR^d} \left( \pi_\lambda^\psi(x) - \pi(x)\right)^+ + \left( \pi_\lambda^\psi(x) - \pi(x)\right)^- \,\diff x \nonumber\\
			&= 2 \int_{\RR^d} \left( \pi_\lambda^\psi(x) - \pi(x)\right)^+ \,\diff x \nonumber\\
			&= 2 \int_{\RR^d} \pi_\lambda^\psi(x) - \pi(x) \,\diff x & \text{by \eqref{eqn:bound_density}} \nonumber\\
			&\le 2\int_{\RR^d} \pi_\lambda^\psi(x) - \pi(x)\frac{\int_{\RR^d} \e^{-U(y)}\,\diff y}{\int_{\RR^d}\e^{-U_\lambda^\psi(y)}\,\diff y}\,\diff x & \text{by \eqref{eqn:bound_integral}} \nonumber\\
			&= 2\left[\frac{1}{\int_{\RR^d}\e^{-U_\lambda^\psi(x)}\,\diff x}\int_{\RR^d} \e^{-U_\lambda^\psi(x)} - \e^{-U(x)}\,\diff x\right]  \nonumber\\
			&= 2\int_{\RR^d} \pi_\lambda^\psi(x)\left(1 - \e^{\env_{\lambda, g}^\psi(x) - g(x)}\right)\,\diff x \label{eqn:bound_tv}\\
			&= 2 \left(1 - \frac{\int_{\RR^d} \e^{-U(x)}\,\diff x}{\int_{\RR^d} \e^{-U_\lambda^\psi(x)}\,\diff x}\right) \nonumber\\
			&\to 0,  \nonumber
		\end{align}
		as $\lambda\downarrow0$, since, using \Cref{prop:gradient}(c) and the monotone convergence theorem, we have
		\[\lim\limits_{\lambda\to0} U_\lambda^\psi(x) = U(x) \quad\Rightarrow\quad \lim\limits_{\lambda\to0} \int_{\RR^d} \e^{-U_\lambda^\psi(x)}\,\diff x =  \int_{\RR^d} \e^{-U(x)}\,\diff x. \]

		\item Since $\env_{\lambda, g}^\psi(x) \le g(x)$ for all $ x\in\RR^d $ and $1-\e^{-u} \le u$ for all $u\in\RP$, by \eqref{eqn:bound_tv}, if \Cref{assum:nonsmooth}(ii$^\ddagger$) holds, then 
		\[ \|\pi_\lambda^\psi - \pi \|_{\TV} \le 2\int_{\RR^d} \pi_\lambda^\psi(x)\left(g(x) - \env_{\lambda, g}^\psi(x) \right) \,\diff x \le \frac{\lambda}{\rho}\|g\|_{\Lip}^2,\]
		where the last inequality follows from \eqref{eqn:bound_env_diff}. 
		
		Now we let $C_U \coloneqq \int_{\RR^d} \e^{-U(x)}\,\diff x$. For the second part, we will make use of the inequalities \Cref{eqn:bound_integral} and  
		\begin{equation}\label{eqn:bound_env}
			\text{\eqref{eqn:bound_env_diff}} \quad \Rightarrow \quad (\forall x\in\RR^d)\quad -U_\lambda^\psi(x) \le -U(x) + \frac{\lambda}{2\rho}\|g\|_{\Lip}^2,   
		\end{equation}
		which implies
		\begin{equation}\label{eqn:bound_env_integral}
			\int_{\RR^d} \e^{-U_\lambda^\psi(x)}\,\diff x \le \int_{\RR^d} \e^{-U(x)}\cdot\e^{\lambda\|g\|_{\Lip}^2/(2\rho)}\,\diff x = C_U\e^{\lambda\|g\|_{\Lip}^2/(2\rho)}. 
		\end{equation}
		Suppose that $h\ge 0$. Then \eqref{eqn:bound_env_integral} and $U_\lambda^\psi(x) \le U(x)$ for all $x\in\RR^d$ imply 
		\begin{align}
			\Ex_{\pi_\lambda^\psi} h &= \int_{\RR^d} h(x) \frac{\e^{-U_\lambda^\psi(x)}}{\int_{\RR^d} \e^{-U_\lambda^\psi(y)}\,\diff y}\,\diff x \nonumber\\
			&\ge C_U^{-1}\e^{-\lambda\|g\|_{\Lip}^2/(2\rho)}\int_{\RR^d} h(x) \e^{-U_\lambda^\psi(x)}\,\diff x \nonumber\\
			&\ge C_U^{-1}\e^{-\lambda\|g\|_{\Lip}^2/(2\rho)}\int_{\RR^d} h(x) \e^{-U(x)}\,\diff x \nonumber\\
			&= \e^{-\lambda\|g\|_{\Lip}^2/(2\rho)}\int_{\RR^d} h(x) \pi(x)\,\diff x \nonumber\\
			&= \e^{-\lambda\|g\|_{\Lip}^2/(2\rho)}\Ex_{\pi} h. \label{eqn:expectation_lower}
		\end{align}
		On the other hand, \eqref{eqn:bound_env_diff} and \eqref{eqn:bound_env} imply 
		\begin{multline}\label{eqn:expectation_upper}
			\Ex_{\pi_\lambda^\psi} h \le C_U^{-1}\int_{\RR^d} h(x) \e^{-U_\lambda^\psi(x)}\,\diff x \le C_U^{-1} \e^{\lambda\|g\|_{\Lip}^2/(2\rho)}\int_{\RR^d} h(x) \e^{-U(x)}\,\diff x\\
			= \e^{\lambda\|g\|_{\Lip}^2/(2\rho)}\int_{\RR^d} h(x) \pi(x)\,\diff x = \e^{\lambda\|g\|_{\Lip}^2/(2\rho)}\Ex_{\pi} h. 
		\end{multline}
		Combining \eqref{eqn:expectation_lower} and \eqref{eqn:expectation_upper} yields 
		\begin{equation}\label{eqn:expectation}
			\e^{-\lambda\|g\|_{\Lip}^2/(2\rho)}\Ex_\pi h \le \Ex_{\pi_\lambda^\psi} h \le \e^{\lambda\|g\|_{\Lip}^2/(2\rho)}\Ex_\pi h. 
		\end{equation}
		Then, applying \eqref{eqn:expectation} gives 
		\begin{align*}
			-\left(\e^{\lambda\|g\|_{\Lip}^2/(2\rho)} - 1\right)\Ex_{\pi} h &= -\max\left\{\e^{\lambda\|g\|_{\Lip}^2/(2\rho)} - 1, 1-\e^{-\lambda\|g\|_{\Lip}^2/(2\rho)} \right\}\Ex_{\pi} h \\
			&= \min\left\{1-\e^{\lambda\|g\|_{\Lip}^2/(2\rho)}, \e^{-\lambda\|g\|_{\Lip}^2/(2\rho)}  - 1\right\}\Ex_{\pi} h \\
			&\le \left(\e^{-\lambda\|g\|_{\Lip}^2/(2\rho)} - 1 \right) \Ex_{\pi} h \\
			&\le \Ex_{\pi_\lambda^\psi}h - \Ex_\pi h\\
			&\le \left(\e^{\lambda\|g\|_{\Lip}^2/(2\rho)} - 1 \right) \Ex_{\pi} h \\
			&\le \max\left\{\e^{\lambda\|g\|_{\Lip}^2/(2\rho)} - 1, 1-\e^{-\lambda\|g\|_{\Lip}^2/(2\rho)} \right\}\Ex_{\pi} h \\
			&= \left(\e^{\lambda\|g\|_{\Lip}^2/(2\rho)} - 1 \right) \Ex_{\pi} h,
		\end{align*}
		which implies that, for any $h\ge0$, 
		\begin{equation}\label{eqn:expectation_bound}
			\left|\Ex_{\pi_\lambda^\psi}h - \Ex_\pi h \right| \le \left(\e^{\lambda\|g\|_{\Lip}^2/\rho} - 1 \right) \Ex_{\pi} h. 
		\end{equation}
		Now, for any general integrable function $h$, we can write $h = h^+ - h^-$, where $h^+\ge 0$ and $h^-\ge 0$. We also have $|h| = h^+ + h^-$. Consequently, we have 
		\begin{align}
			\left|\Ex_{\pi_\lambda^\psi}h - \Ex_\pi h \right| &= \left|\left( \Ex_{\pi_\lambda^\psi}h^+ - \Ex_\pi h^+\right)  - \left(\Ex_{\pi_\lambda^\psi}h^- - \Ex_\pi h^- \right) \right| \nonumber\\
			&\le \left|\Ex_{\pi_\lambda^\psi}h^+ - \Ex_\pi h^+ \right| + \left|\Ex_{\pi_\lambda^\psi}h^- - \Ex_\pi h^- \right| \nonumber\\
			&= \left(\e^{\lambda\|g\|_{\Lip}^2/\rho} - 1 \right) \Ex_\pi h^+ + \left(\e^{\lambda\|g\|_{\Lip}^2/\rho} - 1 \right) \Ex_\pi h^- & \text{by \eqref{eqn:expectation_bound}} \nonumber\\
			&= \left(\e^{\lambda\|g\|_{\Lip}^2/\rho} - 1 \right) \Ex_\pi |h|. \label{eqn:expectation_bound_one}
		\end{align}
		If we switch the role of $\pi_\lambda^\psi$ and $\pi$ in \eqref{eqn:expectation}, i.e., 
		\[\e^{-\lambda\|g\|_{\Lip}^2/(2\rho)}\Ex_{\pi_\lambda^\psi} h \le \Ex_\pi h \le \e^{\lambda\|g\|_{\Lip}^2/(2\rho)}\Ex_{\pi_\lambda^\psi} h, \]
		then we get the following inequality similar to \eqref{eqn:expectation_bound_one}: 
		\begin{equation}\label{eqn:expectation_bound_two}
			\left|\Ex_{\pi_\lambda^\psi}h - \Ex_\pi h \right| \le \left(\e^{\lambda\|g\|_{\Lip}^2/\rho} - 1 \right) \Ex_{\pi_\lambda^\psi} |h|. 
		\end{equation}
		Combining \eqref{eqn:expectation_bound_one} and \eqref{eqn:expectation_bound_two} yields the desired result. 
	\end{enumerate}
\end{proof}

\subsection{Proof of \Cref{prop:gradient}}
\begin{enumerate}[label=(\alph*)]
	\item Recall that $g$ is lower bounded. Then by \citet[Fact 2.6]{bauschke2018regularizing}, $g(\cdot) + \frac{1}{\lambda}D_\psi(\cdot, y)$ and $g(\cdot) + \frac{1}{\lambda}D_\psi(y, \cdot)$ are both coercive for all $y\in\euY$. Then the gradient formulas of the Bregman--Moreau envelopes follow from \citet[Proposition 2.19]{bauschke2018regularizing}, which in turn follows from \citet[Remark 2.14]{bauschke2018regularizing} and \citet[Proposition 3.12]{bauschke2006joint}. 
	
	\item The Lipschitz continuity of the gradient of the left Bregman--Moreau envelope follows from \citet[Theorem 3.5]{soueycatt2020regularization}, whereas the Lipschitz continuity of the gradient of the right Bregman--Moreau envelope holds because, assuming that $\nabla\psi$ is Lipschitz, $y\mapsto\nabla_y D_\psi(y, x) = \nabla\psi(y) - \nabla\psi(x)$ is Lipschitz and we use the fact that the composition of Lipschitz  maps is also Lipschitz. We also remark that if we further assume that $\psi$ is very strictly convex, then $\nabla\psi$ is Lipschitz (\citealp[Proposition 2.10]{bauschke2000dykstras}; \citealp[Lemma 2.3(iii)]{laude2020bregman}). 
	
	\item The asymptotic behavior of the Bregman--Moreau envelopes follow from \citet[Theorem 3.3]{bauschke2018regularizing}. 
	
\end{enumerate}
\hfill$\Box$

\subsection{Proof of \Cref{thm:conv}}
Note that \eqref{eqn:bound_W2_surrogate} follows from Theorem 3.1 of \citet{li2022mirror}, i.e.,
\begin{equation*}
	\sfW_{2, \varphi}(\mu_k, \pi_\lambda^\psi) \le \sqrt{2}\e^{-(\alpha-2M_\varphi)\gamma k}\sfW_{2, \varphi}(\mu_0, \pi_\lambda^\psi) + C\sqrt{2\gamma},
\end{equation*}
where $\gamma\in\left] 0, \gamma_{\max}\right]$, with $\gamma_{\max} = \euO\left(\frac{(\alpha-2M_\varphi)^2}{\left( \beta^2(1+8M_\varphi)^2\right)}\right)$ and $C = \euO\left(\frac{\beta(1+8M_\varphi)\sqrt{d}}{(\alpha-2M_\varphi)}\right)$.

By \citet[][Theorem 4]{gibbs2002choosing} with $d(x, y) = \|\nabla \varphi(x) - \nabla\varphi(y)\|^2$, we have the following inequality between the Wasserstein distance and the total variation distance: 
\begin{equation}\label{eqn:W_TV}
	\sfW_{2,\varphi}(\mu,\nu) \le \eta\|\mu-\nu\|_{\TV}, 
\end{equation}
where $\eta \coloneqq \sup_{u\sim\mu, v\sim\nu} \|\nabla \varphi(u) - \nabla\varphi(v)\|^2$. 

Invoking the triangle inequality, \eqref{eqn:bound_W2_surrogate} and \eqref{eqn:W_TV}, we have 
\begin{align}
	\sfW_{2, \varphi}(\mu_k, \pi) &\le \sfW_{2, \varphi}(\mu_k, \pi_\lambda^\psi) + \sfW_{2, \varphi}(\pi_\lambda^\psi, \pi) \nonumber\\
	&\le \sqrt{2}\e^{-(\alpha-2M_\varphi)\gamma k}\sfW_{2, \varphi}(\mu_0, \pi_\lambda^\psi) + C\sqrt{2\gamma} + \eta\|\pi_\lambda^\psi - \pi \|_{\TV}. \label{eqn:thm1_1}
\end{align}
Recall from \Cref{prop:approx}(c) that if \Cref{assum:nonsmooth}(ii$^\ddagger$) holds and $\psi$ is $\rho$-strongly convex, then \[\|\pi_\lambda^\psi - \pi \|_{\TV} \le \frac{\lambda}{\rho}\|g\|_{\Lip}^2. \]
Hence, we obtain 
\begin{equation}\label{eqn:thm1_2}
	\sfW_{2, \varphi}(\pi_\lambda^\psi, \pi) \le \eta\|\pi_\lambda^\psi - \pi \|_{\TV} \le \frac{\eta\lambda}{\rho}\|g\|_{\Lip}^2.
\end{equation}
Then \eqref{eqn:thm1_1} becomes 
\begin{equation}\label{eqn:thm1_3}
	\sfW_{2, \varphi}(\mu_k, \pi) \le \sqrt{2}\e^{-(\alpha-2M_\varphi)\gamma k}\sfW_{2, \varphi}(\mu_0, \pi_\lambda^\psi) + C\sqrt{2\gamma} + \frac{\eta\lambda}{\rho}\|g\|_{\Lip}^2.
\end{equation}
On the other hand, applying the triangle inequality again and \eqref{eqn:thm1_2}, we have 
\[\sfW_{2, \varphi}(\mu_0, \pi_\lambda^\psi) \le \sfW_{2, \varphi}(\mu_0, \pi) + \sfW_{2, \varphi}(\pi_\lambda^\psi, \pi) \le \sfW_{2, \varphi}(\mu_0, \pi) + \frac{\eta\lambda}{\rho}\|g\|_{\Lip}^2. \]
Plugging into \eqref{eqn:thm1_3} yields the desired result \eqref{eqn:bound_W2}
\[\sfW_{2, \varphi}(\mu_k, \pi) \le \sqrt{2}\e^{-(\alpha-2M_\varphi)\gamma k}\sfW_{2, \varphi}(\mu_0, \pi)
+ C\sqrt{2\gamma} + \left(1 + \sqrt{2}\e^{-(\alpha-2M_\varphi)\gamma k}\right)\frac{\eta\lambda}{\rho}\|g\|_{\Lip}^2.\]

\subsection{Proof of \Cref{cor:mixing}}
\eqref{eqn:mixing_time} simply follows from \eqref{eqn:bound_W2} with certain algebraic manipulations.

\section{Details of Numerical Experiments}
\label{sec:details_expt}

\paragraph{More notation.} 
For any $\bx=(x_1, \ldots, x_d)^\top\in\RR^d$, $\Diag(\bx)\in\RR^{d\times d}$ is the diagonal matrix whose diagonal entries are $x_1, \ldots, x_d$. We also write $\setd \coloneqq\{1, \ldots, d\}$.

With the choice of 
\[\varphi_{\bbeta}(\btheta) = \sumd \left[\theta_i\operatorname{arsinh}\left(\theta_i/\beta_i\right) - \sqrt{\theta_i^2 + \beta_i^2}\right], \]
and $\psi(\btheta) = \frac12\|\btheta\|_{\bM}^2$ with $\bM\in\bbS^d_{++}$, simple calculation yields 
\begin{align*}
	\nabla\varphi_{\bbeta}(\btheta) &= \left( \operatorname{arsinh}(\theta_i/\beta_i)\right)_{1\le i\le d}, \\
	\nabla^2\varphi_{\bbeta}(\btheta) &= \Diag\left(\left(  (\theta_i^2 + \beta_i^2)^{-\sfrac{1}{2}}\right)_{1\le i\le d}\right), \\
	\varphi_{\bbeta}^*(\btheta) &= \sumd \beta_i\operatorname{cosh}(\theta_i), \\
	\nabla\varphi_{\bbeta}^*(\btheta) &= (\beta_i\operatorname{sinh}(\theta_i))_{1\le i\le d}, \\
	\nabla^2\varphi_{\bbeta}^*(\btheta) &= \Diag\left( (\beta_i\operatorname{cosh}(\theta_i))_{1\le i\le d}\right),
\end{align*}
and
\begin{align*}	
	\nabla\psi(\btheta) &= \bM\btheta, \\
	\nabla^2\psi(\btheta) &= \bM. 
\end{align*}
Note that for $f = 0$ and $g(\btheta) = \sumd \alpha_i|\theta_i|$, $\euF = \euG = \euX = \euY = \RR^d$. It is straightforward to see that \Cref{assum:smooth,assum:nonsmooth}(i), (ii$^\dagger$), (ii$^\ddagger$) are satisfied. 
For \Cref{assum:mirror1}, we check the modified self-concordance condition since the other assumptions are obvious. Since $\varphi_{\bbeta}(\btheta)$ is separable in a sense that it is in the form $\sumd \phi_{\beta_i}(\theta_i)$, where $\phi_\beta(\theta) \coloneqq \theta\operatorname{arsinh}\left(\theta/\beta\right) - \sqrt{\theta^2 + \beta^2}$ with $\beta>0$, it suffices to show that  $\phi_\beta$ is a modified self-concordant function. As noted in \citet{zhang2020wasserstein}, it suffices to check that $[(\phi_\beta^*)'']^{-\sfrac12}$ is Lipschitz. 

Since $[(\phi_\beta^*(\theta))'']^{-\sfrac12} = \beta^{-\sfrac12}\sqrt{\operatorname{sech}(\theta)}$, we have 
\[\left[\frac{1}{\sqrt{(\phi_\beta^*(\theta))''}}\right]' = -\frac{1}{2\sqrt{\beta}}\sinh(\theta)\operatorname{sech}^{\sfrac32}(\theta) \quad\Rightarrow\quad \left| \left[\frac{1}{\sqrt{(\phi_\beta^*(\theta))''}}\right]'\right| \le \frac{1}{\sqrt{2}\cdot 3^{\sfrac34}}. \]
Hence, $[(\phi_\beta^*)'']^{-\sfrac12}$ is Lipschitz.

It is also obvious to see that $\psi$ satisfies \Cref{assum:mirror2}. 

For $\bM\in\bbS^d_{++}$, the Bregman divergence associated to $\frac12\|\cdot\|_{\bM}^2$ is given by $D_\psi(\btheta, \bvartheta) = \frac12\|\btheta - \bvartheta\|_{\bM}^2$, which is indeed a distance since it is symmetric in its arguments.

The choice of $\psi$ implies its associated Bregman divergence $D_\psi$ satisfies all of \Cref{assum:Bregman}. 
According to \citet[Corollary 7.2 and Example 7.3]{bauschke2001joint}, since $[\nabla^2\psi(\btheta)]^{-1} = \bM^{-1}$ is constant for all $\btheta\in\RR^d$, and thus trivially matrix-concave. Hence $D_\psi$ is jointly convex. In addition, the gradient and Hessian of $D_\psi$ in the second argument are 
\begin{align*}
	\nabla_{\bvartheta}D_\psi(\btheta, \bvartheta) &= \bM(\bvartheta - \btheta), \\
	\nabla_{\bvartheta}^2 D_\psi(\btheta, \bvartheta) &= \bM.  
\end{align*}
Since $\bM\in\bbS^d_{++}$, for any $\btheta\in\RR^d$, $D_\psi(\btheta, \cdot)$ is strictly convex on $\RR^d$. Obviously, for any $\btheta\in\RR^d$, $D_\psi(\btheta, \cdot)$ is also continuous on $\RR^d$ and coercive. 

To check \Cref{assum:functions}, we first compute the expressions of $\bprox{\lambda, \alpha|\cdot|}{\frac12 m(\cdot)^2}$. Then the expressions of $\lprox{\lambda, g}{\psi}(\btheta)$ and $\rprox{\lambda, g}{\psi}(\btheta)$ are given by 
\[\lprox{\lambda, g}{\psi}(\btheta) = \left( \lprox{\lambda, \alpha_i|\cdot|}{\frac12 m(\cdot)^2}(\theta_i)\right)_{1\le i \le d} \quad \text{and} \quad \rprox{\lambda, g}{\psi}(\btheta) = \left( \rprox{\lambda, \alpha_i|\cdot|}{\frac12 m(\cdot)^2}(\theta_i)\right)_{1\le i \le d}, \]
attributed to the separable structures of $g$ and $D_\psi$. Note that $\lprox{\lambda, g}{\psi} = \rprox{\lambda, g}{\psi}$ since $D_\psi$ is symmetric in its arguments. 

Simple manipulation yields
\begin{align*}
	\bprox{\lambda, \alpha|\cdot|}{\frac12 m(\cdot)^2}(\theta) &= \argmin_{\vartheta\in\RR} \ \left\{ \alpha|\vartheta| + \frac{m}{2\lambda}(\theta - \vartheta)^2\right\} = \argmin_{\vartheta\in\RR} \ \left\{ \frac{\lambda\alpha}{m}|\vartheta| + \frac{1}{2}(\theta - \vartheta)^2\right\}\\
	&= \prox_{\lambda\alpha|\cdot|/m}(\theta),
\end{align*}
where $\prox_{\mu|\cdot|}(\theta) = \sign(\theta)\max\{|\theta|-\mu, 0\}$ is the soft-thresholding operator, for $\theta\in\RR$ and $\mu>0$. Consequently, we have 
\[\bprox{\lambda, g}{\psi}(\btheta) = \left( \sign(\theta_i)\max\{|\theta_i| - \lambda\alpha_i/m_i, 0\}\right)_{1\le i \le d}. \]

It remains to check \Cref{assum:functions}. It appears that $U_\lambda^\psi = \env_{\lambda, g}^\psi$ in this case is Legendre strongly convex with $\alpha=0$ (i.e., convex but not strongly convex), which does not satisfy the required assumption that $M_\varphi < \alpha/2$. However, for practical purpose, this choice of $\psi$ works well. We will give

We then check that $\env_{\lambda, g}^\psi$ is $\beta_g$-smooth relative to $\varphi$. We check this via the equivalent second-order characterization: $\beta_g\nabla^2\varphi_{\bbeta}(\btheta) - \nabla^2\env_{\lambda, g}^\psi(\btheta) \succeq 0 $ for all $\btheta\in\RR^d$. 

Let $i\in\setd$. Then we have  
\begin{equation*}
	\left[\nabla^2\env_{\lambda, g}^\psi(\btheta)\right]_{i,i} = \begin{cases*}
		m_i/\lambda & if $\theta_i \in [ -\lambda\alpha_i/m_i, \lambda\alpha_i/m_i]$,\\
		0 & otherwise. 
	\end{cases*}
\end{equation*}

Since 
\[\nabla^2\varphi_{\bbeta}(\btheta) = \Diag\left(\left( \frac{1}{\sqrt{\theta_i^2 + \beta_i^2}}\right)\right)_{1\le i\le d} , \]
we can choose 
\[\beta_g = \sup_{i\in\setd} \sup_{\theta_i\in[ -\lambda\alpha_i/m_i, \lambda\alpha_i/m_i]} \,\left\lbrace \frac{m_i\sqrt{\theta_i^2 + \beta_i^2}}{\lambda}\right\rbrace. \]

Given the choice $m_i = \alpha_i/2$ for all $i\in\setd$, we then have 
\[\beta_g = \sup_{i\in\setd} \sup_{\theta_i\in[ -2\lambda, 2\lambda]} \,\left\lbrace \frac{\alpha_i\sqrt{\theta_i^2 + \beta_i^2}}{2\lambda}\right\rbrace = \sup_{i\in\setd} \sup_{\theta_i\in[0, 2\lambda]} \,\left\lbrace \frac{\alpha_i\sqrt{\theta_i^2 + \beta_i^2}}{2\lambda}\right\rbrace = \sup_{i\in\setd} \frac{\alpha_i\sqrt{4\lambda^2 + \beta_i^2}}{2\lambda} < +\infty, \]

which implies that $\lenv_{\lambda, g}^\psi$ is $\beta_g$-smooth relative to $\varphi$.

\subsection{Different Bregman--Moreau Envelopes}
\label{subsec:different}
To find Bregman--Moreau envelopes which also satisfy \Cref{assum:functions}, we let $\psi = \psi_{\bsigma}$ be also the hyperbolic entropy parameterized by $\bsigma$. By slight abuse of notation, we also write $\psi_{\bsigma}(\btheta) = (\psi_{\sigma_i}(\theta_i))_{1\le i \le d}$. 

We have the following expression of the associatedl left Bregman proximity operator. 
\begin{proposition}\label{prop:bregman_prox_hyp}
	The left Bregman proximity operator of $\alpha|\cdot|$ associated to the Legendre function $\psi_\sigma$ for $\alpha>0$ is 
	\[
	\lprox{\lambda, \alpha|\cdot|}{\psi_\sigma}(\theta) = 
	\begin{cases*}
		\sigma\sinh(\arsinh(\theta/\sigma) - \alpha\lambda) & if $\theta > \sigma\sinh(\alpha\lambda)$,\\
		\sigma\sinh(\arsinh(\theta/\sigma) + \alpha\lambda) & if $\theta < \sigma\sinh(-\alpha\lambda)$,\\
		\sqrt{\theta^2 + \beta^2} & otherwise.
	\end{cases*}
	\]
	\begin{proof}[Proof of \Cref{prop:bregman_prox_hyp}]
		According to \Cref{def:bregman_prox}, 
		\[\lprox{\lambda, \alpha|\cdot|}{\exp}(\theta) = \argmin_{\vartheta\in\RR} \,\left\{\lambda\alpha|\vartheta| + \vartheta\left(\arsinh(\vartheta/\sigma) - \arsinh(\theta/\sigma)\right) - \sqrt{\vartheta^2 + \sigma^2} + \sqrt{\theta^2 + \sigma^2}\right\}. \]
		First-order conditions give
		\[\begin{cases*}
			\alpha\lambda + \arsinh(\vartheta/\sigma) - \arsinh(\theta/\sigma) = 0 & if $\vartheta > 0$, \\
			-\alpha\lambda + \arsinh(\vartheta/\sigma) - \arsinh(\theta/\sigma) = 0 & if $\vartheta < 0$, 
		\end{cases*}\]
		which implies 
		\begin{equation}\label{eqn:left_bregman_prox_hyp_1}
			\vartheta^\star = 
			\begin{cases*}
				\sigma\sinh(\arsinh(\theta/\sigma) - \alpha\lambda) & if $\vartheta^\star > 0$, \\
				\sigma\sinh(\arsinh(\theta/\sigma) + \alpha\lambda) & if $\vartheta^\star < 0$ 
			\end{cases*} = 
			\begin{cases*}
				\sigma\sinh(\arsinh(\theta/\sigma) - \alpha\lambda) & if $\theta > \sigma\sinh(\alpha\lambda)$, \\
				\sigma\sinh(\arsinh(\theta/\sigma) + \alpha\lambda) & if $\theta < \sigma\sinh(-\alpha\lambda)$. 
			\end{cases*}
		\end{equation}
		On the other hand, if $\vartheta = 0$, then 
		\begin{multline}\label{eqn:left_bregman_prox_hyp_2}
			\argmin_{\vartheta\in\RR} \,\left\{\lambda\alpha|\vartheta| + \vartheta\left(\arsinh(\vartheta/\sigma) - \arsinh(\theta/\sigma)\right) - \sqrt{\vartheta^2 + \sigma^2} + \sqrt{\theta^2 + \sigma^2}\right\} \\
			= \argmin_{\vartheta\in\RR} \,\left\{- \sqrt{\sigma^2} + \sqrt{\theta^2 + \sigma^2}\right\} = \sqrt{\theta^2 + \sigma^2} - \sigma,  
		\end{multline}
		which corresponds to the range $[\sigma\sinh(-\alpha\lambda), \sigma\sinh(\alpha\lambda)]$ for $\theta$. 
		Combining \eqref{eqn:left_bregman_prox_hyp_1} and \eqref{eqn:left_bregman_prox_hyp_2} yields the desired result. 			
	\end{proof}
\end{proposition}
The closed-form expression of the right Bregman proximity operator is much more complicated and is not given.

We show that $\lenv_{\lambda, g}^{\psi_{\bsigma}} - \alpha \varphi_{\bbeta}$ is convex for some $\lambda\in\RPP$, $\bbeta\in\Rp^d$, $\bsigma\in\Rp^d$ and $\alpha>2M_{\varphi_{\bbeta}}$. Also, recall that 
\[\left| \left[\frac{1}{\sqrt{(\phi_\beta^*(\theta))''}}\right]'\right| \le \frac{1}{\sqrt{2}\cdot 3^{\sfrac34}} \quad\Rightarrow\quad M_{\varphi_{\bbeta}} = \left(2\cdot 3^{\sfrac32}\cdot\min_{i\in\setd} \beta_i\right)^{-1}. \]
In particular, if we choose $\lambda=10^{-5}$, $\bbeta=(2\sqrt{d-i+1})_{1\le i\le d}^\top$, $\bsigma=(d,d-1,\ldots, 1)^\top$ and $\alpha=2M_{\varphi_{\bbeta}}+10^{-1}$, then we plot the following: 
\begin{figure}[h!]
	\centering
	\includegraphics[scale=0.8]{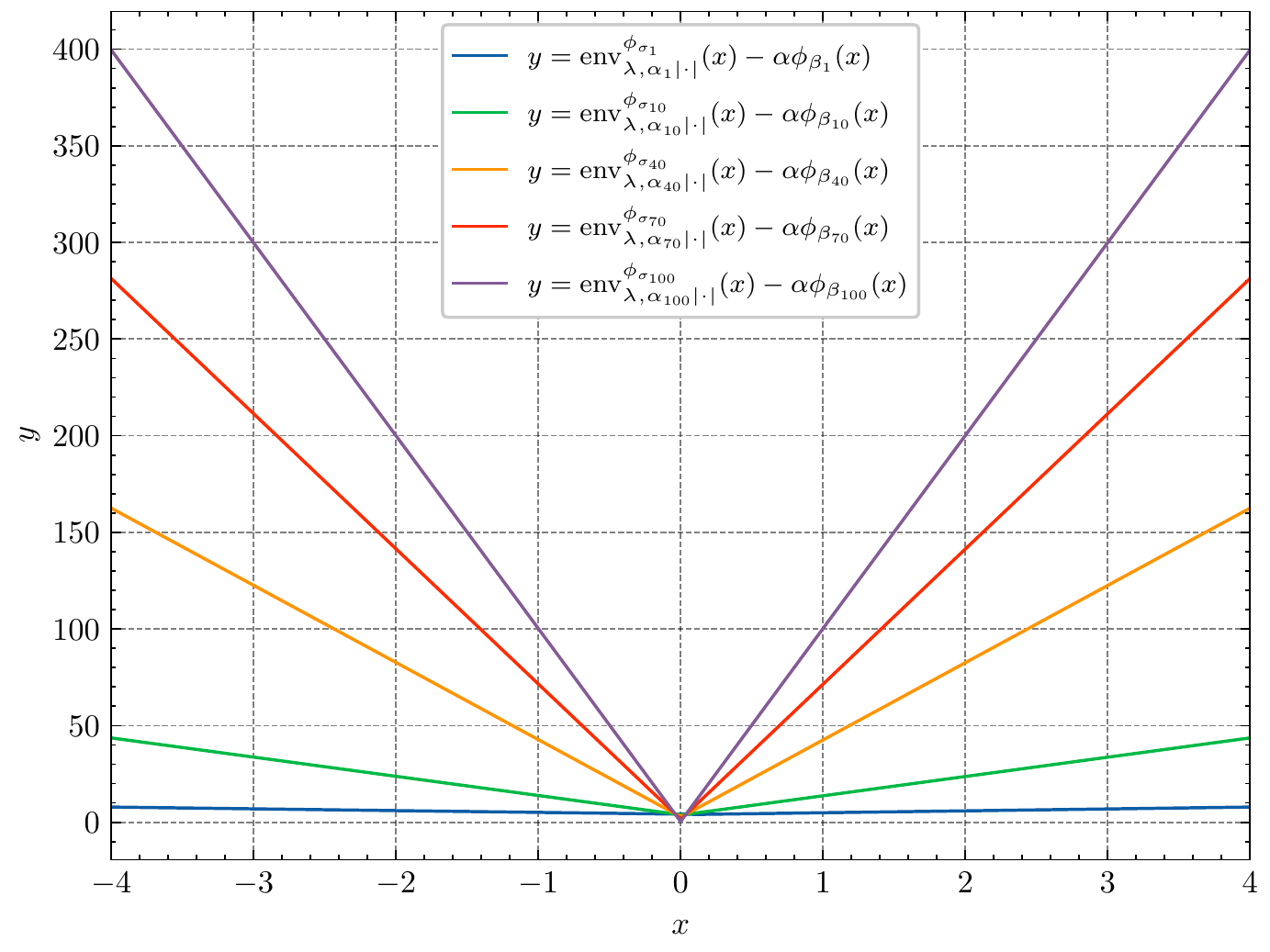}		
	\caption{Plots of $y = \env_{\lambda, \alpha_i|\cdot|}^{\psi_{\sigma_i}}(x) - \alpha\phi_{\beta_i}(x)$, for $i\in\{1, 10, 40, 70, 100\}$. }
	\label{fig:env_1}
\end{figure}

\Cref{fig:env_1} shows that the above choices give a convex $\lenv_{\lambda, g}^{\psi_{\bsigma}} - \alpha \varphi_{\bbeta}$. 

Similarly, we also show graphically that $\beta_g\varphi_{\bbeta} - \lenv_{\lambda, g}^{\psi_{\bsigma}}$ is convex for some $\beta_g>0$, e.g., $\beta_g= 2500$ (not tight). 
\begin{figure}[h!]
	\centering
	\includegraphics[scale=0.8]{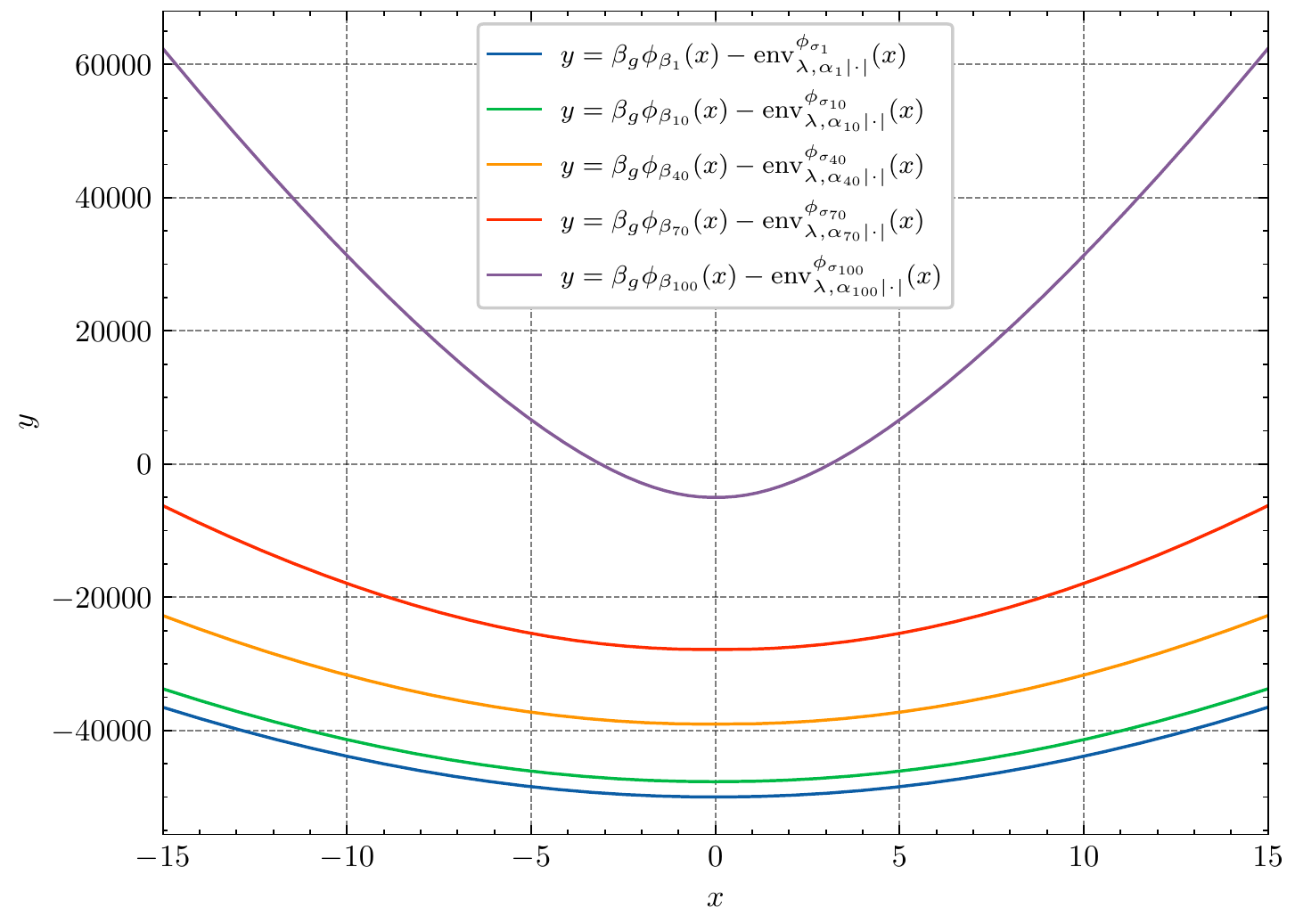}		
	\caption{Plots of $y = \beta_g\phi_{\beta_i}(x) - \env_{\lambda, \alpha_i|\cdot|}^{\psi_{\sigma_i}}(x)$, for $i\in\{1, 10, 40, 70, 100\}$. }
	\label{fig:env_2}
\end{figure}

Since all \Cref{assum:smooth,assum:nonsmooth,assum:mirror1,assum:mirror2,assum:Bregman,assum:functions} are satisfied, $g$ is Lipschitz and $\psi$ is strongly convex, the convergence results in the main text, i.e., \Cref{thm:conv} and \Cref{cor:mixing}, hold.

\section{The Bregman--Moreau Mirrorless Mirror-Langevin Algorithm}
\label{sec:BMMMLA}
In this section, we give the details of the Bregman--Moreau mirrorless mirror-Langevin algorithm (BMMMLA), whose results are mostly taken from \citet{ahn2021efficient}. 

\subsection{Assumptions}
\label{subsec:assumptions}
We first state the assumptions required in \citet{ahn2021efficient}. Instead of the modified self-concordance condition, the Legenedre function $\varphi$ has to be 
$M_\varphi$-\emph{self-concordant} \citep[\S5.1.3]{nesterov2018lectures}, i.e., for any $x\in\euX$, there exists $M_\varphi\ge0$ such that $\left| \nabla^3\varphi(x)[u,u,u]\right|  \le 2M_\varphi\|u\|_{\nabla^2\varphi(x)}^3$ for all $u\in\RR^d$. Furthermore, in addition to the $\alpha$-relative convexity (to $\varphi$) and $\beta$-relative smoothness (to $\varphi$) assumption, $U_\lambda^\psi$ also has to be $L$-Lipschitz relative to $\varphi$, which is defined as follows. 	
\begin{definition}[Relative Lipschitz continuity
	]\label{def:rel_Lip}
	A function $f\in\scrC^1$ is $L$-Lipschitz relative to a very strictly convex (see \Cref{assum:mirror1}(iii)) Legendre function $\varphi$ if there exists $L>0$ such that $\|\nabla f(x)\|_{[\nabla^2\varphi(x)]^{-1}} \le L$ for all $x\in\interior\dom f$. 
\end{definition}

It is worth noting that it is difficult to verify that Bregman--Moreau envelopes would satisfy such a relative Lipschitzness condition in general.

Now we let $x_0 \in \euY$. Similar to MLA \eqref{eqn:MLA} in \citet{ahn2021efficient}, the Bregman--Moreau mirrorless Mirror-Langevin algorithm (BMMMLA) iterates 	
\begin{equation*}
	\begin{aligned}
		x_{k+\sfrac{1}{2}} &= \nabla\varphi^* \left(\nabla\varphi(x_k) -\gamma\nabla U_\lambda^\psi(x_k) \right), \\
		x_{k+1} &= \nabla\varphi^*(Y_{\gamma}), 
	\end{aligned}
\end{equation*}
where
\begin{equation}\label{eqn:BMMMLA}
	\begin{cases}
		\diff Y_t = \sqrt{2} \left[ \nabla^2 \varphi^*(Y_t) \right]^{-\sfrac{1}{2}}\,\diff W_t\\
		Y_0 = \nabla\varphi\left(x_{k+\sfrac{1}{2}}\right) = \nabla\varphi(x_k) -\gamma\nabla U_\lambda^\psi(x_k).  
	\end{cases}
\end{equation}

\subsection{Convergence Results}
We suppose that the assumptions in \Cref{subsec:assumptions} hold.     
We define the mixture distribution $\bar{\mu}_K \coloneqq \frac1K\sumK \mu_k$, and let $ \beta' \coloneqq \beta+2M_\varphi L$. Then we have the following convergence results. 
\begin{theorem}[Convex]\label{thm:conv_non_bmmmla}
	Assume $\alpha = 0$ and $\beta' > 0$. Let $X_k \sim \mu_k$ be generated by \eqref{eqn:BMMMLA} with step size $\gamma = \min\left\{\varepsilon/(2\beta' d), 1/\beta'\right\}$. Then for all $\varepsilon>0$, there exists $\lambda>0$ such that $D_{\KL} (\bar{\mu}_K\midd \pi_\lambda^\psi) \le \varepsilon$ for 
	\begin{equation*}
		K\ge \frac{4d\beta' \sfD_\varphi(\pi, \mu_0)}{\varepsilon^2}\max\left\{1, \frac{\varepsilon}{2d}\right\}. 
	\end{equation*}
	
\end{theorem}

\begin{theorem}[Legendre strongly convex]\label{thm:conv_bmmmla}
	Assume $\alpha > 0$ and $\beta' > 0$.  
		Suppose that $X_0\sim\mu_0$ satisfies $\sfD_\varphi(\pi, \mu_0) \le \varepsilon$. Let $X_k \sim \mu_k$ be generated by \eqref{eqn:BMMMLA} with step size $\gamma = \min\left\{\varepsilon/(2\beta' d), 1/\beta'\right\}$. Then for all $\varepsilon>0$, there exists $\lambda>0$ such that $D_{\KL} (\bar{\mu}_K\midd \pi_\lambda^\psi) \le \varepsilon$ for 
		\begin{equation*}
			K\ge \frac{4\beta' d}{\varepsilon}\max\left\{1, \frac{\varepsilon}{2d}\right\}. 
		\end{equation*}
\end{theorem}
The results follow from \citet[][Theorems 1 and 2(b)]{ahn2021efficient}. Similar bounds on the total variation distance follows from \emph{Pinsker's inequality}: $\| P-Q \|_{\TV}^2 \le \frac12 D_{\KL}(P\midd Q)$, and also bounds on the total variation distance between $\bar{\mu}_K$ and the target distribution $\pi$ instead of the surrogate distribution $\pi_\lambda^\psi$. 
\hfill$\Box$

Note also that convergence in the Bregman transport cost also holds \citep[Theorem 2(a)]{ahn2021efficient}, where the Bregman transport cost is defined as follows. 
\begin{definition}[Bregman transport cost]
	For two probability measures $\mu$ and $\nu$ on $\euB(\RR^d)$, the \emph{Bregman transport cost} \citep{cordero2017transport} from $\mu$ to $\nu$ with respect to the Bregman divergence associated with a Legendre function $\varphi$ is defined by         
	\[\sfD_\varphi(\mu, \nu) \coloneqq \inf_{\pi\sim\Pi(\mu, \nu)}\int_{\RR^d\times\RR^d} D_\varphi(x, y) \,\diff\pi(x,y). \]
\end{definition}

We also refer to \citet[Theorem 2.1]{ahn2021efficient} for convergence results in terms of the Bregman transport cost \citep{cordero2017transport}. By Pinsker's inequality, we can also obtain similar results in terms of the total variation distance.

Finally, we give the Bregman--Moreau mirrorless mirror-Langevin algorithm \eqref{eqn:BMMMLA} with an Euler--Maruyama discretization for the second step. 

\begin{algorithm}[H]
	\caption{The Bregman--Moreau Mirrorless Mirror-Langevin Algorithm (BMMMLA)}
	\label{alg:BMMMLA}
	\begin{algorithmic}
		\STATE {\bfseries Initialize:} Legendre functions $\varphi$ and $\psi$, $\btheta_0 \in \RR^d$, step size $\gamma \in\RPP$, number of samples to be drawn $K\in\NN^*$, number of inner steps of Euler--Maruyama discretization $N\in\NN^*$. 
		\FOR{$k=0,1,2,\ldots, K-1$}				
		\STATE $\btheta_{k+\sfrac12} = \nabla\varphi^*\left(\nabla\varphi(\btheta_k) - \gamma\nabla U_\lambda^\psi(\btheta_k)\right)$
		\STATE $\by_0 = \nabla\varphi\left(\btheta_{k+\sfrac{1}{2}}\right) = \nabla\varphi(\btheta_k) -\gamma\nabla U_\lambda^\psi(\btheta_k)$
		\FOR{$n=0,1,2,\ldots, N$}
		\STATE $\bxi_n\sim\sfN_d(\zero_d, I_d)$					
		\STATE $\by_{n+1} = \by_n + \sqrt{2\gamma/N} \left[ \nabla^2 \varphi^*(\by_n) \right]^{-\sfrac{1}{2}}\bxi_n$
		\ENDFOR
		\STATE $\btheta_{k+1} = \nabla\varphi^*(\by_{N+1})$
		\ENDFOR
	\end{algorithmic}
\end{algorithm}

We also give the experimental results of the Bregman--Moreau mirrorless mirror-Langevin algorithm in \Cref{sec:add_expt}.

\section{Additional Numerical Experiments}
\label{sec:add_expt}

\subsection{Anisotropic Laplace Distribution}
We first give more plots of different dimensions for the experiment in \Cref{sec:expt}. 

\begin{figure}[h!]
	\centering
	\begin{subfigure}[b]{0.24\textwidth}
		\centering
		\includegraphics[width=\textwidth]{fig/an_laplace_myula_1.pdf}
		\caption{1\textsuperscript{st} dimension}
	\end{subfigure}
	\begin{subfigure}[b]{0.24\textwidth}
		\centering
		\includegraphics[width=\textwidth]{fig/an_laplace_myula_2.pdf}
		\caption{2\textsuperscript{nd} dimension}
	\end{subfigure}
	\begin{subfigure}[b]{0.24\textwidth}
		\centering
		\includegraphics[width=\textwidth]{fig/an_laplace_myula_5.pdf}
		\caption{5\textsuperscript{th} dimension}
	\end{subfigure}
	\begin{subfigure}[b]{0.24\textwidth}
		\centering
		\includegraphics[width=\textwidth]{fig/an_laplace_myula_10.pdf}
		\caption{10\textsuperscript{th} dimension}
	\end{subfigure}
	\newline
	\centering
	\begin{subfigure}[b]{0.24\textwidth}
		\centering
		\includegraphics[width=\textwidth]{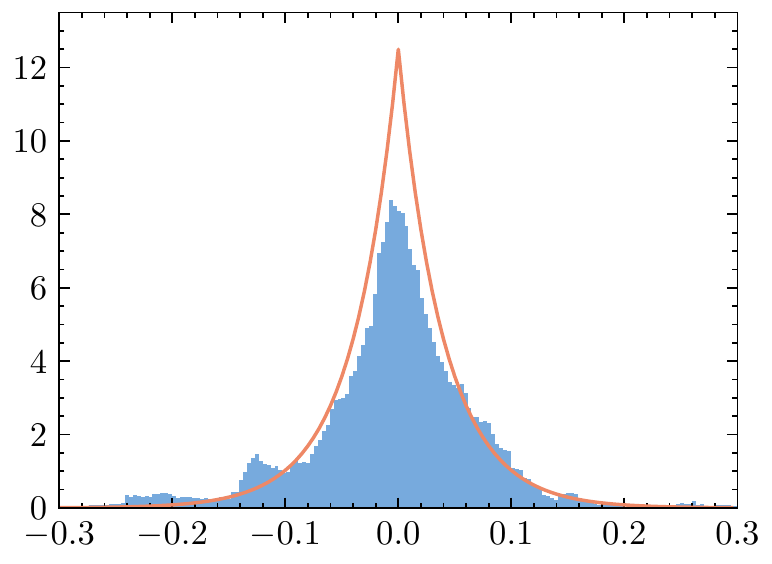}
		\caption{25\textsuperscript{th} dimension}
	\end{subfigure}
	\begin{subfigure}[b]{0.24\textwidth}
		\centering
		\includegraphics[width=\textwidth]{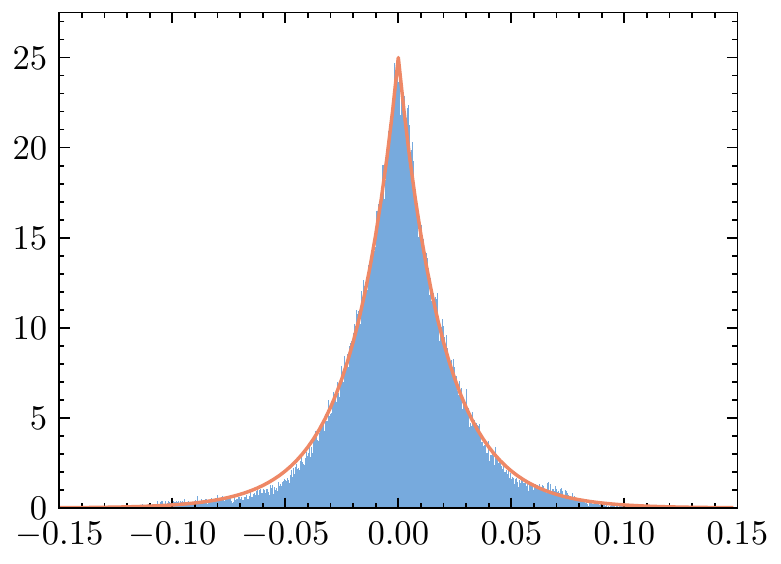}
		\caption{50\textsuperscript{th} dimension}
	\end{subfigure}
	\begin{subfigure}[b]{0.24\textwidth}
		\centering
		\includegraphics[width=\textwidth]{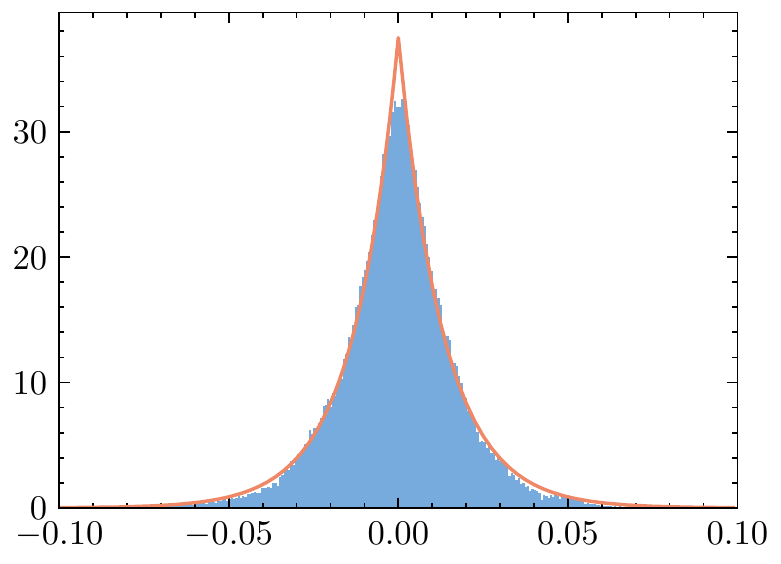}
		\caption{75\textsuperscript{th} dimension}
	\end{subfigure}
	\begin{subfigure}[b]{0.24\textwidth}
		\centering
		\includegraphics[width=\textwidth]{fig/an_laplace_myula_100.pdf}
		\caption{100\textsuperscript{th} dimension}
	\end{subfigure}
	\caption{Histograms of samples (blue) from MYULA and the true densities (orange). }
	\label{fig:an_laplace_3}
\end{figure}

\begin{figure}[h!]
	\centering
	\begin{subfigure}[b]{0.24\textwidth}
		\centering
		\includegraphics[width=\textwidth]{fig/an_laplace_bmumla_1.pdf}
		\caption{1\textsuperscript{st} dimension}
	\end{subfigure}
	\begin{subfigure}[b]{0.24\textwidth}
		\centering
		\includegraphics[width=\textwidth]{fig/an_laplace_bmumla_2.pdf}
		\caption{2\textsuperscript{nd} dimension}
	\end{subfigure}
	\begin{subfigure}[b]{0.24\textwidth}
		\centering
		\includegraphics[width=\textwidth]{fig/an_laplace_bmumla_5.pdf}
		\caption{5\textsuperscript{th} dimension}
	\end{subfigure}
	\begin{subfigure}[b]{0.24\textwidth}
		\centering
		\includegraphics[width=\textwidth]{fig/an_laplace_bmumla_10.pdf}
		\caption{10\textsuperscript{th} dimension}
	\end{subfigure}
	\newline
	\centering
	\begin{subfigure}[b]{0.24\textwidth}
		\centering
		\includegraphics[width=\textwidth]{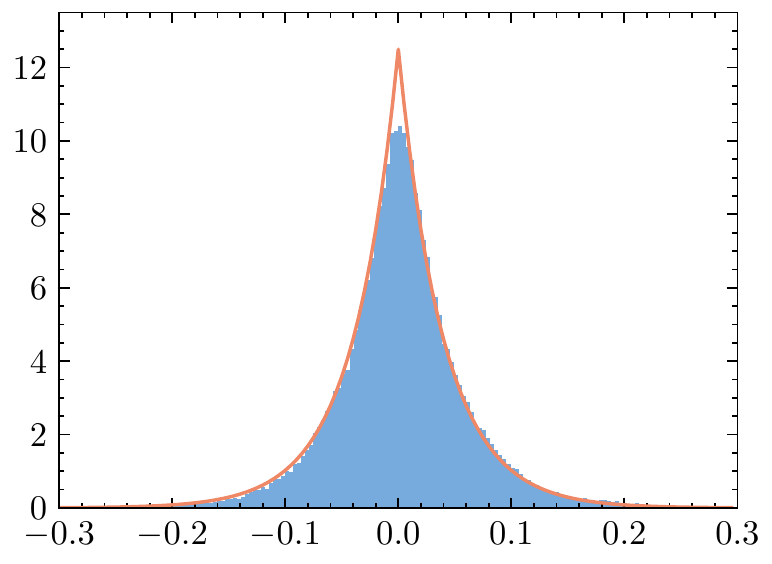}
		\caption{25\textsuperscript{th} dimension}
	\end{subfigure}
	\begin{subfigure}[b]{0.24\textwidth}
		\centering
		\includegraphics[width=\textwidth]{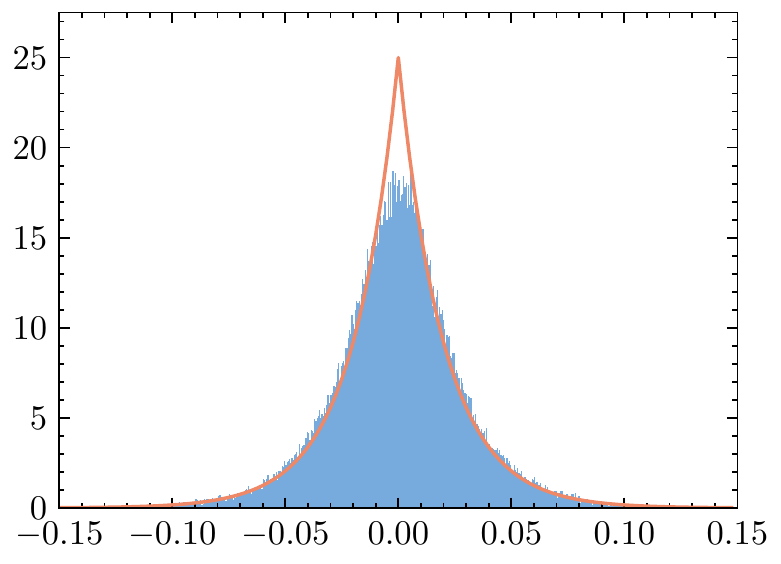}
		\caption{50\textsuperscript{th} dimension}
	\end{subfigure}
	\begin{subfigure}[b]{0.24\textwidth}
		\centering
		\includegraphics[width=\textwidth]{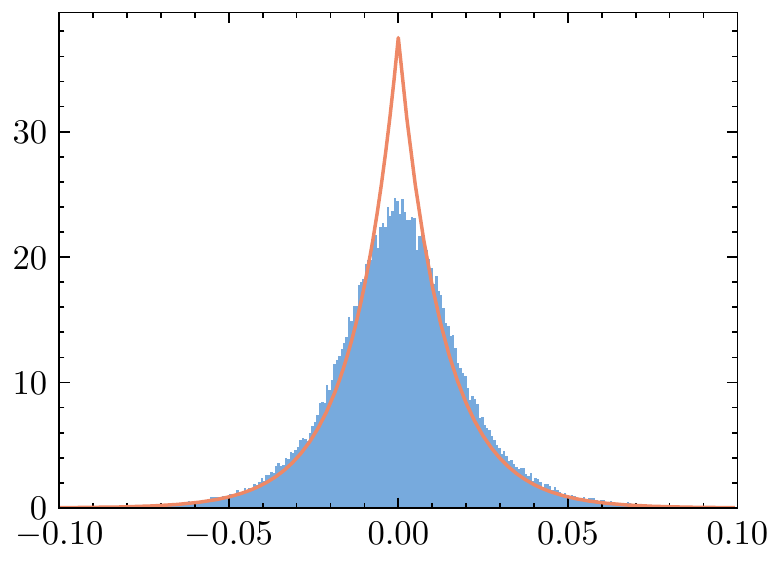}
		\caption{75\textsuperscript{th} dimension}
	\end{subfigure}
	\begin{subfigure}[b]{0.24\textwidth}
		\centering
		\includegraphics[width=\textwidth]{fig/an_laplace_bmumla_100.pdf}
		\caption{100\textsuperscript{th} dimension}
	\end{subfigure}
	\caption{Histograms of samples (blue) from BMUMLA and the true densities (orange). }
	\label{fig:an_laplace_4}
\end{figure}

We also give the experimental results using a different (left) Bregman--Moreau envelope introduced in \Cref{subsec:different}, using the same step size $\gamma=5\times10^{-6}$ in \Cref{fig:an_laplace_5}. 
\begin{figure}[h!]
	\centering
	\begin{subfigure}[b]{0.24\textwidth}
		\centering
		\includegraphics[width=\textwidth]{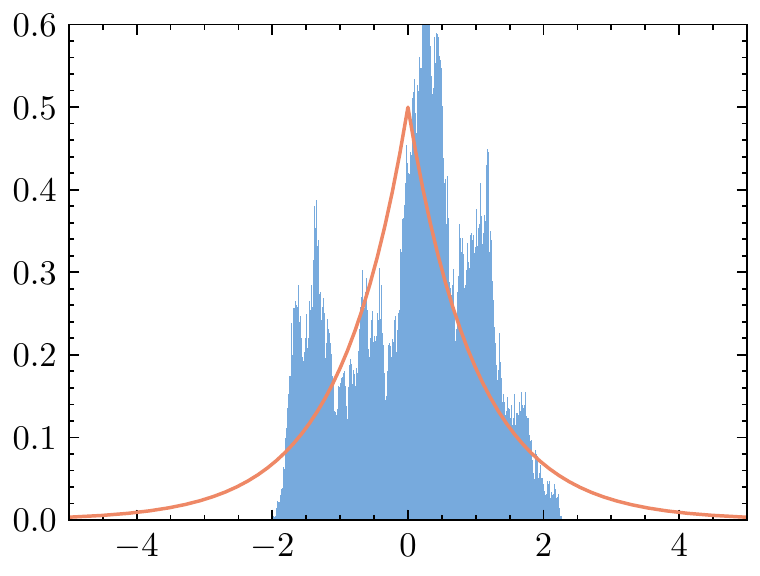}
		\caption{1\textsuperscript{st} dimension}
	\end{subfigure}
	\begin{subfigure}[b]{0.24\textwidth}
		\centering
		\includegraphics[width=\textwidth]{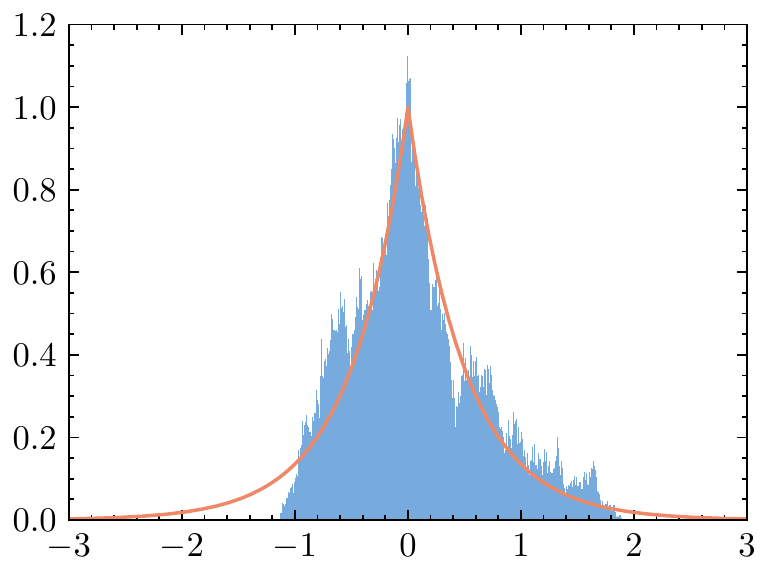}
		\caption{2\textsuperscript{nd} dimension}
	\end{subfigure}
	\begin{subfigure}[b]{0.24\textwidth}
		\centering
		\includegraphics[width=\textwidth]{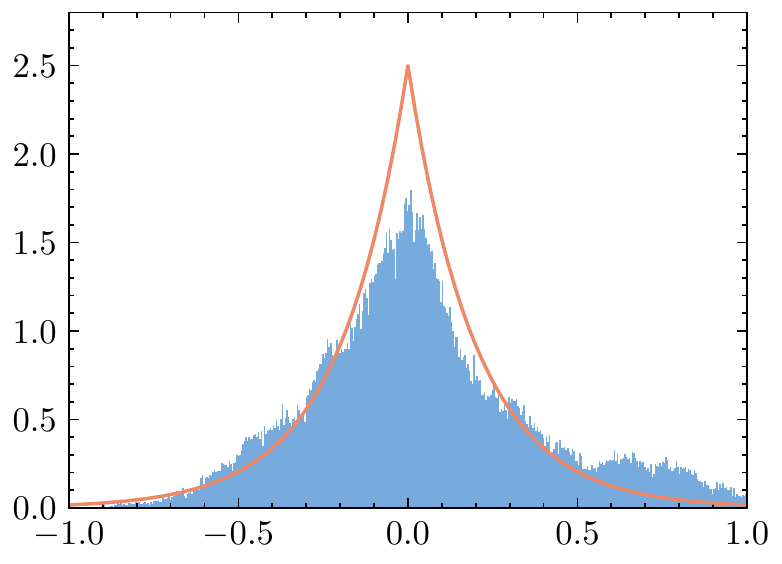}
		\caption{5\textsuperscript{th} dimension}
	\end{subfigure}
	\begin{subfigure}[b]{0.24\textwidth}
		\centering
		\includegraphics[width=\textwidth]{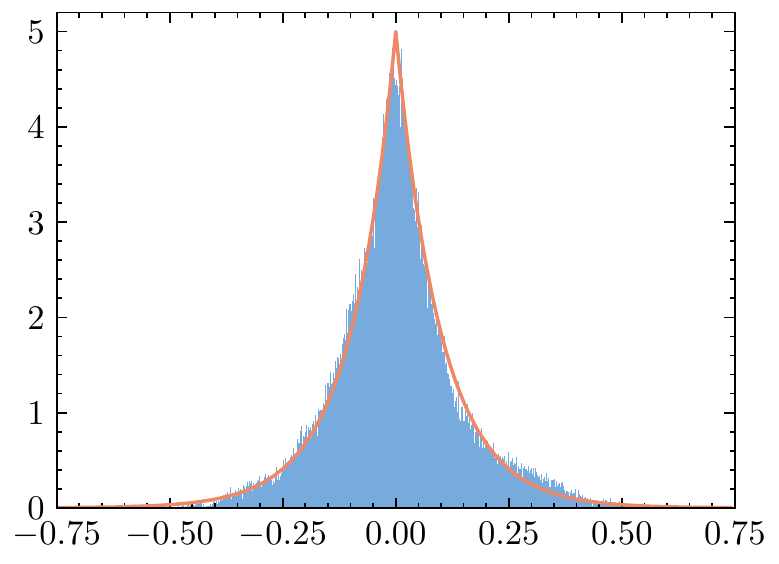}
		\caption{10\textsuperscript{th} dimension}
	\end{subfigure}
	\newline
	\centering
	\begin{subfigure}[b]{0.24\textwidth}
		\centering
		\includegraphics[width=\textwidth]{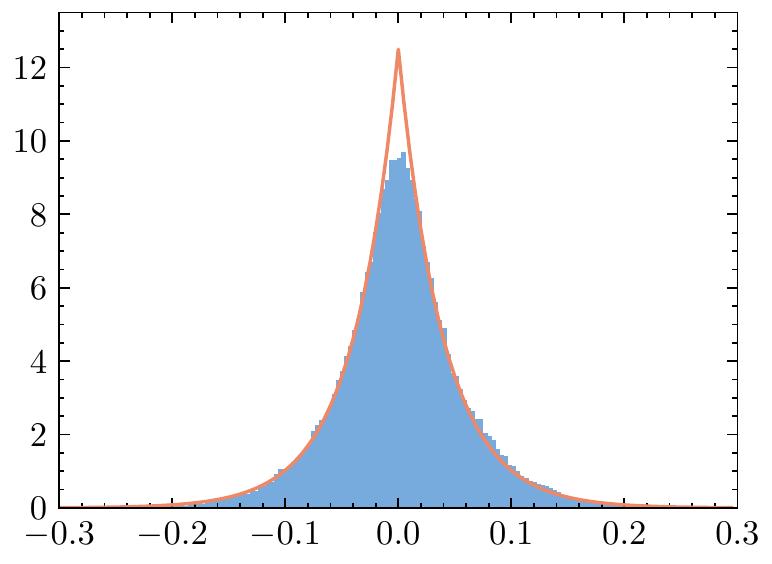}
		\caption{25\textsuperscript{th} dimension}
	\end{subfigure}
	\begin{subfigure}[b]{0.24\textwidth}
		\centering
		\includegraphics[width=\textwidth]{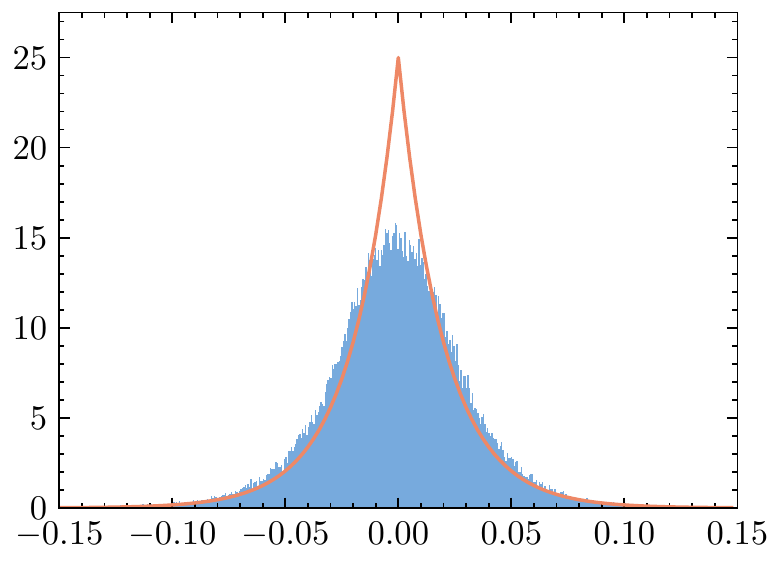}
		\caption{50\textsuperscript{th} dimension}
	\end{subfigure}
	\begin{subfigure}[b]{0.24\textwidth}
		\centering
		\includegraphics[width=\textwidth]{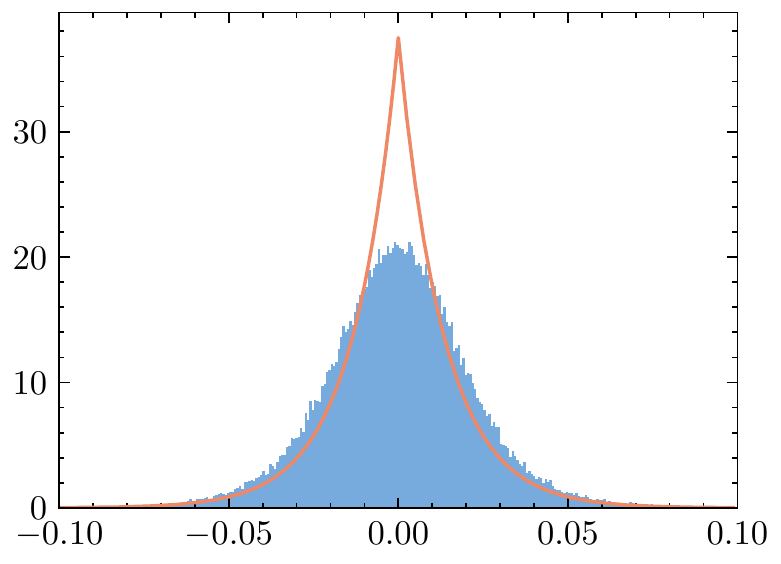}
		\caption{75\textsuperscript{th} dimension}
	\end{subfigure}
	\begin{subfigure}[b]{0.24\textwidth}
		\centering
		\includegraphics[width=\textwidth]{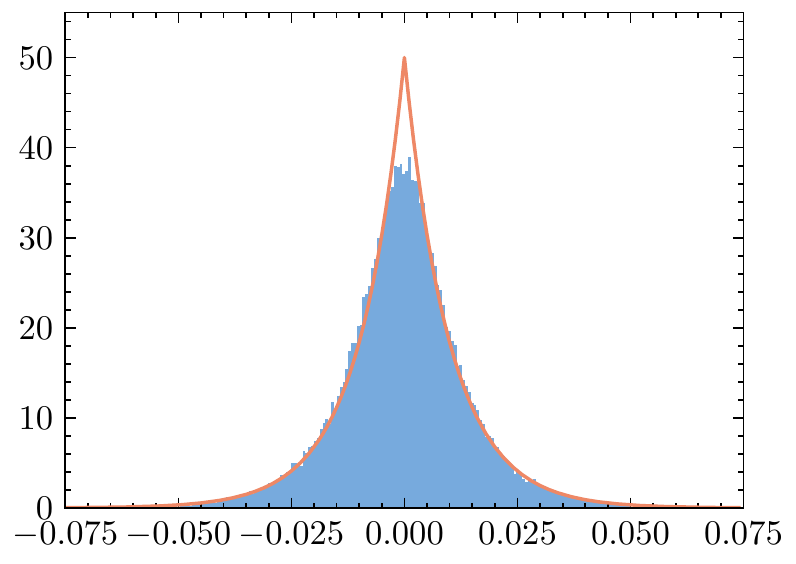}
		\caption{100\textsuperscript{th} dimension}
	\end{subfigure}
	\caption{Histograms of samples (blue) from LBMUMLA and the true densities (orange) with a different Bregman--Moreau envelope. }
	\label{fig:an_laplace_5}
\end{figure}

On the other hand, for practical purpose, we also perform the same set of experiments with another Bregman--Moreau envelope associated to the Legendre function $\psi(\btheta) = \sumd \e^{\theta_i}$. This is chosen particularly because we can compute the corresponding closed form expressions of both of its associated left and right Bregman proximity operators. To do so, we compute the following left and right Bregman proximity operators associated to the exponential function. 	
\newpage
\begin{proposition}\label{prop:bregman_prox}
	The left Bregman proximity operator of $\alpha|\cdot|$ associated to the Legendre function $\exp$ for $\alpha>0$ is 
	\[
	\lprox{\lambda, \alpha|\cdot|}{\exp}(\theta) = 
	\begin{cases*}
		\log(\e^\theta - \alpha\lambda) & if $\theta > \log(1+\alpha\lambda)$,\\
		\log(\e^\theta + \alpha\lambda) & if $\theta <\log(1-\alpha\lambda)$,\\
		1-\e^\theta(1+\theta) & otherwise.
	\end{cases*}
	\]
	The right Bregman proximity operator of $\alpha|\cdot|$ associated to the Legendre function $\exp$ for $\alpha>0$ is 
	\[
	\rprox{\lambda, \alpha|\cdot|}{\exp}(\theta) = 
	\begin{cases*}
		W(-\alpha\lambda\e^{-\theta}) + \theta & if $\theta > \alpha\lambda$,\\
		W(\alpha\lambda\e^{-\theta}) + \theta & if $\theta < -\alpha\lambda$,\\
		\e^\theta - (1+\theta) & otherwise, 
	\end{cases*}
	\]
	where $W$ is the Lambert $W$ function \citep{lambert1758observationes,corless1996lambertw}, i.e., the inverse of $\xi\mapsto\xi\e^\xi$ on $\RP$. 
	\begin{proof}[Proof of \Cref{prop:bregman_prox}]
		According to \Cref{def:bregman_prox}, 
		\[\lprox{\lambda, \alpha|\cdot|}{\exp}(\theta) = \argmin_{\vartheta\in\RR} \,\left\{\lambda\alpha|\vartheta| + \e^\vartheta - \e^\theta - \e^\theta(\vartheta - \theta) \right\}. \]
		First-order conditions give
		\[\begin{cases*}
			\alpha\lambda + \e^\vartheta - \e^\theta = 0 & if $\vartheta > 0$, \\
			-\alpha\lambda + \e^\vartheta - \e^\theta = 0 & if $\vartheta < 0$, 
		\end{cases*}\]
		which implies 
		\begin{equation}\label{eqn:left_bregman_prox_1}
			\vartheta^\star = 
			\begin{cases*}
				\log(\e^\theta - \alpha\lambda) & if $\vartheta^\star > 0$, \\
				\log(\e^\theta + \alpha\lambda) & if $\vartheta^\star < 0$ 
			\end{cases*} = 
			\begin{cases*}
				\log(\e^\theta - \alpha\lambda) & if $\theta > \log(1+\alpha\lambda)$, \\
				\log(\e^\theta + \alpha\lambda) & if $\theta < \log(1-\alpha\lambda)$. 
			\end{cases*}
		\end{equation}
		On the other hand, if $\vartheta = 0$, then 
		\begin{equation}\label{eqn:left_bregman_prox_2}
			\argmin_{\vartheta\in\RR} \,\left\{\lambda\alpha|\vartheta| + \e^\vartheta - \e^\theta - \e^\theta(\vartheta - \theta) \right\} = \argmin_{\vartheta\in\RR} \,\left\{1 - \e^\theta + \theta\e^\theta\right\} = 1-\e^\theta(1+\theta), 
		\end{equation}
		which corresponds to the range $[\log(1-\alpha\lambda), \log(1+\alpha\lambda)]$ for $\theta$. 
		Combining \eqref{eqn:left_bregman_prox_1} and \eqref{eqn:left_bregman_prox_2} yields the first desired result. 
		
		Again, according to \Cref{def:bregman_prox}, 
		\[\rprox{\lambda, \alpha|\cdot|}{\exp}(\theta) = \argmin_{\vartheta\in\RR} \,\left\{\lambda\alpha|\vartheta| + \e^\theta - \e^\vartheta - \e^\vartheta(\theta - \vartheta) \right\}. \]
		First-order conditions give
		\[\begin{cases*}
			\alpha\lambda - \e^\vartheta(\theta-\vartheta) = 0 & if $\vartheta > 0$, \\
			-\alpha\lambda - \e^\vartheta(\theta-\vartheta) = 0 & if $\vartheta < 0$ 
		\end{cases*} \Leftrightarrow 
		\begin{cases*}
			(\vartheta-\theta)\e^{\vartheta-\theta} = -\alpha\lambda\e^{-\theta} & if $\vartheta > 0$, \\
			(\vartheta-\theta)\e^{\vartheta-\theta} = \alpha\lambda\e^{-\theta} & if $\vartheta < 0$, 
		\end{cases*}
		\]
		which implies 
		\begin{align}\label{eqn:right_bregman_prox_1}
			\vartheta^\star &= 
			\begin{cases*}
				W(-\alpha\lambda\e^{-\theta})+\theta & if $\vartheta^\star > 0$ and $-\alpha\lambda\e^{-\theta} \ge -\e^{-1}$, \\
				W(\alpha\lambda\e^{-\theta})+\theta & if $\vartheta^\star < 0$ 
			\end{cases*} \nonumber\\
			&= 
			\begin{cases*}
				W(-\alpha\lambda\e^{-\theta})+\theta & if $\theta > \alpha\lambda$ and $\theta\ge \log(\alpha\lambda) + 1$, \\
				W(\alpha\lambda\e^{-\theta})+\theta & if $\theta < -\alpha\lambda$. 
			\end{cases*} \nonumber\\
			&=
			\begin{cases*}
				W(-\alpha\lambda\e^{-\theta})+\theta & if $\theta > \alpha\lambda$, \\
				W(\alpha\lambda\e^{-\theta})+\theta & if $\theta < -\alpha\lambda$. 
			\end{cases*} 
		\end{align}
		since $u \ge \log u + 1$ for any $u>0$. Notice that the condition $-\alpha\lambda\e^{-\theta} \ge -\e^{-1}$ is required for the Lambert $W$ function to be defined for a negative value. 
		
		On the other hand, if $\vartheta = 0$, then 
		\begin{equation}\label{eqn:right_bregman_prox_2}
			\argmin_{\vartheta\in\RR} \,\left\{\lambda\alpha|\vartheta| + \e^\theta - \e^\vartheta - \e^\vartheta(\theta - \vartheta) \right\} = \argmin_{\vartheta\in\RR} \,\left\{\e^\theta - 1 - \theta\right\} = \e^\theta - (1+\theta), 
		\end{equation}
		which corresponds to the range $[-\alpha\lambda, \alpha\lambda]$ for $\theta$. 
		Combining \eqref{eqn:right_bregman_prox_1} and \eqref{eqn:right_bregman_prox_2} yields the second desired result.
	\end{proof}
\end{proposition}

The corresponding experiments are illustrated in \Cref{fig:an_laplace_2}. BMMMLA (\Cref{fig:an_laplace_6}) are also used in this setting. We observe that the right variants perform comparably to the left ones, both outperforming MYULA at the wide marginals (i.e., the lower dimensions). 
\begin{figure}[h!]
	\begin{subfigure}[b]{0.24\textwidth}
		\centering
		\includegraphics[width=\textwidth]{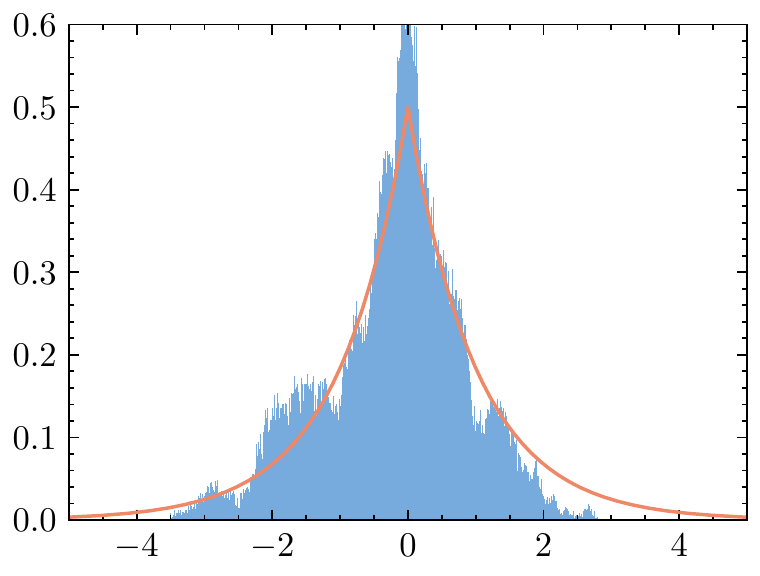}
	\end{subfigure}
	\hfill
	\begin{subfigure}[b]{0.24\textwidth}
		\centering
		\includegraphics[width=\textwidth]{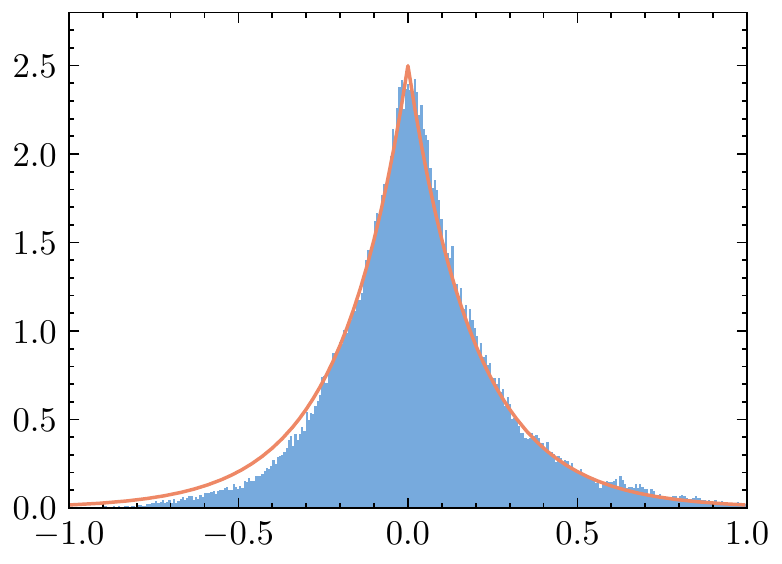}
	\end{subfigure}
	\begin{subfigure}[b]{0.24\textwidth}
		\centering
		\includegraphics[width=\textwidth]{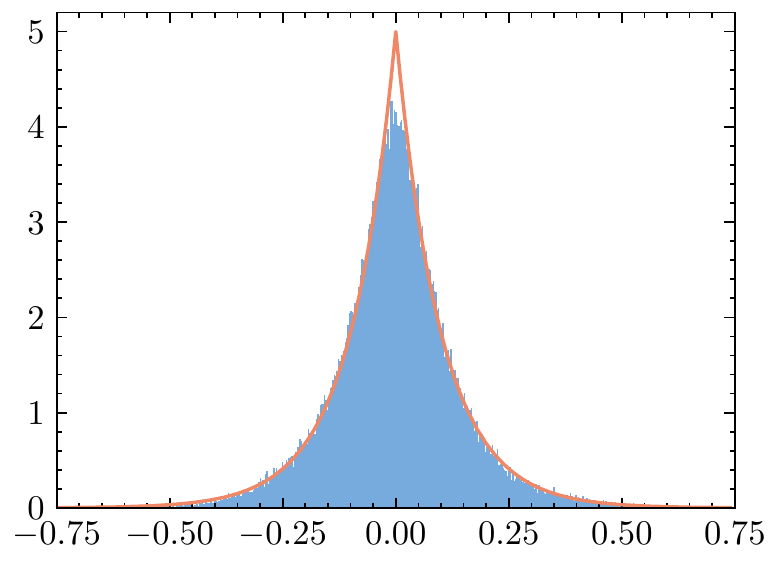}
	\end{subfigure}
	\begin{subfigure}[b]{0.24\textwidth}
		\centering
		\includegraphics[width=\textwidth]{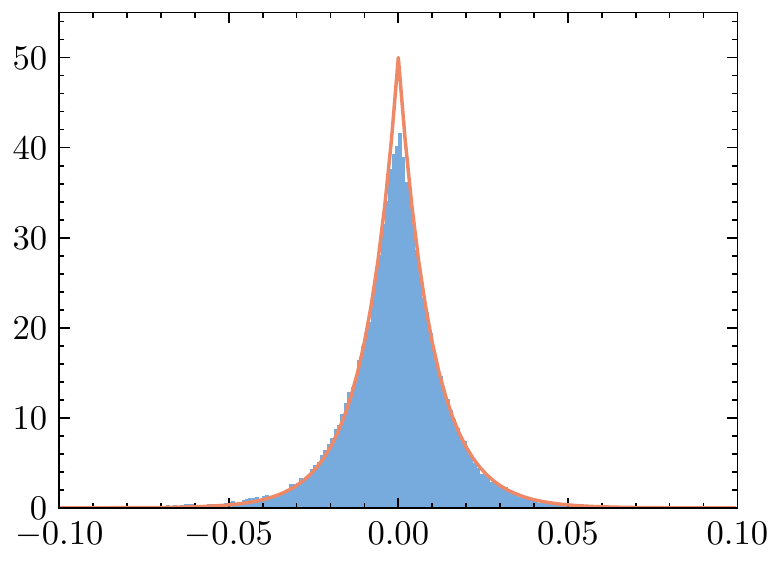}
	\end{subfigure}
	\newline
	\begin{subfigure}[b]{0.24\textwidth}
		\centering
		\includegraphics[width=\textwidth]{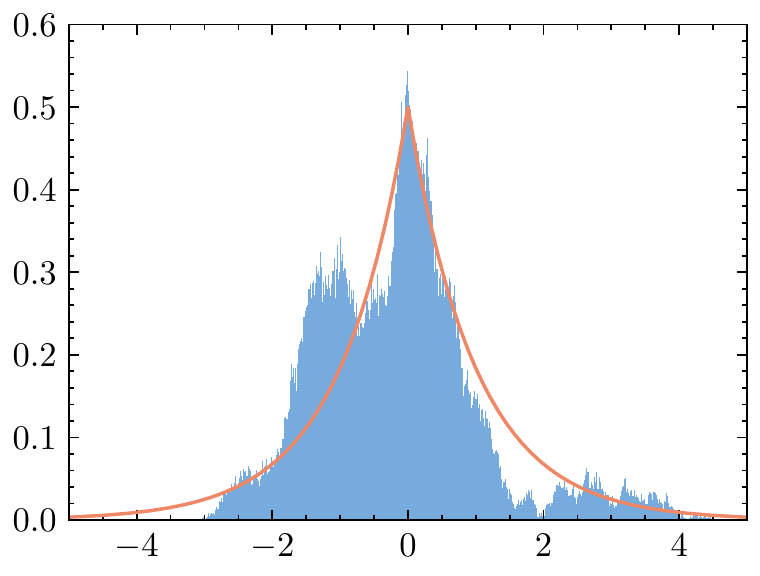}
		\caption{1\textsuperscript{st} dimension}
		\label{fig:rbmumla_1}
	\end{subfigure}
	\hfill
	\begin{subfigure}[b]{0.24\textwidth}
		\centering
		\includegraphics[width=\textwidth]{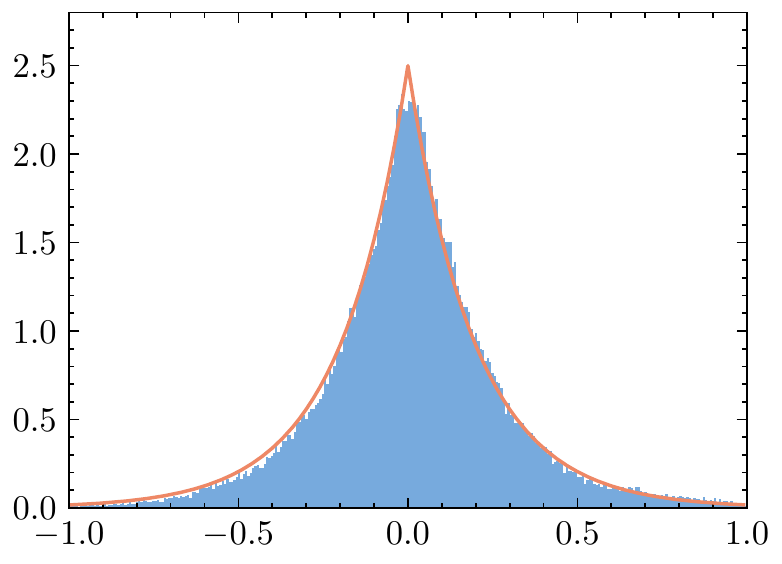}
		\caption{5\textsuperscript{th} dimension}
		\label{fig:rbmumla_5}
	\end{subfigure}
	\begin{subfigure}[b]{0.24\textwidth}
		\centering
		\includegraphics[width=\textwidth]{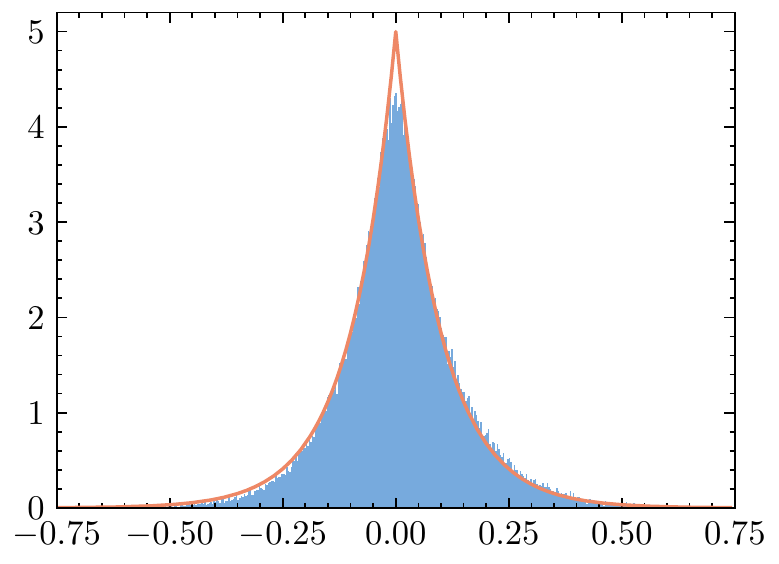}
		\caption{10\textsuperscript{th} dimension}
		\label{fig:rbmumla_10}
	\end{subfigure}
	\begin{subfigure}[b]{0.24\textwidth}
		\centering
		\includegraphics[width=\textwidth]{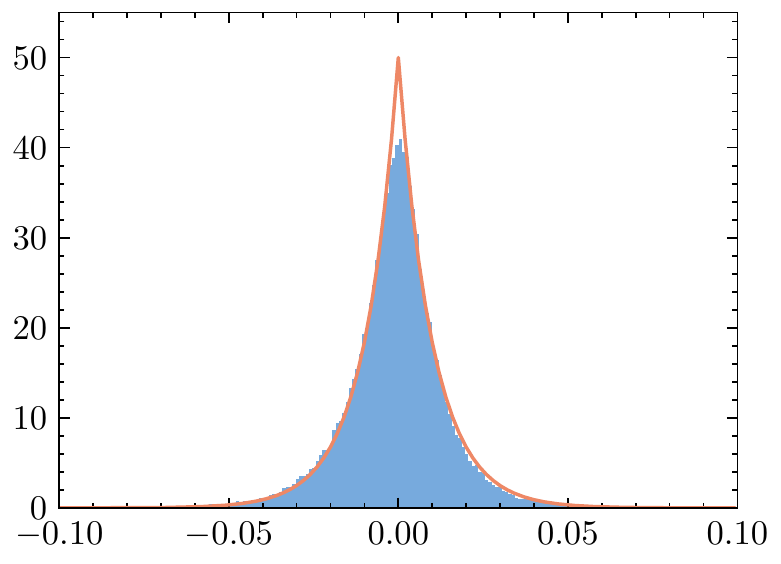}
		\caption{100\textsuperscript{th} dimension}
		\label{fig:rbmumla_100}
	\end{subfigure}
	
	\caption{Histograms of samples (in blue) from left BMUMLA (1\textsuperscript{st} row), right BMUMLA (2\textsuperscript{nd} row) and the true densities (in orange). 
	}
	\label{fig:an_laplace_2}
\end{figure}
\begin{figure}[h!]
	\begin{subfigure}[b]{0.24\textwidth}
		\centering
		\includegraphics[width=\textwidth]{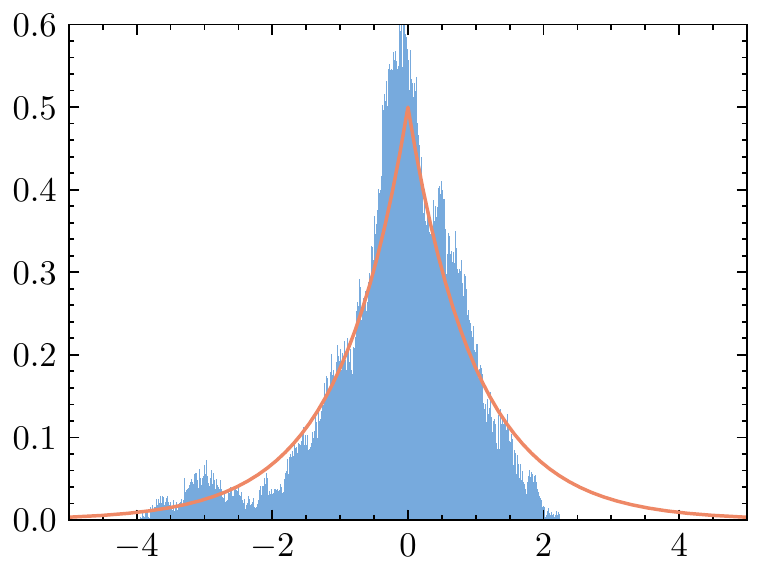}
	\end{subfigure}
	\hfill
	\begin{subfigure}[b]{0.24\textwidth}
		\centering
		\includegraphics[width=\textwidth]{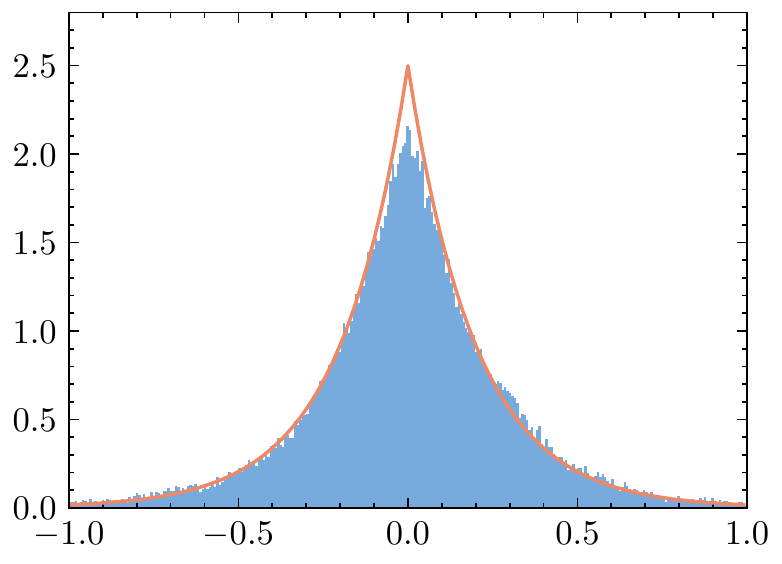}
	\end{subfigure}
	\begin{subfigure}[b]{0.24\textwidth}
		\centering
		\includegraphics[width=\textwidth]{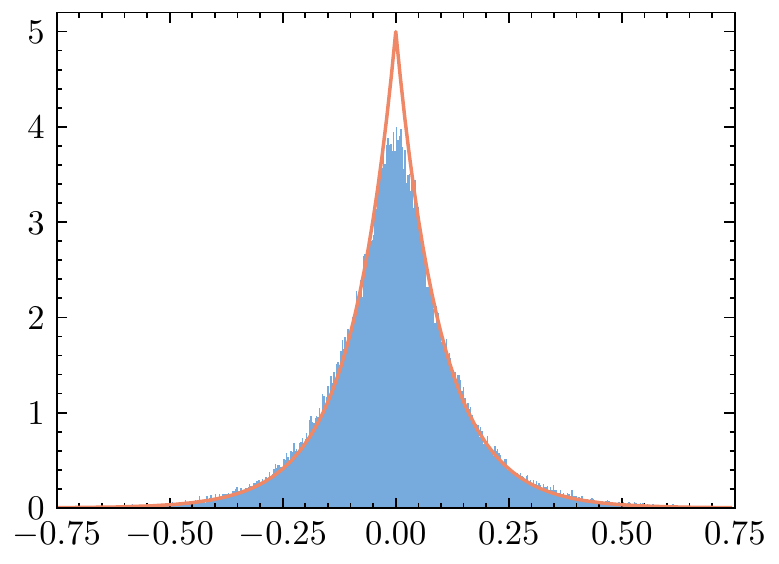}
	\end{subfigure}
	\begin{subfigure}[b]{0.24\textwidth}
		\centering
		\includegraphics[width=\textwidth]{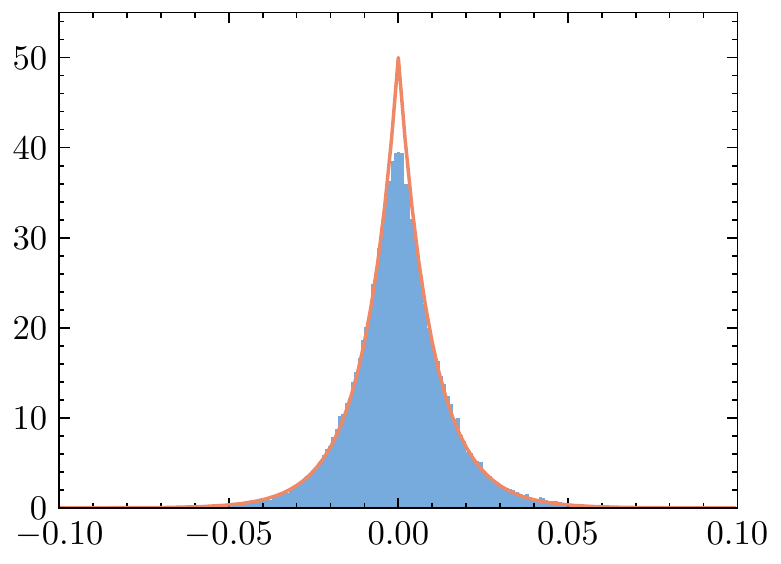}
	\end{subfigure}
	\newline
	\begin{subfigure}[b]{0.24\textwidth}
		\centering
		\includegraphics[width=\textwidth]{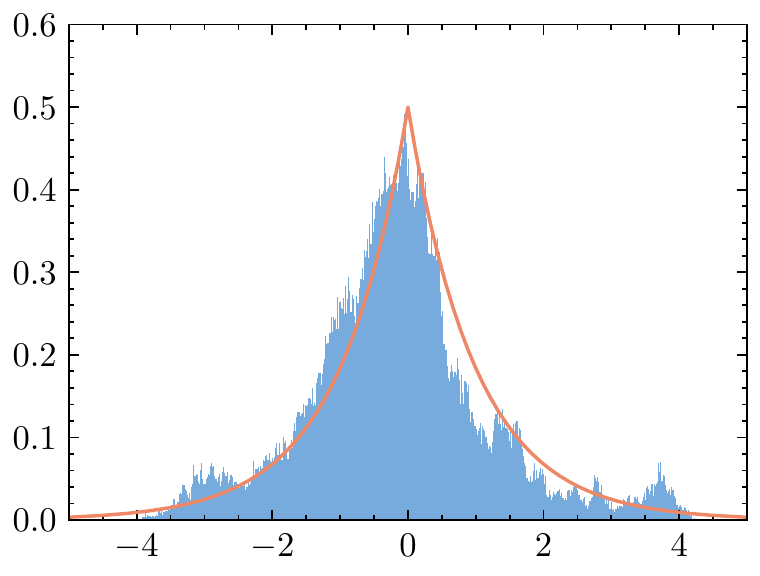}
		\caption{1\textsuperscript{st} dimension}
		\label{fig:rbmmmla_1}
	\end{subfigure}
	\hfill
	\begin{subfigure}[b]{0.24\textwidth}
		\centering
		\includegraphics[width=\textwidth]{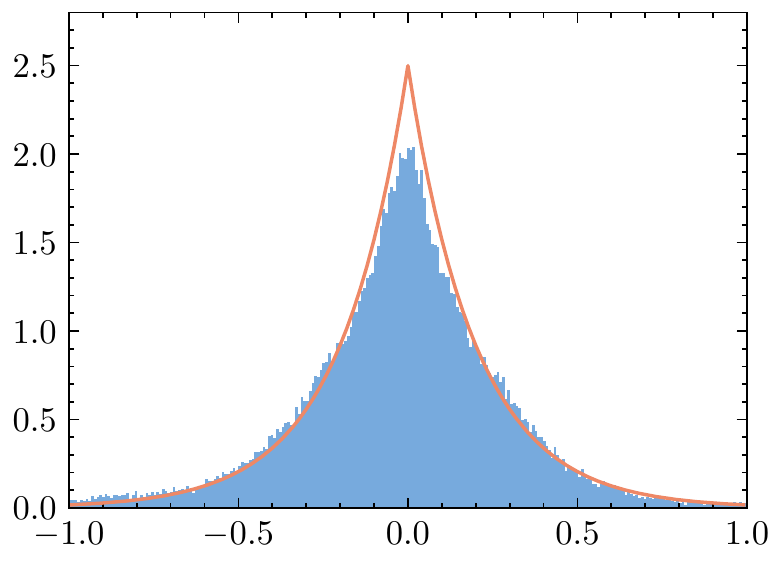}
		\caption{5\textsuperscript{th} dimension}
		\label{fig:rbmmmla_5}
	\end{subfigure}
	\begin{subfigure}[b]{0.24\textwidth}
		\centering
		\includegraphics[width=\textwidth]{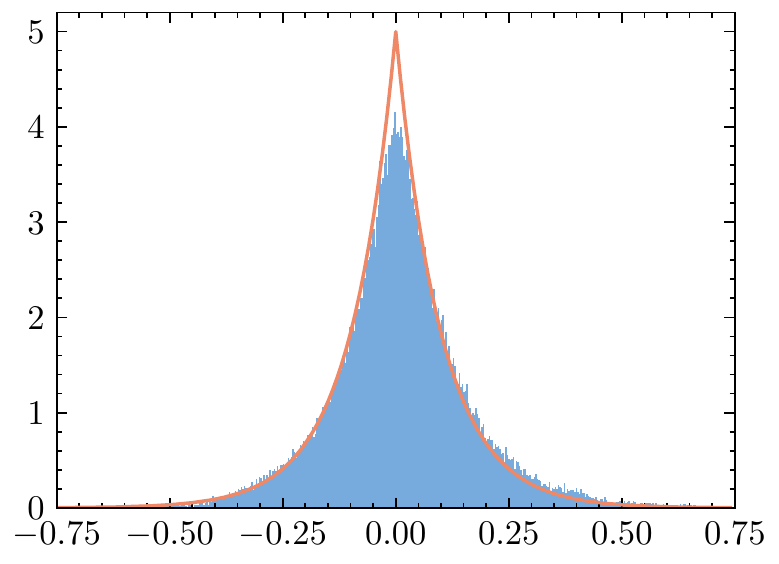}
		\caption{10\textsuperscript{th} dimension}
		\label{fig:rbmmmla_10}
	\end{subfigure}
	\begin{subfigure}[b]{0.24\textwidth}
		\centering
		\includegraphics[width=\textwidth]{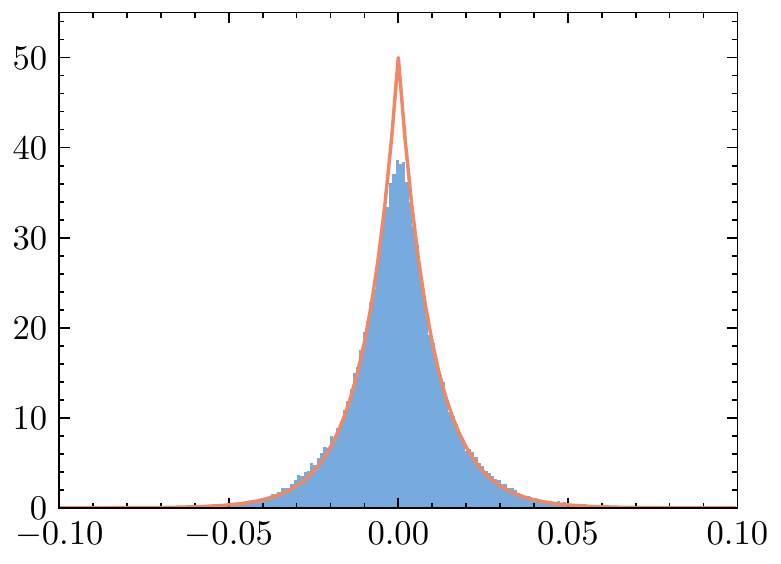}
		\caption{100\textsuperscript{th} dimension}
		\label{fig:rbmmmla_100}
	\end{subfigure}
	\caption{Histograms of samples (in blue) from left BMMMLA (1\textsuperscript{st} row), right BMMMLA (2\textsuperscript{nd} row) and the true densities (in orange). 
	}
	\label{fig:an_laplace_6}
\end{figure}

\subsection{Anisotropic Uniform Distribution}
We consider the task of sampling from an anisotropic uniform distribution over the set $\euC \coloneqq \prod_{i=1}^d [a_i, b_i]$, where $\ba=(a_i)_{1\le i\le d}^\top\in\RR^d$ and $\bb=(b_i)_{1\le i\le d}^\top\in\RR^d$. To perform this task using our proposed algorithm, we let $f=0$ and $g = \iota_\euC$. Note that the original mirror Langevin algorithm cannot apply to sampling uniform distributions, as mentioned in \citet{li2022mirror}, as $f=0$. However, by suitably choosing a Bregman--Moreau envelope, we can still perform approximate sampling (as opposed to exact sampling) using the BMMMLA. 

Note that when $g = \iota_\euC$ with $\euC\subseteq\RR^d$ being a closed convex set, the  Bregman proximity operators of $g$ are the Bregman projections (or projectors) onto $\euC$, as illustrated in the following definition \citep{bauschke2018regularizing}. 	
\begin{definition}[Bregman projections]
	Let $\euC\subseteq\RR^d$ be a closed convex set such that $\euX\cap\euC\ne\varnothing$, then $\lprox{\euC}{\varphi} \coloneqq \lprox{\iota_\euC}{\varphi}$ and $\rprox{\euC}{\varphi} \coloneqq \rprox{\iota_\euC}{\varphi}$ are the \emph{left} and \emph{right Bregman projections} onto $\euC$ respectively. 
\end{definition}

For simplicity, we choose $\psi = \frac12\euclidnorm{\cdot}^2$. Then the Bregman projection onto $\euC$ boils down to the Euclidean projection onto $\euC$, which is given by 
\[\lprox{\euC}{\varphi}(\btheta) = \rprox{\euC}{\varphi}(\btheta) = \proj_{\euC}(\btheta) = \left( \min\{b_i, \max\{a_i, \theta_i\}\} \right)_{1\le i\le d}^\top. \]  
In the experiment, we consider the case where $a_i = -i$ and $b_i = i$ for all $i\in\setd$, so that the target uniform distribution on $\euC=[-1, 1] \times [-2, 2] \times \cdots \times [-d, d]$ is anisotropic, varying significantly across different dimensions. 	
We use $\gamma=0.01$, $\lambda=1$ and $\bbeta=(2\sqrt{d-i+1})_{1\le i\le d}^\top$, and give the experimental results in \Cref{fig:an_uniform_myula,fig:an_uniform_bmumla}. We observe that BMUMLA outperforms MYULA at higher dimensions with wide marginals, where most samples lie in the desired ranges. 	
Also note that all of \Cref{assum:smooth,assum:nonsmooth,assum:mirror1,assum:mirror2,assum:Bregman,assum:functions} hold. See \Cref{fig:env_uniform} as a graphical verification of \Cref{assum:functions}, with $\alpha=2M_{\varphi_{\bbeta}}+0.1$ and $\beta_g = 250$.

\begin{figure}[h]		
	\centering
	\begin{subfigure}[b]{0.24\textwidth}
		\centering
		\includegraphics[width=\textwidth]{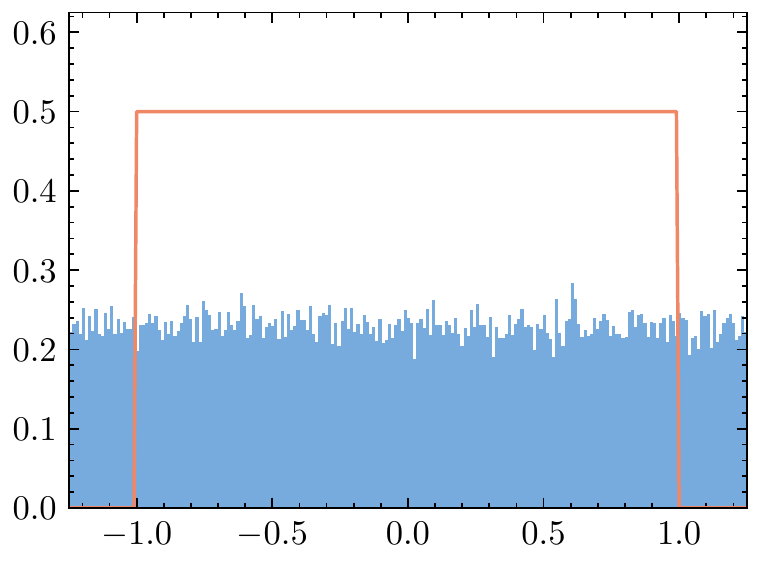}
		\caption{1\textsuperscript{st} dimension}
	\end{subfigure}
	\begin{subfigure}[b]{0.24\textwidth}
		\centering
		\includegraphics[width=\textwidth]{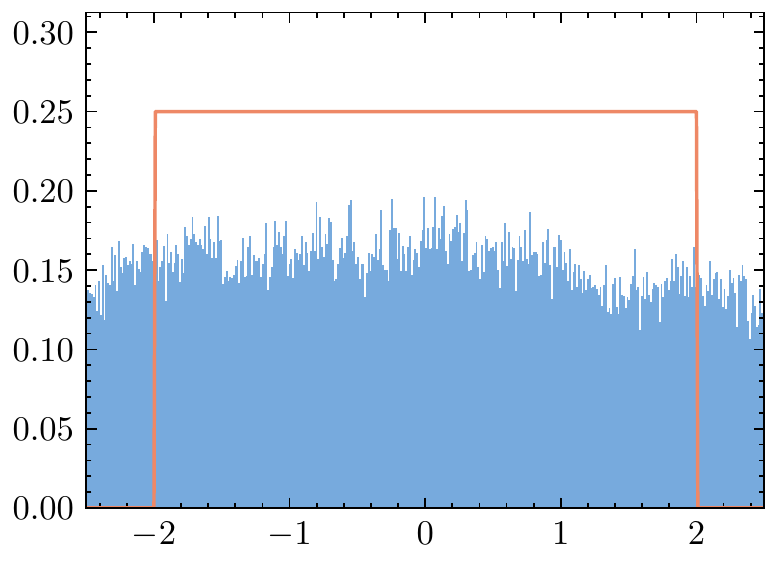}
		\caption{2\textsuperscript{nd} dimension}
	\end{subfigure}
	\begin{subfigure}[b]{0.24\textwidth}
		\centering
		\includegraphics[width=\textwidth]{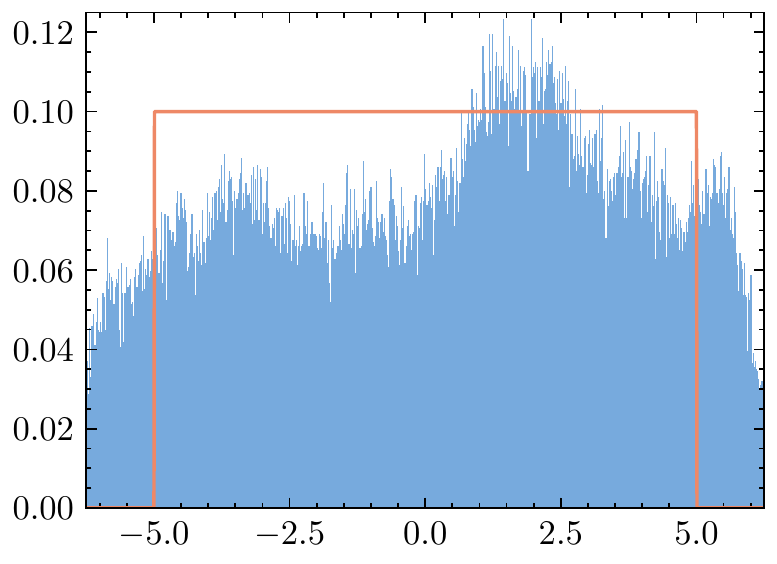}
		\caption{5\textsuperscript{th} dimension}
	\end{subfigure}
	\begin{subfigure}[b]{0.24\textwidth}
		\centering
		\includegraphics[width=\textwidth]{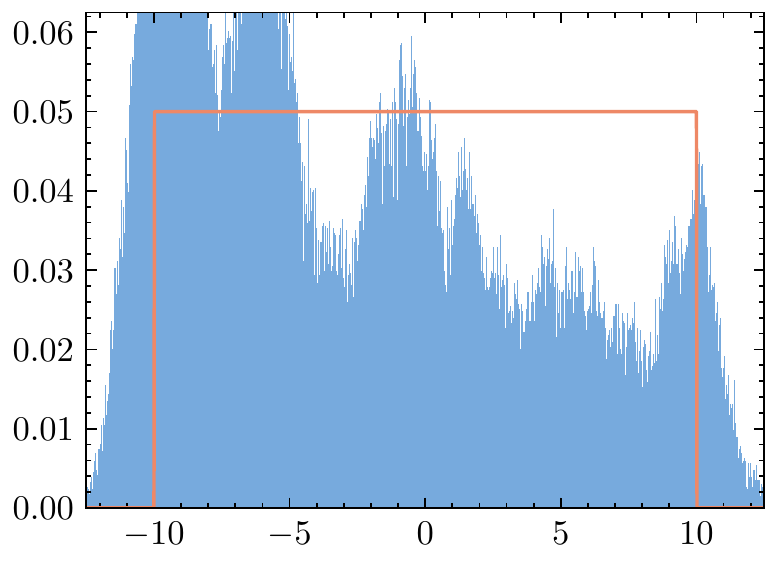}
		\caption{10\textsuperscript{th} dimension}
	\end{subfigure}
	\newline
	\centering
	\begin{subfigure}[b]{0.24\textwidth}
		\centering
		\includegraphics[width=\textwidth]{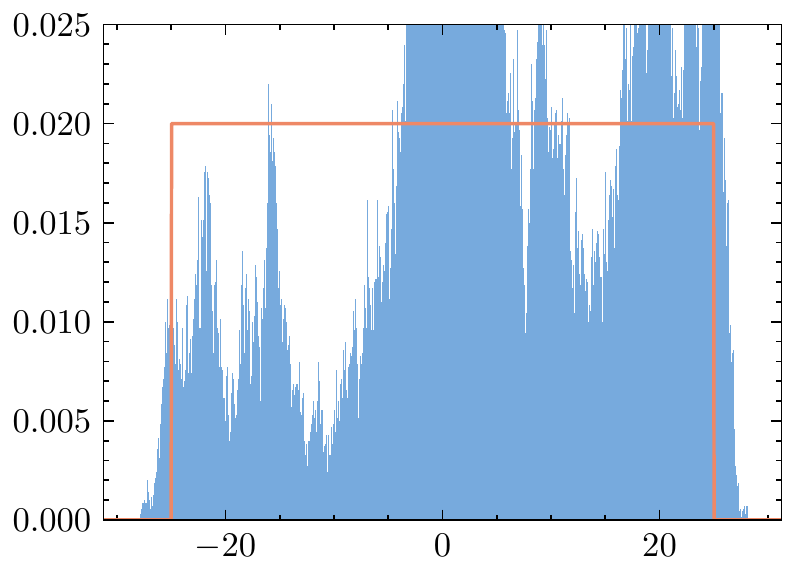}
		\caption{25\textsuperscript{th} dimension}
	\end{subfigure}
	\begin{subfigure}[b]{0.24\textwidth}
		\centering
		\includegraphics[width=\textwidth]{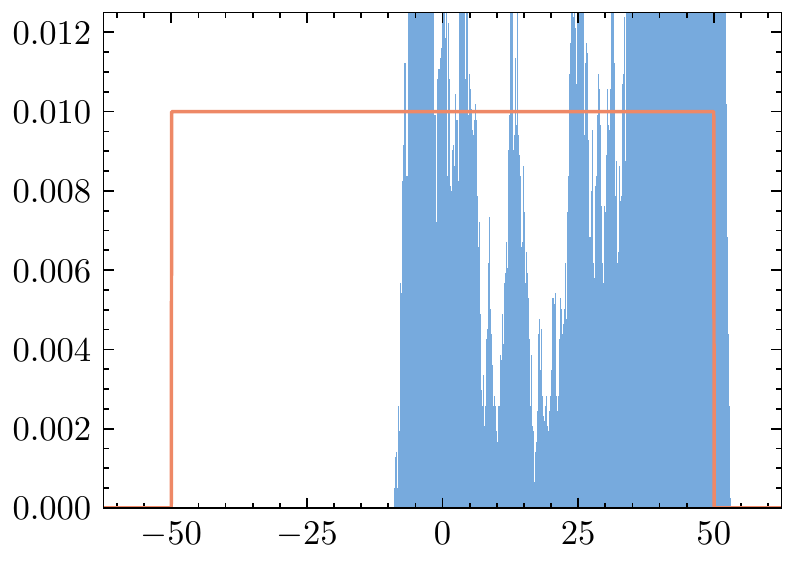}
		\caption{50\textsuperscript{th} dimension}
	\end{subfigure}
	\begin{subfigure}[b]{0.24\textwidth}
		\centering
		\includegraphics[width=\textwidth]{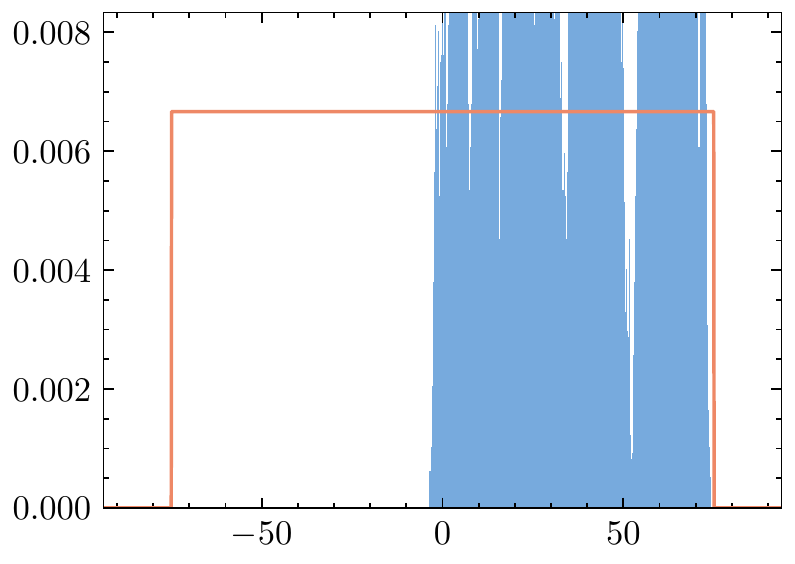}
		\caption{75\textsuperscript{th} dimension}
	\end{subfigure}
	\begin{subfigure}[b]{0.24\textwidth}
		\centering
		\includegraphics[width=\textwidth]{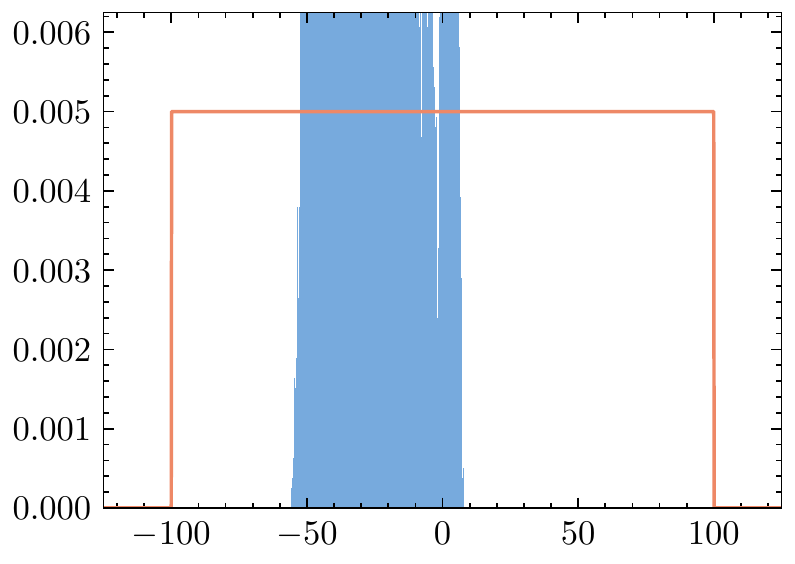}
		\caption{100\textsuperscript{th} dimension}
	\end{subfigure}
	\caption{Histograms of samples (blue) from MYULA and the true densities (orange) for uniform distribution on $\euC$. }
	\label{fig:an_uniform_myula}
\end{figure}

\begin{figure}[h!]
	\centering
	\begin{subfigure}[b]{0.24\textwidth}
		\centering
		\includegraphics[width=\textwidth]{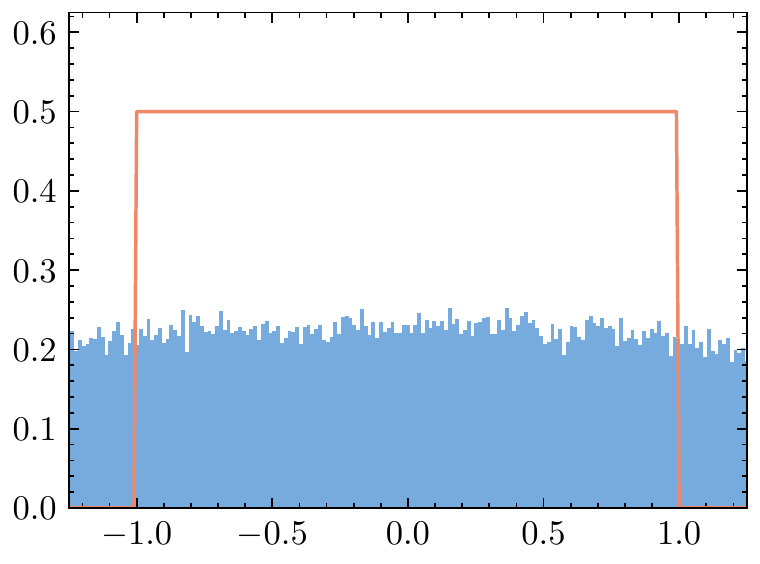}
		\caption{1\textsuperscript{st} dimension}
	\end{subfigure}
	\begin{subfigure}[b]{0.24\textwidth}
		\centering
		\includegraphics[width=\textwidth]{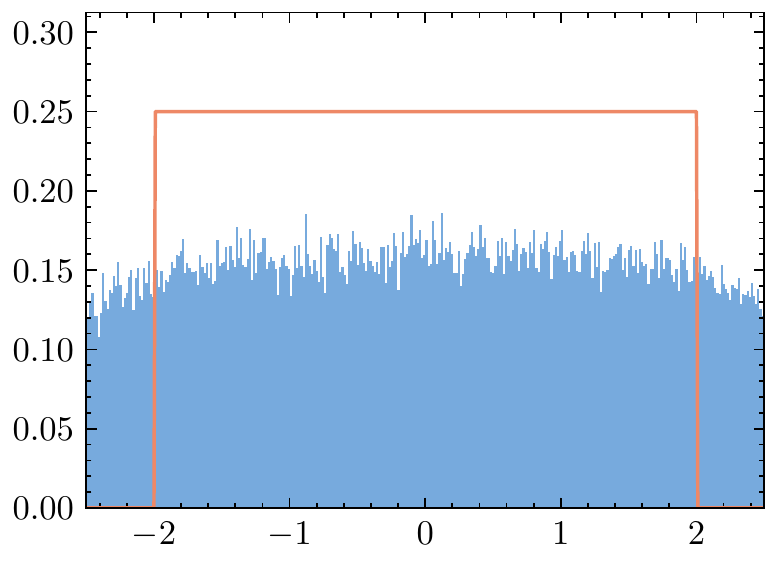}
		\caption{2\textsuperscript{nd} dimension}
	\end{subfigure}
	\begin{subfigure}[b]{0.24\textwidth}
		\centering
		\includegraphics[width=\textwidth]{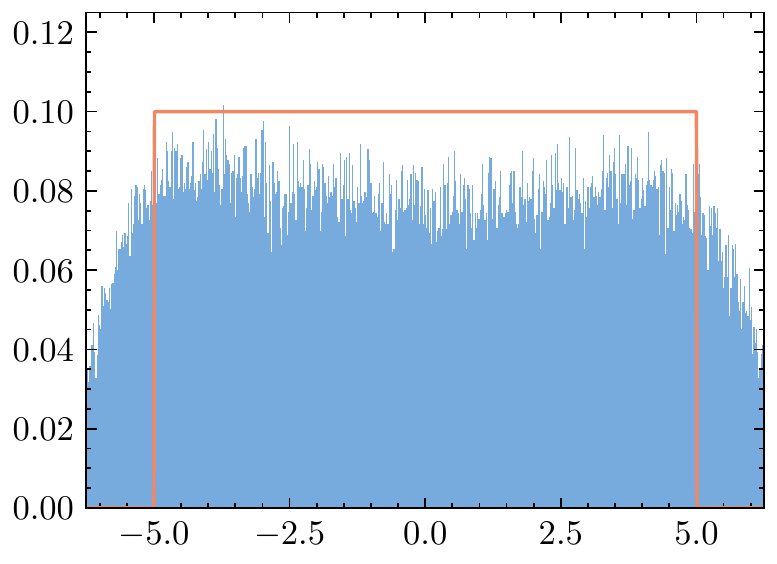}
		\caption{5\textsuperscript{th} dimension}
	\end{subfigure}
	\begin{subfigure}[b]{0.24\textwidth}
		\centering
		\includegraphics[width=\textwidth]{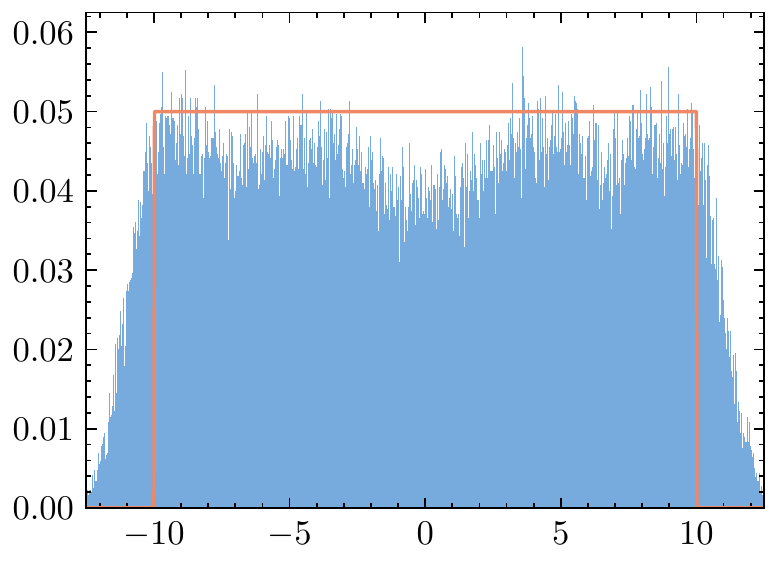}
		\caption{10\textsuperscript{th} dimension}
	\end{subfigure}
	\newline
	\centering
	\begin{subfigure}[b]{0.24\textwidth}
		\centering
		\includegraphics[width=\textwidth]{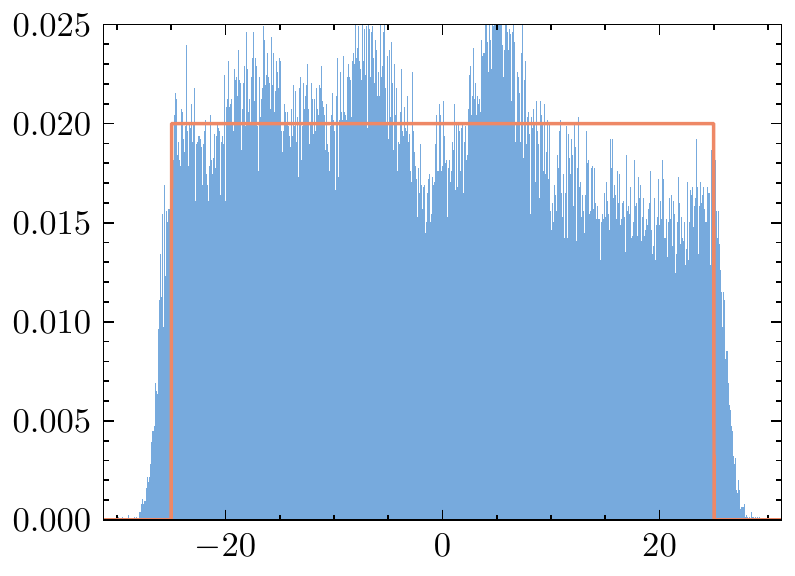}
		\caption{25\textsuperscript{th} dimension}
	\end{subfigure}
	\begin{subfigure}[b]{0.24\textwidth}
		\centering
		\includegraphics[width=\textwidth]{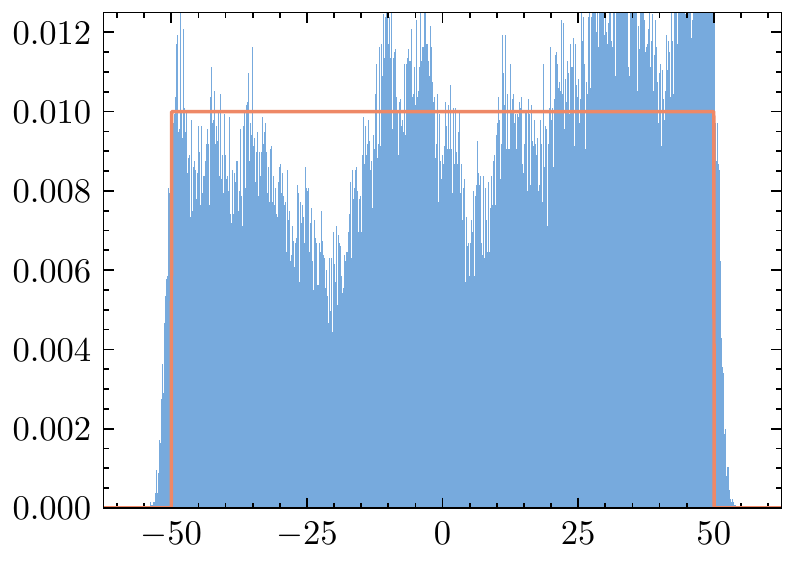}
		\caption{50\textsuperscript{th} dimension}
	\end{subfigure}
	\begin{subfigure}[b]{0.24\textwidth}
		\centering
		\includegraphics[width=\textwidth]{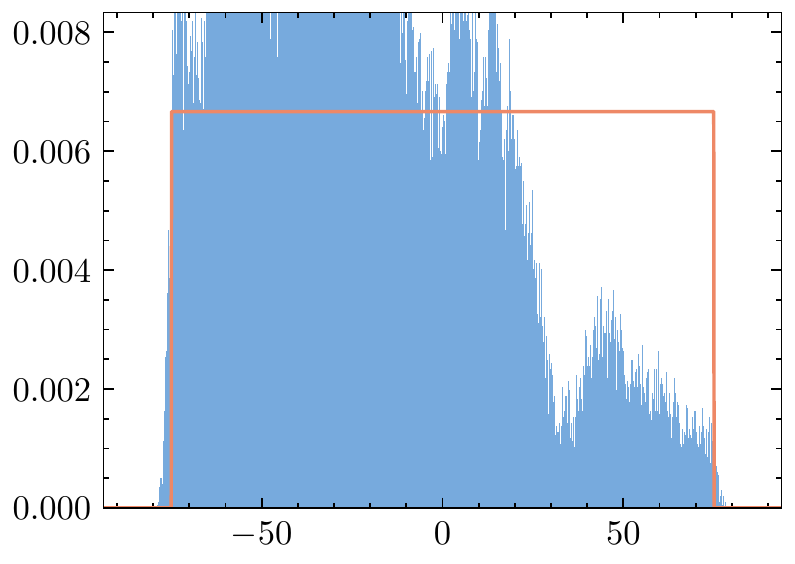}
		\caption{75\textsuperscript{th} dimension}
	\end{subfigure}
	\begin{subfigure}[b]{0.24\textwidth}
		\centering
		\includegraphics[width=\textwidth]{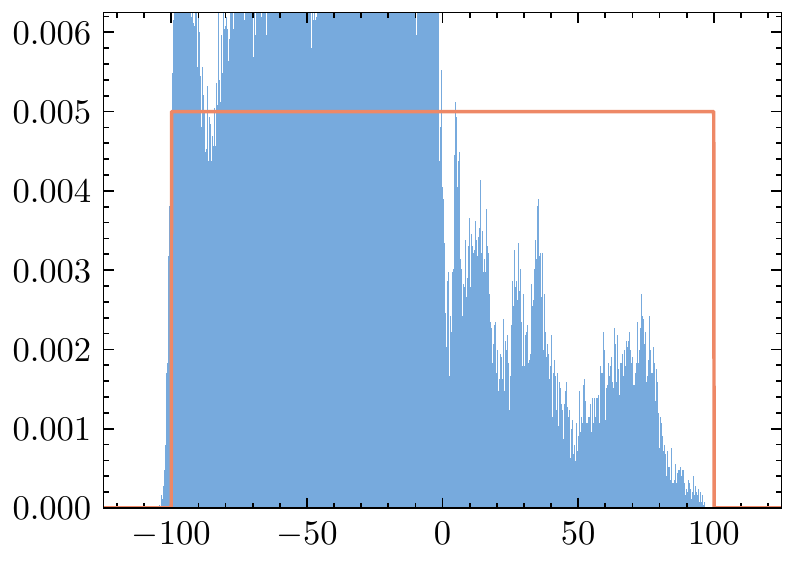}
		\caption{100\textsuperscript{th} dimension}
	\end{subfigure}
	\caption{Histograms of samples (blue) from BMUMLA and the true densities (orange) for uniform distribution on $\euC$. }
	\label{fig:an_uniform_bmumla}
\end{figure}

\begin{figure}[h!]
	\centering
	\begin{subfigure}[b]{.49\textwidth}
		\centering
		\includegraphics[width=\textwidth]{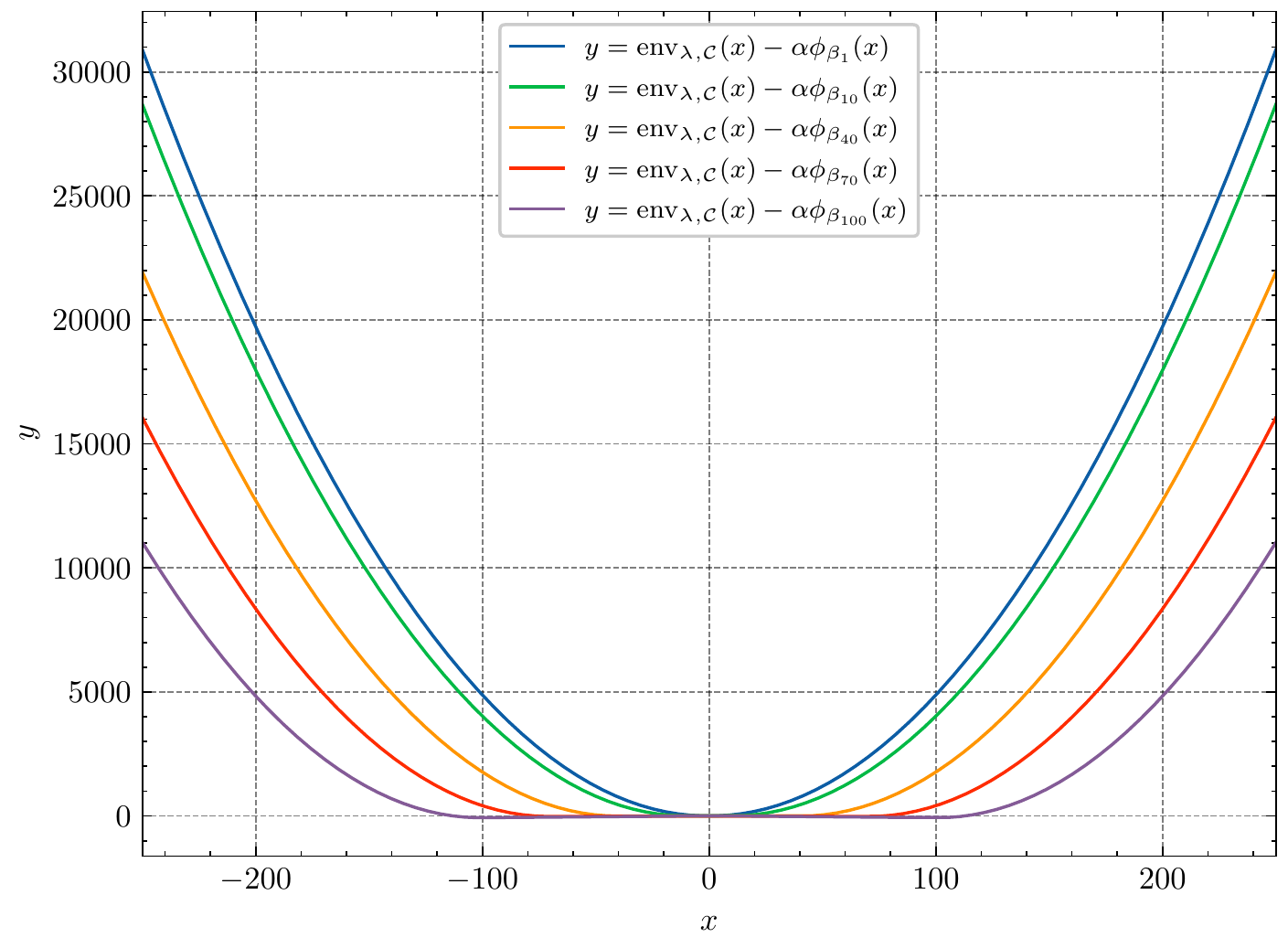}		
	\end{subfigure}
	\begin{subfigure}[b]{.49\textwidth}
		\centering
		\includegraphics[width=\textwidth]{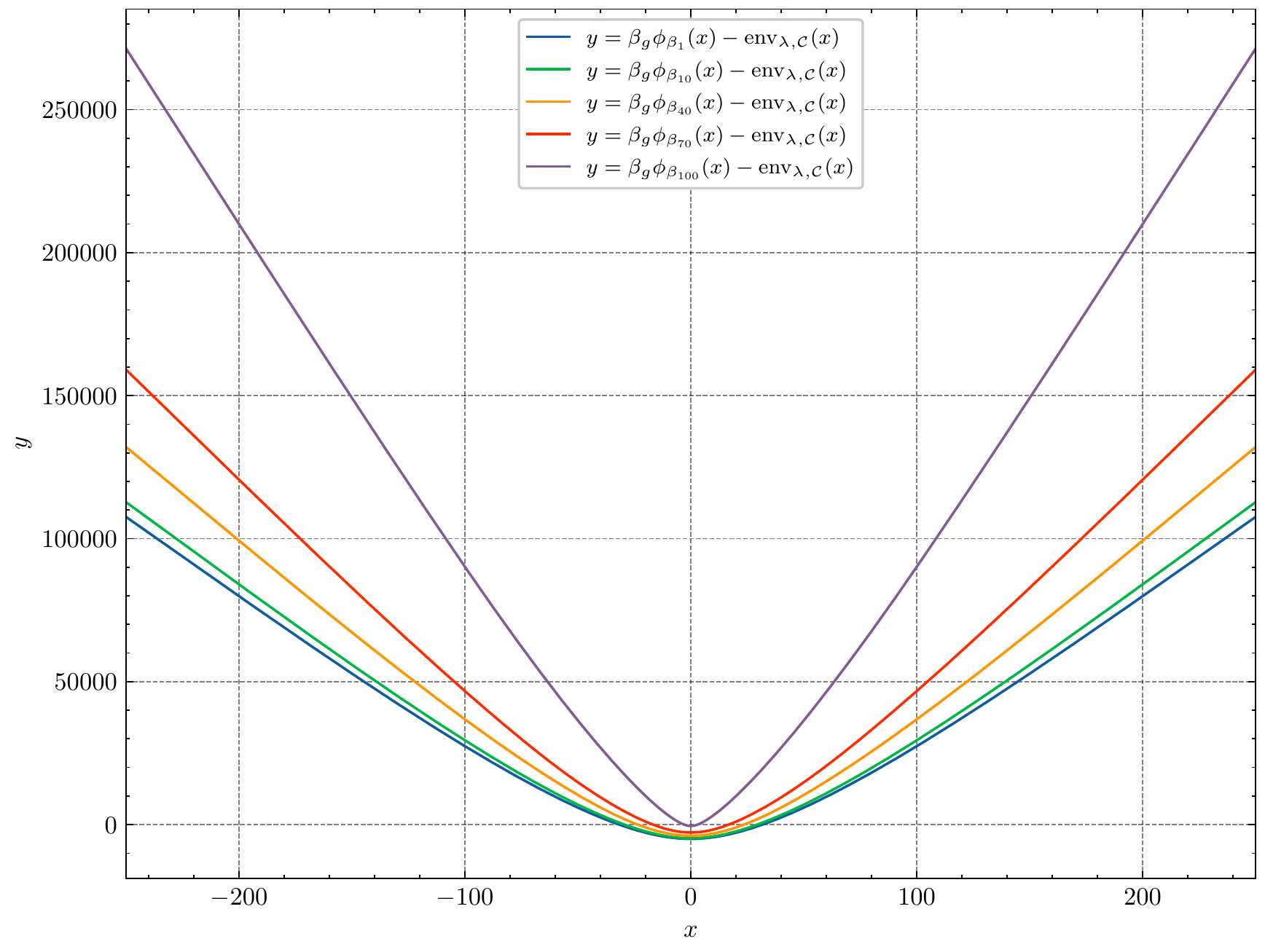}		
	\end{subfigure}
	\caption{Plots of $y = \env_{\lambda, \euC}(x) - \alpha\phi_{\beta_i}(x)$ (left) and $y = \beta_g\phi_{\beta_i}(x) - \env_{\lambda, \euC}(x)$ (right), for $i\in\{1, 10, 40, 70, 100\}$. }
	\label{fig:env_uniform}
\end{figure}

\subsection{Bayesian Sparse Logistic Regression}	
We compare the performance of MYULA and BMUMLA in Bayesian sparse logistic regression. Suppose that we observe the samples  $\{(\bx_n, y_n) \}_{n=1}^N$, where $\bx_n\in\RR^d$ and $y_n\in\{0,1\}$. In Bayesian logistic regression, the data are assumed to follow the model 
\begin{equation}\label{eqn:labels}
	y_n \iiddist \mathrm{Bernoulli}\left( \frac{\exp(\dotp{\btheta}{\bx_n})}{1+\exp(\dotp{\btheta}{\bx_n})}\right), 
\end{equation}
for each $n\in\set{N}$. The parameter $\btheta = (\theta_i)_{1 \le i \le d}^\top\in\RR^d$ is a random variable with a prior density $p$ with respect to Lebesgue measure. Then, the posterior distribution of $\btheta$ takes the form 
\[p(\btheta\mid\{(\bx_n, y_n) \}_{n=1}^N) \propto p(\btheta) \exp\left\lbrace\sumN \left(y_n \dotp{\btheta}{\bx_n} - \log (1 + \exp(\dotp{\btheta}{\bx_n})) \right)\right\rbrace. \]

We are particularly concerned with the case with a prior in the form of a combination of an anisotropic Laplace distribution (which is sparsity-inducing) and a Gaussian distribution, where the unadjusted Langevin algorithm is no longer viable due to the nonsmoothness induced by the anisotropic Laplace distribution. In general, such a prior takes the form: 
\[
p(\btheta) \coloneqq p(\btheta\mid\balpha_1, \alpha_2) \propto \exp\left\{- \sum_{i=1}^d \alpha_{1,i}|\theta_i| - \frac{\alpha_2}{2}\sumd \theta_i^2 \right\},
\]
where $\balpha_1 = (\alpha_{1,i})_{1\le i\le d}^\top\in\RP^d$ and $\alpha_2\in\RP$. 

Then, the resulting posterior distribution has a potential of the following form:
\[U(\btheta) = \underbrace{\sumN [\log (1+\exp(\dotp{\btheta}{\bx_n})) - y_n\dotp{\btheta}{\bx_n}] + \alpha_2\euclidnorm{\btheta}^2}_{\eqqcolon f(\btheta)} + \underbrace{\onenorm{\balpha_1\odot\btheta}}_{\eqqcolon g(\btheta)}. \]

We take $d=100$, $N=1000$ and $\btheta^\star = (\zero_{10}^\top, 0.1\cdot\One_{10}^\top, 0.2\cdot\One_{10}^\top, \ldots, 0.9\cdot\One_{10}^\top)^\top\in\RR^{100}$ as the ground truth. Then, each $x_{n,i}$ is generated from a standard Gaussian distribution and each $y_n$ is sampled following \eqref{eqn:labels} with $\btheta = \btheta^\star$. In addition, we choose $\balpha_1 = (10\cdot\One_{10}^\top, 9\cdot\One_{10}^\top, \ldots, 1\cdot\One_{10}^\top)^\top$ and $\alpha_2=0.1$. Again, we use the hypentropy functions $\varphi_{\bbeta}$  (for the mirror map) and $\psi_{\bsigma}$ (for the Bregman--Moreau envelope), with $\bbeta = (2i^{\sfrac14}\cdot\One_{10}^\top)_{1\le i\le 10}^\top$ and $\bsigma = (\alpha_{1,i}^2)_{1\le i\le d}^\top$. We also use a step size $\gamma=5\times10^{-4}$ and a smoothing parameter $\lambda=0.01$. Note that all of \Cref{assum:smooth,assum:nonsmooth,assum:mirror1,assum:mirror2,assum:Bregman,assum:functions} hold in this case. In particular, for \Cref{assum:functions}, notice that $f$ is indeed strongly convex.

We compare the performance of MYULA and BMUMLA by estimating the posterior means of $\btheta$ (as a whole or componentwise) and $\euclidnorm{\btheta}^2/d$. We generate 30 samples (indexed by $s$) using each algorithm for 4000 iterations and average the samples to obtain estimates $\btheta_k = \frac{1}{30}\sum_{s=1}^{30} \btheta_{k, s}$ and $\euclidnorm{\btheta_k}^2/d$ for the posterior means. From \Cref{fig:logistic_error}, we observe that the proposed left BMUMLA outperforms MYULA in the estimation of both posterior means. 

\begin{figure}[h!]
	\centering
	\begin{subfigure}[b]{.49\textwidth}
		\centering
		\includegraphics[width=\textwidth]{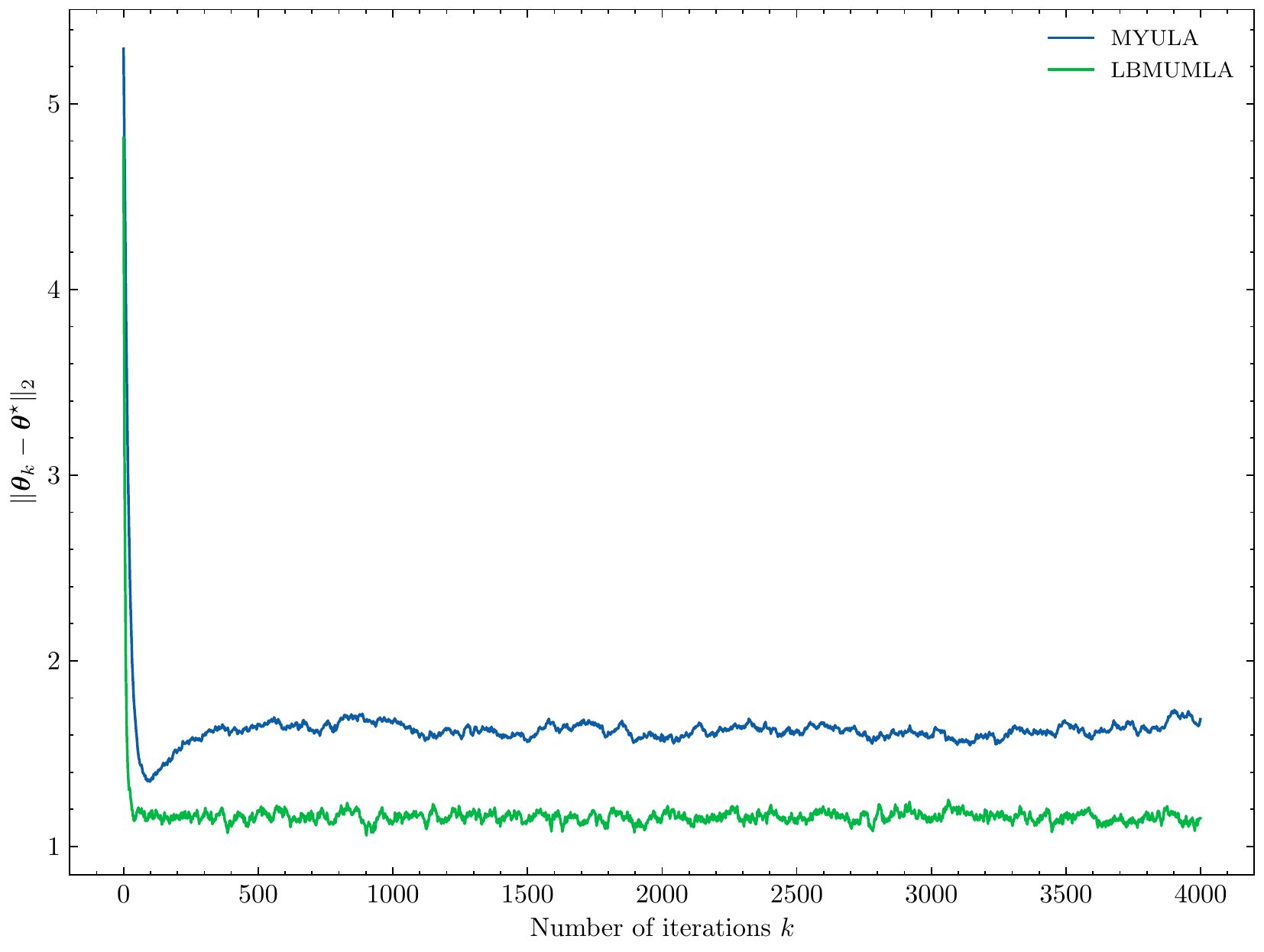}		
	\end{subfigure}
	\begin{subfigure}[b]{.49\textwidth}
		\centering
		\includegraphics[width=\textwidth]{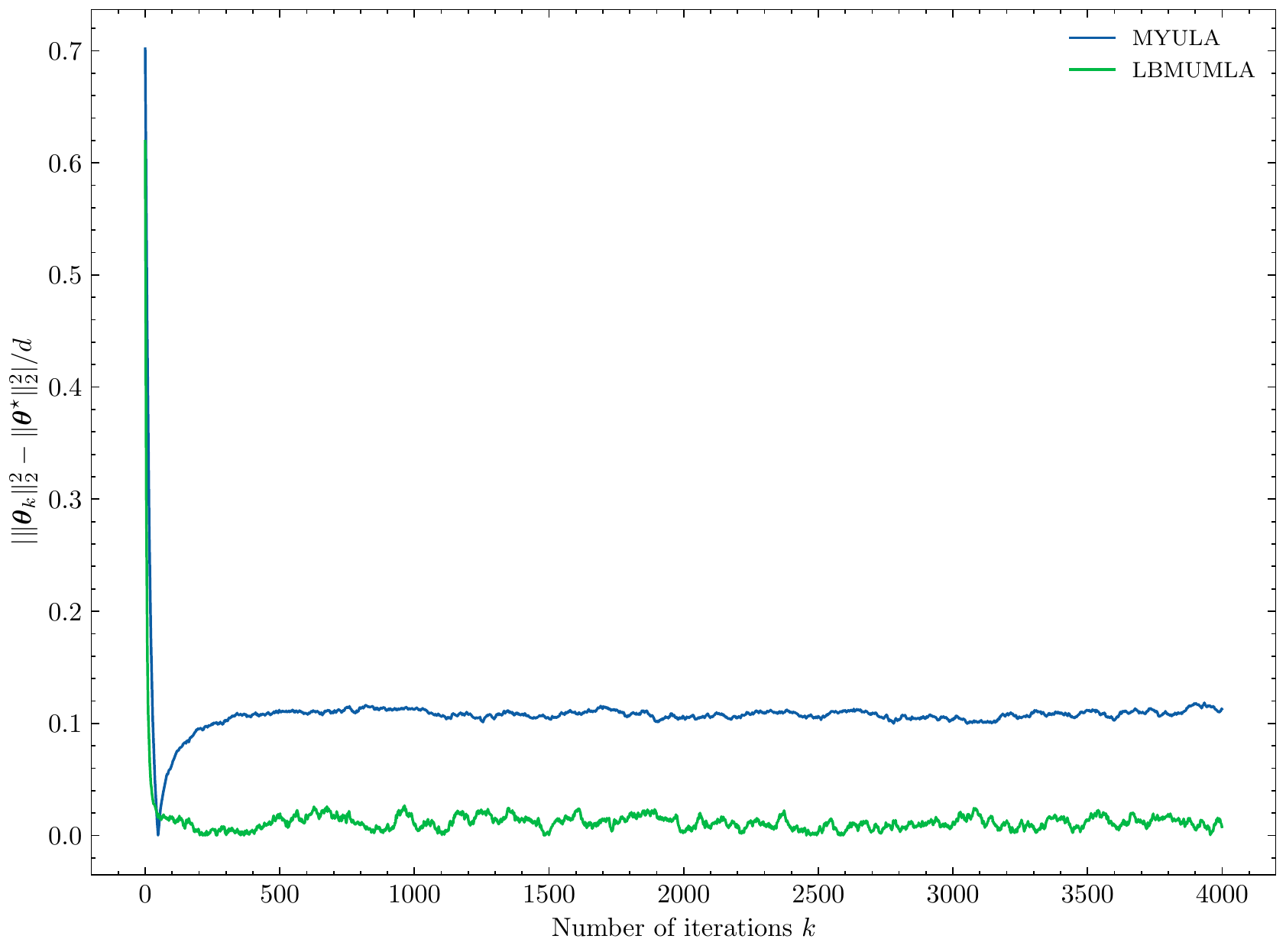}		
		
	\end{subfigure}
	\caption{Plots of estimation errors of the posterior means $\euclidnorm{\btheta_k - \btheta^\star}$ (left) and $\left|\euclidnorm{\btheta_k}^2 - \euclidnorm{\btheta^\star}^2\right|/d$ (right). }
	\label{fig:logistic_error}
\end{figure}

We also plot the estimation errors of the posterior means of some components of $\btheta$. \Cref{fig:logistic_single} reveals that MYULA gives smaller estimation errors than LBMUMLA at lower dimensions but high estimation errors at higher dimensions. However, we expect that the performance of BMUMLA would be further improved if $\bbeta$ and $\bsigma$ are more carefully picked or tuned, in order to fully adapt to the geometry of the posterior potential. 

\begin{figure}[h!]
	\centering
	\begin{subfigure}[b]{0.32\textwidth}
		\centering
		\includegraphics[width=\textwidth]{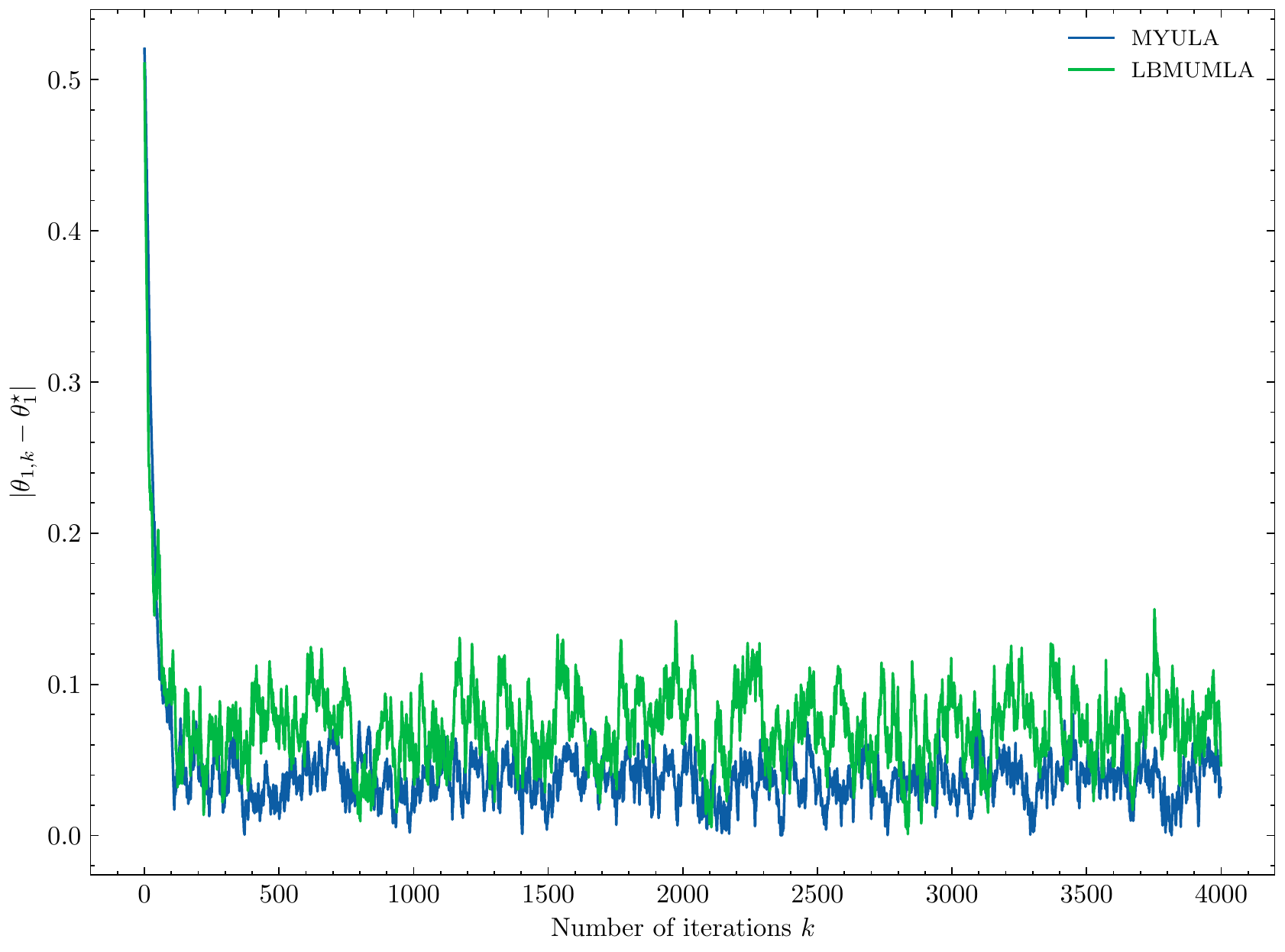}
		\caption{1\textsuperscript{st} dimension}
	\end{subfigure}
	\begin{subfigure}[b]{0.32\textwidth}
		\centering
		\includegraphics[width=\textwidth]{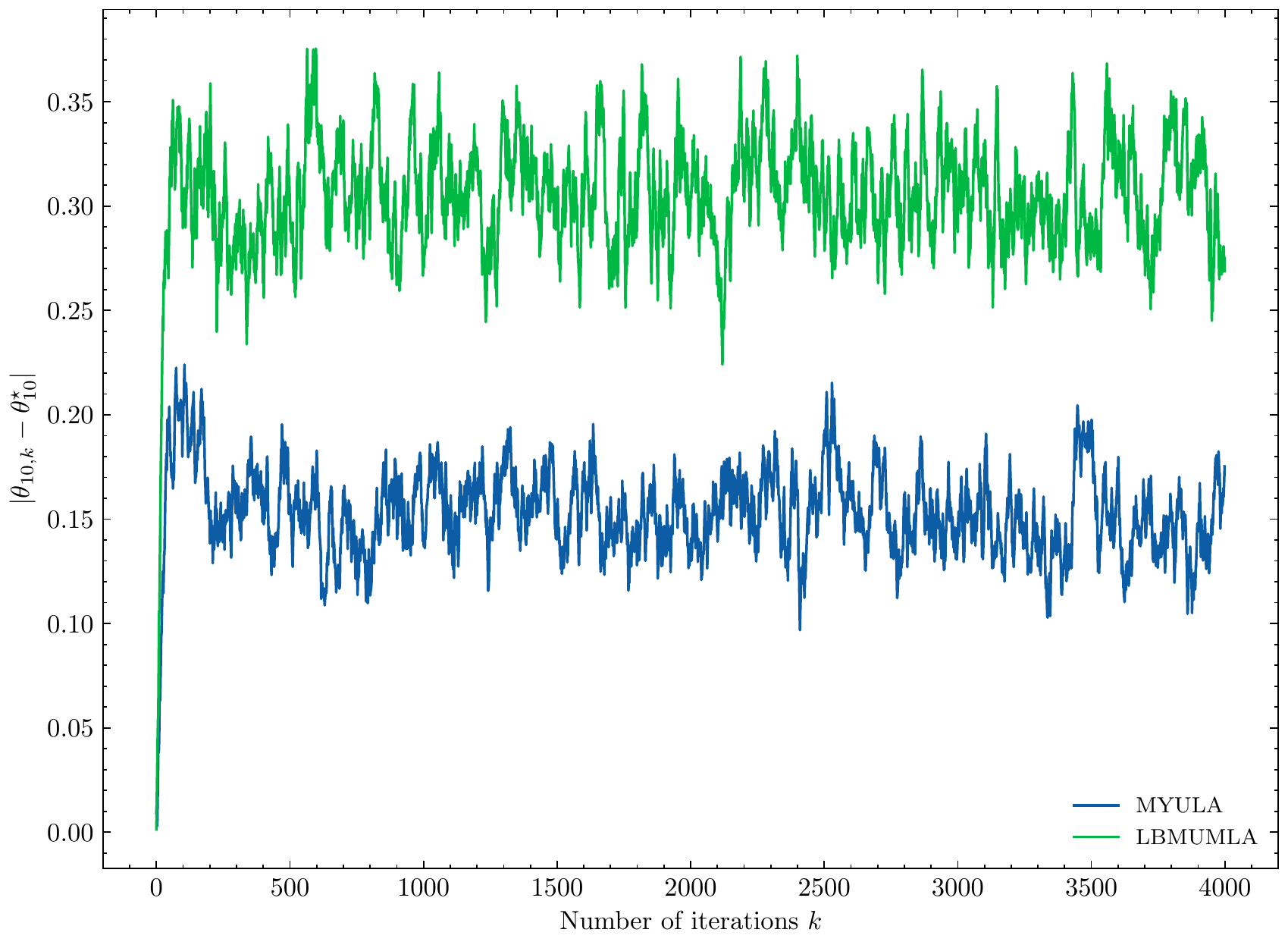}
		\caption{10\textsuperscript{th} dimension}
	\end{subfigure}
	\begin{subfigure}[b]{0.32\textwidth}
		\centering
		\includegraphics[width=\textwidth]{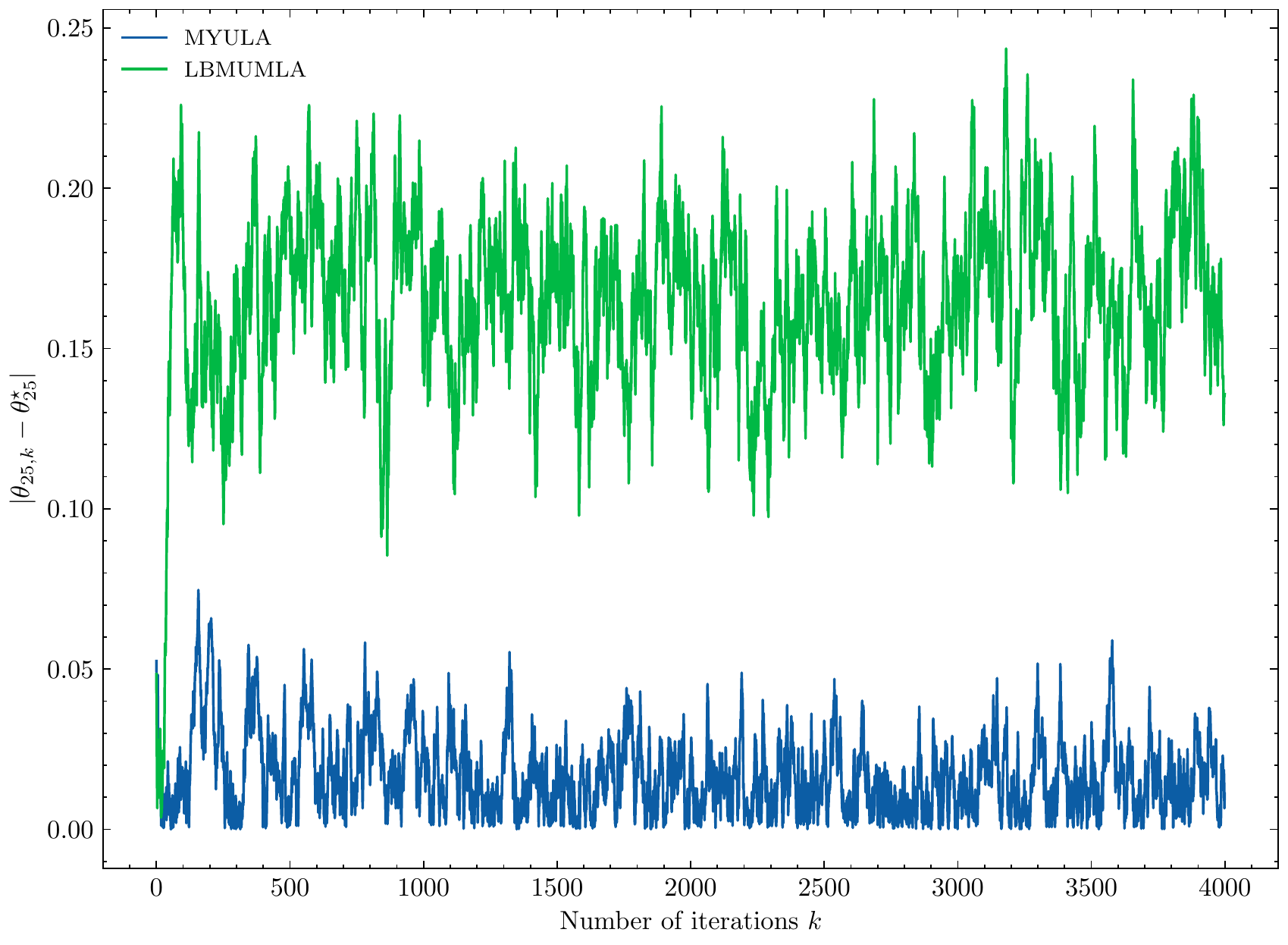}
		\caption{25\textsuperscript{th} dimension}
	\end{subfigure}		
	\newline
	\centering
	\begin{subfigure}[b]{0.32\textwidth}
		\centering
		\includegraphics[width=\textwidth]{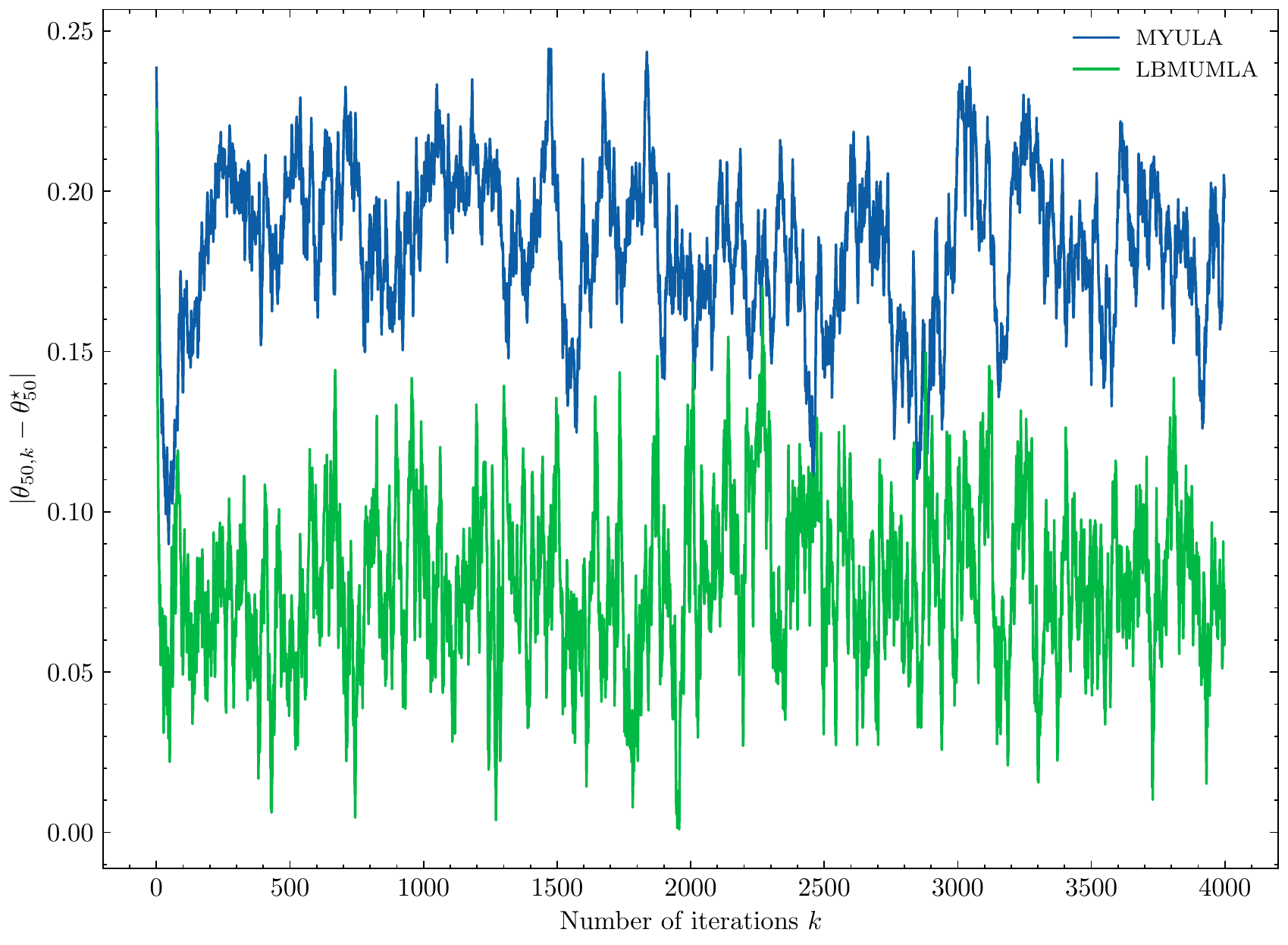}
		\caption{50\textsuperscript{th} dimension}
	\end{subfigure}
	\begin{subfigure}[b]{0.32\textwidth}			
		\centering
		\includegraphics[width=\textwidth]{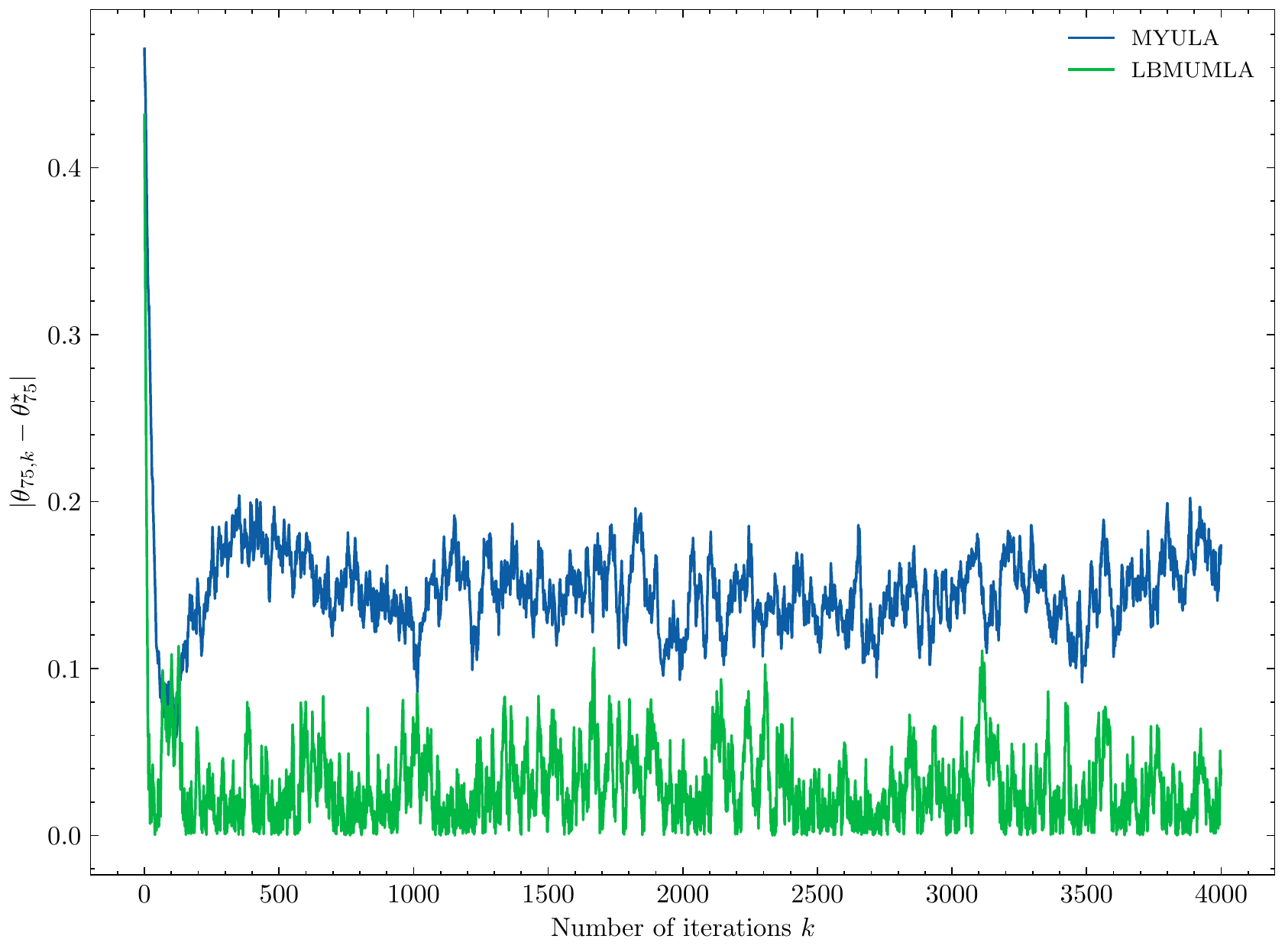}
		\caption{75\textsuperscript{th} dimension}
	\end{subfigure}
	\begin{subfigure}[b]{0.32\textwidth}
		\centering
		\includegraphics[width=\textwidth]{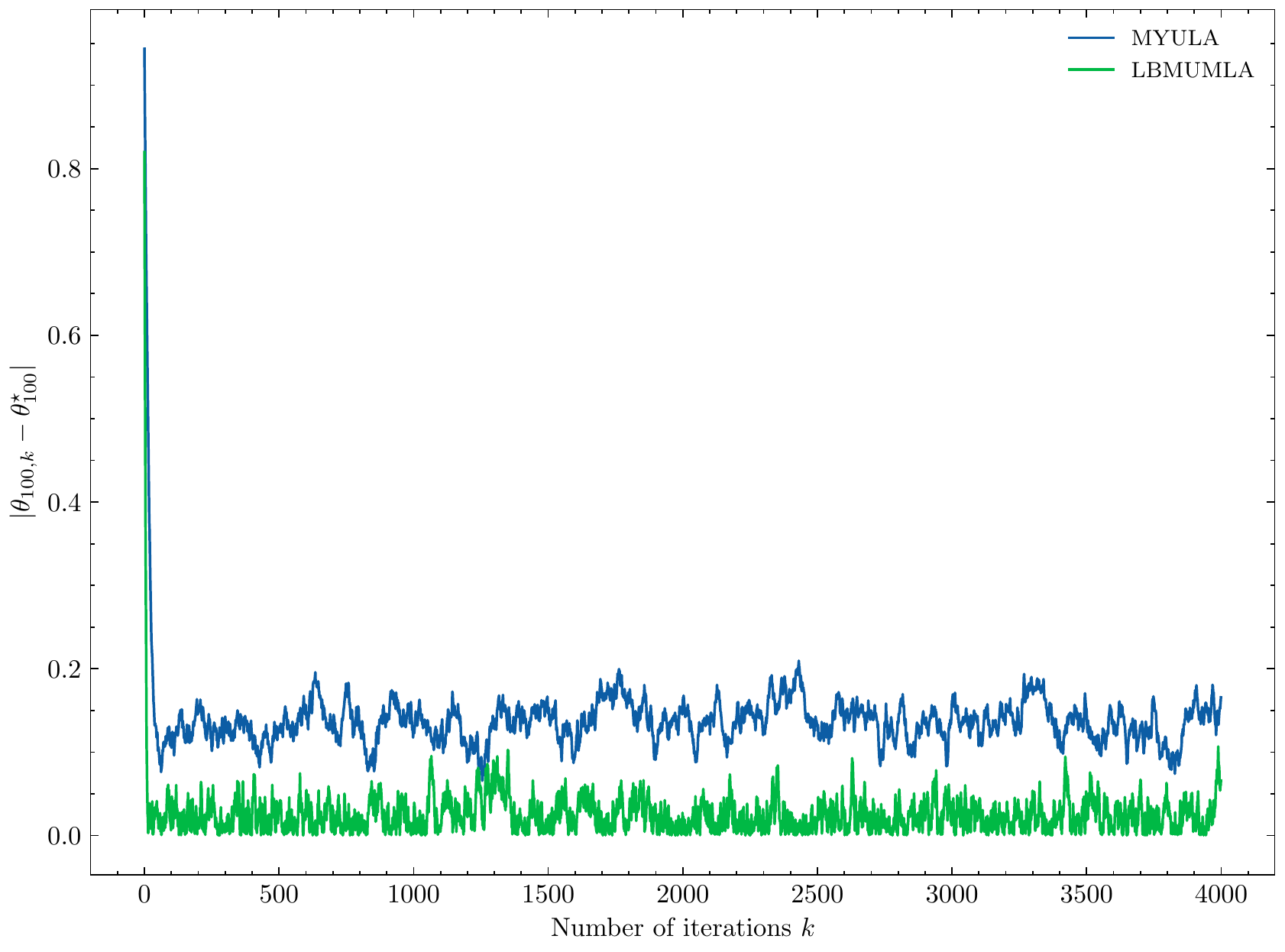}
		\caption{100\textsuperscript{th} dimension}
	\end{subfigure}
	\caption{Plots of estimation errors of the posterior means $|\theta_{k,i} - \theta_i^\star|$ for $i\in\{1,25,50,75,100\}$. }
	\label{fig:logistic_single}
\end{figure}

\end{document}